\def\year{2021}\relax
\documentclass[letterpaper]{article} 
\usepackage{aaai21}  
\usepackage{times}  
\usepackage{helvet} 
\usepackage{courier}  
\usepackage[hyphens]{url}  
\usepackage{graphicx} 
\usepackage{natbib}  
\usepackage{caption} 
\usepackage{lipsum}
\usepackage{mymacros}
\usepackage{mathrsfs}
\usepackage{algorithm}
\usepackage{algpseudocode}
\usepackage{wrapfig}
\usepackage{subcaption}
\usepackage{dsfont}
\usepackage{xcolor}
\usepackage{prodint}
\usepackage{enumitem}
\usepackage{booktabs}
\setlist[itemize]{noitemsep, topsep=0pt}
\setlist[enumerate]{noitemsep, topsep=0pt}
\allowdisplaybreaks[4]
\allowdisplaybreaks[4]

\newcommand{\algname}{RANDOMIST\@\xspace}
\newcommand{\problem}{PO\@\xspace}
\newcommand{\algnameext}{RANDomized-exploration policy Optimization via Multiple Importance Sampling with Truncation\@\xspace}

\newcommand{\HL}[1]{\textcolor{black}{#1}}

\title{Policy Optimization as Online Learning with Mediator Feedback}
\author{
    Alberto Maria Metelli\thanks{Equal contribution.},
    Matteo Papini\footnotemark[1],
    Pierluca D'Oro,
    Marcello Restelli
    \\
}
\affiliations{
    Dipartimento di Elettronica, Informazione e Bioingegneria, Politecnico di Milano\\
		Piazza Leonardo da Vinci, 32, 20133, Milano, Italy\\
	 \{albertomaria.metelli, matteo.papini, marcello.restelli\}@polimi.it, pierluca.doro@mail.polimi.it
}
\begin{document}

\maketitle

\begin{abstract}
Policy Optimization (PO) is a widely used approach to address continuous control tasks. In this paper, we introduce the notion of \emph{mediator feedback} that frames PO as an online learning problem over the policy space. The additional available information, compared to the standard bandit feedback, allows reusing  samples generated by one policy to estimate the performance of other policies. Based on this observation, we propose an algorithm, \emph{RANDomized-exploration policy Optimization via Multiple Importance Sampling with Truncation} (RANDOMIST), for regret minimization in PO, that employs a randomized exploration strategy, differently from the existing optimistic approaches. When the policy space is finite, we show that under certain circumstances, it is possible to achieve constant regret, while always enjoying logarithmic regret. We also derive problem-dependent regret lower bounds. Then, we extend RANDOMIST to compact policy spaces. Finally, we provide numerical simulations on finite and compact policy spaces, in comparison with PO and bandit baselines.
\end{abstract}
\section{Introduction}
Policy Optimization~\citep[PO,][]{deisenroth2013survery} is a family of Reinforcement Learning~\citep[RL,][]{sutton2018reinforcement} algorithms based on the explicit optimization of the policy parameters. It represents the most promising approach for learning large-scale continuous control tasks and has already achieved marvelous results in video games~\citep[\eg][]{vinyals2019grandmaster} and robotics~\cite[\eg][]{peng2020learning}. These achievements, however, rely on massive amounts of simulation rollouts. The efficient use of experience data is essential both to reduce computational costs and to make learning online from real interaction possible. This is still largely an open problem and calls for better theoretical understanding.
Any online-learning agent must face the exploration-exploitation dilemma: whether to leverage on its current knowledge to maximize performance or consider new alternatives.
Fortunately, the Multi-Armed Bandit (MAB) literature~\cite{bubeck2012regret,lattimore2018bandit} provides a theoretical framework for the problem of efficient exploration under \textit{bandit feedback}, \ie observing the effects of the chosen actions.
 The dilemma is addressed by minimizing the cumulative \emph{regret} of the online performance \wrt the optimal one. The most popular exploration strategies are based on the Optimism in the Face of Uncertainty~\citep[OFU,][]{lai1985asymptotically}, of which UCB1~\cite{auer2002finite} is the prototypical algorithm, and on Thompson Sampling~\citep[TS,][]{thompson1933likelihood}. Both suffer only sublinear regret~\cite{auer2002finite,agrawal2012analysis,kaufmann2012thompson}. TS typically performs better in practice~\cite{chapelle2011empirical}, but it is only computationally efficient in artificial settings~\cite{kveton2019garbage}. More recent randomized algorithms such as PHE (Perturbed History Exploration)~\cite{kveton2019perturbed} are able to match the theoretical and practical advantages of TS without the computational burden, and with no assumptions on the payoff distribution.
 
The OFU principle has been applied to RL~\cite{jaksch2010near}
and recently also to PO~\cite{chowdhury2019online,efroni2020optimistic}, at the level of action selection. These methods are promising but limited to finite actions.
A different perspective is proposed by~\citet{papini2019optimistic}, where the decision problem is not defined over the agent's actions but over the policy parameters. This change of viewpoint allows exploiting the special structure of the PO problem: for each policy, a sequence of states and actions performed by the agent is collected, constituting, alongside the rewards, a vastly richer signal than the simple bandit feedback.
In this paper, we call it \emph{mediator feedback} since this extra information acts as a mediator variable between the policy parameters and the return.
OPTIMIST~\cite{papini2019optimistic} is an OFU algorithm that uses Multiple Importance Sampling~\citep[MIS,][]{veach1995optimally} to exploit the mediator feedback, so that the results of one policy provide information on all the others.
This allows, in principle, to optimize over an infinite policy space with only finite samples and no regularity assumptions on the underlying process.
There are two important limitations in~\citet{papini2019optimistic}. First, the advantages of the mediator feedback over the bandit feedback are not clear from a theoretical perspective since the regret of OPTIMIST is comparable with that of UCB1 with finite policy space. Second, the policy selection of OPTIMIST requires maximizing a non-convex and non-differentiable index. In the continuous setting, this is addressed via discretization, with clear scalability issues.

In this work, we provide two major advancements. From the theoretical side,
we provide regret lower bounds for the policy optimization problem with finite policy space, and we show that OPTIMIST actually enjoys \emph{constant} regret under the assumptions made in~\cite{papini2019optimistic}. In fact, mediator feedback is so special that, under strong-enough assumptions, a greedy algorithm enjoys the same guarantees.
We also devise a PHE-inspired randomized algorithm, called RANDOMIST (RANDomized-exploration policy Optimization via Multiple Importance Sampling with Truncation), with similar regret guarantees as OPTIMIST. From the practical side, this allows replacing the unfeasible index maximization of OPTIMIST with a sampling procedure. Although our regret guarantees apply to the finite setting only, we propose a heuristic version of RANDOMIST for continuous problems, using a Markov Chain Monte Carlo~\cite[MCMC,][]{owen2013monte}. We show the advantages of this algorithm over continuous OPTIMIST in terms of computational complexity and performance.

The structure of the paper is as follows. We start in Section~\ref{sec:preliminaries} with the basic background. In Section~\ref{sec:problem}, we formalize the concept of mediator feedback in PO and derive two regret lower bounds. We illustrate, in Section~\ref{sec:is}, a possible way to exploit mediator feedback, based on importance sampling. Section~\ref{sec:det} is devoted to the discussion of deterministic algorithms, providing the improved regret guarantees for OPTIMIST. In Section~\ref{sec:rand}, we present RANDOMIST with its regret guarantees and the heuristic extension to the continuous case. In Section~\ref{sec:exp}, we compare empirically RANDOMIST with relevant baselines on both illustrative examples and continuous-control problems. In Section~\ref{sec:related}, we discuss relationships with similar approaches from the bandit and RL literature. We conclude in Section~\ref{sec:discussion}, discussing the obtained results and proposing future research directions. The proofs of all the results can be found in Appendix~\ref{apx:proofs}.

\section{Preliminaries}\label{sec:preliminaries}
In this section, we introduce some notation, the background on Markov decision processes and policy optimization.

\paragraph{Mathematical Background}Let $(\mathcal{X}, \mathscr{F})$ be a measurable space, we denote with $\mathscr{P}(\mathcal{X})$ the set of probability measures over $\mathcal{X}$. Let $P,Q \in \mathscr{P}(\mathcal{X})$ such that $P \ll Q$,\footnote{\HL{$P$ is absolutely continuous \wrt $Q$, \ie for every measurable set $\mathcal{Y} \subseteq \mathcal{X}$ we have $Q(\mathcal{Y}) = 0 \Rightarrow P(\mathcal{Y})=0$}.} for any $\alpha \in [0,\infty]$ the $\alpha$-\Renyi divergence~\cite{renyi1961measures} is defined as:\footnote{In the limit, for $\alpha \rightarrow 1$ we have $D_1(P\|Q) = D_{\text{KL}}(P\|Q)$ and for $\alpha \rightarrow \infty$ we have $D_{\infty}(P\|Q) = \esssup_{\mathcal{X}}  \frac{\de P}{\de Q}$.}
$$
	D_{\alpha}(P\|Q) = \frac{1}{\alpha-1} \log \int_{\mathcal{X}} \left(\frac{\de P}{\de Q} \right)^{\alpha} \de Q.$$
We denote with $d_{\alpha}(P\|Q) = \exp \left[ D_{\alpha}(P\|Q) \right]$ the exponentiated \Renyi divergence \cite{cortes2010learning}.

\paragraph{Markov Decision Processes and Policy Optimization}A discrete-time Markov Decision Process~\citep[MDP,][]{puterman1994markov} is a 6-tuple $\mathcal{M} = (\Ss, \As, \Ps, \Rs, \gamma, \mu)$, where $\Ss$ is the state space, $\As$ is the action space, $\Ps$ is the  transition model that for each $(s,a) \in \SAs$ provides the probability distribution of the next state $\Ps(\cdot|s,a) \in \mathscr{P}(\Ss)$, $\Rs(s,a) \in \mathbb{R}$ is the reward function, $\gamma \in [0,1]$ is the discount factor, and $\mu \in \mathscr{P}(\Ss)$ is the initial-state distribution.
In Policy Optimization~\citep[PO,][]{peters2008reinforcement}, we model the agent's behavior by means of a policy $\pi_{\vtheta}(\cdot|s) \in \mathscr{P}(\As) $ belonging to a space of parametric policies $\Pi_{\Theta} = \{\pi_{\vtheta} \, : \, \vtheta \in \Theta \}$.
%
The interaction between an agent and an MDP generates a sequence of state-action pairs, named \emph{trajectory}: $\tau = (s_0,a_0,s_1,a_1,\dots,s_{H-1},a_{H-1})$ where $s_0 \sim \mu$, for all $h \in \{0, \dots, H-1\}$ we have $a_h \sim \pi_{\vtheta}(\cdot|s_h)$, $s_{h+1} \sim \Ps(\cdot|s_h,a_h)$ and $H \in \mathbb{N}$ is the trajectory length. Each parameter $\vtheta \in \Theta$ determines a policy $\pi_{\vtheta} \in \Pi_{\Theta}$ which, in turn, induces a probability measure $p_{\vtheta} \in \mathscr{P}(\mathcal{T})$ over the trajectory space $\mathcal{T}$.
To every trajectory $\tau \in \mathcal{T}$, we associate an index of performance $\Rs(\tau) = \sum_{h=0}^{H-1} \gamma^h \Rs(s_h,a_h)$, called \emph{return}. Without loss of generality we assume that $\mathcal{R}(\tau) \in [0,1]$. Thus, we can evaluate the performance of a policy $\pi_{\vtheta} \in \Pi_\Theta$ by means of its \emph{expected return}: $J(\vtheta) = \E_{\tau \sim p_{\vtheta} }\left[\Rs(\tau)\right]$. The goal of the agent consists in finding an optimal parameter, \ie any $\vtheta^*$ maximizing $J(\vtheta)$.\footnote{To simplify the presentation, we frame our results for the usual \emph{action-based} PO. Our findings directly extend to \emph{parameter-based} exploration~\citep{sehnke2008policy}, in which policies are indirectly optimized by learning a hyperpolicy that outputs the policy parameters. Coherently with~\citet{papini2019optimistic}, the empirical evaluation of Section~\ref{sec:exp} is carried out in the parameter-based framework.}

\section{Online Policy Optimization and Mediator Feedback}\label{sec:problem}
The online PO protocol works as follows. At each round $t\in[n]$, we evaluate a parameter vector $\vtheta_t \in \Theta$ by running policy $\pi_{\vtheta_t}$, collecting one (or more) trajectory $\tau_t \in \OutcomeSet$ and observing the corresponding return $\Rs(\tau_t)$. Then, based on the history $\Hs_{t} = \{(\vtheta_i, \tau_i, \Rs(\tau_i))\}_{i=1}^t$, we update $\vtheta_{t}$ to get $\vtheta_{t+1}$. From an \emph{online learning} perspective, the goal of the agent consists in maximizing the sum of the expected returns over $n$ rounds or, equivalently, minimizing the cumulative regret $R(n)$:
\begin{equation*}
	\max_{\xs_1, \dots \xs_{n} \in \DecisionSet} \sum_{t=1}^n J(\xs_t) \; \Leftrightarrow \;\min_{\xs_1, \dots \xs_{n} \in \DecisionSet} R(n) = \sum_{t=1}^n \Delta(\xs_t),
\end{equation*}
where $\Delta(\xs) = J^* - J(\xs)$ is the optimality gap of $\xs \in \DecisionSet$ and $J^* =  \sup_{\xs \in \DecisionSet} J(\xs)$.
Thus, whenever policy $\pi_{\vtheta_t}$ is executed the agent receives the trajectory-return pair $(\tau_t, \Rs(\tau_t))$, that we name \emph{mediator feedback} (MF). The term \quotes{mediator} refers to the side information, the trajectory $\tau_t$, that \emph{mediates} between the parameter choice $\vtheta_t$ and the return $\Rs(\tau_t)$. By na\"ively approaching PO as an online-learning problem over policy space, we would only consider \emph{bandit feedback}, in which just the return $\Rs(\tau_t)$ is observable. In comparison, the MF allows to better exploit the \emph{structure} underlying the PO problem (Figure~\ref{fig:GraphicalModels}).\footnote{In this paper, we employ the wording \quotes{bandit feedback} with a different meaning compared to some provably efficient approaches to PO~\citep[\eg][]{efroni2020optimistic}. See also Section~\ref{sec:related}.} Indeed, while the return function $\mathcal{R}$ is unknown, the trajectory distribution $p_{\vtheta}$ is \emph{partially} known:
\begin{align}\label{eq:factoriz}
	p_{\vtheta}(\tau) = \mu(s_0) \prod_{h=0}^{H-1} \pi_{\vtheta}(a_h|s_h) P(s_{h+1}|s_h,a_h).
\end{align}
The policy factors $\pi_{\vtheta}$, that depend on $\vtheta$, are known to the agent, whereas the factors due to the environment ($\mu$ and $P$) are unknown but do not depend on $\vtheta$. Intuitively, if two policies $\pi_{\xs}$ and $\pi_{\xs'}$ are sufficiently \quotes{similar}, given a trajectory $\outcome$ from policy $\pi_{\xs}$, the return $\Rs(\outcome)$ provides information on the expected return of policy $\pi_{\xs'}$ too.

\begin{figure}
\centering
\includegraphics[scale=1]{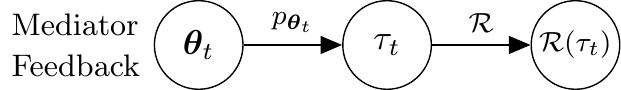} \hfill \includegraphics[scale=1]{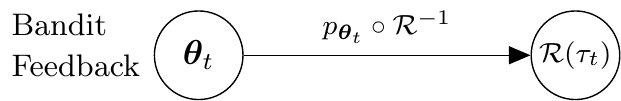}
\caption{Graphical models comparing mediator and bandit feedbacks.}\label{fig:GraphicalModels}
\end{figure}
\subsection{Regret Lower Bounds for Finite Policy Space}
We focus on the intrinsic complexity of PO with finite policy space, deriving two lower bounds to the regret. The results are phrased, for simplicity, for the  case of two policies, \ie $|\DecisionSet|=2$, and the proof techniques are inspired to~\cite{bubeck2013bounded}.
We start showing that, with enough structure between the policies, \ie when the KL-divergence between the trajectory distributions is bounded, the best achievable regret is constant.
\begin{restatable}[]{thr}{lbConstantRegret}\label{thr:lbConstantRegret}
There exist an MDP and a parameter space $\Theta=\{\xs_1,\xs_2\}$ with $\KL(p_{\xs_1} \| p_{\xs_2}) < \infty$, $\KL( p_{\xs_2} \| p_{\xs_1}) < \infty$ and $J(\xs_1) - J(\xs_2) = \Delta$ such that, for sufficiently large $n$, all algorithms suffer regret $\E R(n) \ge \frac{1}{32 \Delta}$.
\end{restatable}
Instead, the presence of policies that are uninformative of one another, \ie with infinite KL-divergence between the trajectory distributions, leads to a logarithmic regret.

\begin{restatable}[]{thr}{lbLogRegret}\label{thr:lbLogRegret}
There exist an MDP and a parameter space $\Theta=\{\xs_1,\xs_2\}$ with $\KL(p_{\xs_1} \| p_{\xs_2}) \!= \!\infty$ or $\KL( p_{\xs_2} \| p_{\xs_1}) \!= \! \infty$, and $J(\xs_1) - J(\xs_2) = \Delta$ such that, for any $n \ge 1$, all algorithms suffer regret $\E R(n) \ge \frac{1}{8\Delta}  \log (\Delta^2 n)$.
\end{restatable}

\section{Exploiting Mediator Feedback with Importance Sampling}\label{sec:is}
In this section, we illustrate how Importance Sampling techniques~\citep[IS,][]{cochran77sampling, owen2013monte} can be employed to effectively exploit the mediator feedback in PO.\footnote{We stress that IS is just \emph{one} method, and not necessarily the best one, to exploit the structure of the PO problem.}

\paragraph{Monte Carlo Estimation}With the bandit feedback at each round $t \in [n]$, the agent has access to the history of parameter-return pairs $\Hs_{t} = \{ (\xs_i, {\func}(\tau_i)\}_{i=1}^{t-1}$. Let $T_t(\xs) = \sum_{i=1}^{t-1} \Ind \{ \xs_i  = \xs \}$ be the number of trajectories collected with policy $\pi_{\xs} \in \Pi_{\DecisionSet}$  up to round $t-1$. To estimate the expected return $J(\xs)$, if no additional structure is available, we can only use the samples collected when executing $\pi_{\xs}$, leading to the Monte Carlo (MC) estimator:
\begin{equation}
	\widehat{J}_t^{\text{MC}}(\xs) = \frac{1}{T_t(\xs)} \sum_{i=1}^{t-1} {\func}(\tau_i) \Ind \{ \xs_i  = \xs \}.
\end{equation}
$\widehat{J}_t^{\text{MC}}$ is unbiased for $J(\xs)$ and its variance scales with $\Var[\widehat{J}_t^{\text{MC}}(\xs)] \le 1/T_t(\xs)$. Clearly, $\widehat{J}_t^{\text{MC}}(\xs)$ can be computed only for the policies that have been executed at least once.

\paragraph{Multiple Importance Sampling Estimation}
With the mediator feedback, at each round $t \in [n]$ we have access to additional information, \ie the history  of parameter-trajectory-return triples  $\Hs_{t} = \{ (\xs_i, \outcome_i, {\func}(\tau_i))\}_{i=1}^{t-1}$. Thanks to the factorization in Equation~\eqref{eq:factoriz}, we can compute the trajectory distribution ratios without knowing $P$ and $\mu$:
\begin{align*}
	\frac{p_{\xs}(\tau)}{p_{\xs'}(\tau)} = \prod_{h=0}^{H-1} \frac{\pi_{\xs}(a_h|s_h)}{\pi_{\xs'}(a_h|s_h)}.
\end{align*}
Thus, we can use \emph{all} the samples to estimate the expected return of \emph{any} policy. Let $\Phi_t =  \sum_{j=1}^{t-1} \frac{1}{t-1} p_{\xs_j}$ be the mixture induced by the policies executed up to time $t-1$: if $p_{\xs} \ll \Phi_t$, we can employ a Multiple Importance Sampling~\citep[MIS,][]{veach1995optimally} estimator (with balance heuristic):\footnote{For an extensive discussion of importance sampling and heuristics (\eg balance heuristic) refer to~\cite{owen2013monte}.}
\begin{equation}
	\widehat{J}_t(\xs) = \frac{1}{t-1} \sum_{i=1}^{t-1} \omega_{\xs,t}(\tau_i) {\func}(\tau_i),
\end{equation}
where $\omega_{\xs,t}(\tau_i) =  p_{\xs}(\tau_i) / \Phi_t(\tau_i)$ is the \emph{importance weight}. Thus, for estimating the expected return $J(\xs)$ of policy $\pi_{\xs}$ we do not need to execute $\pi_{\xs}$, but just require the absolute continuity $p_{\xs} \ll \Phi_t$ (surely fulfilled if $T_t(\xs) \ge 1$).
The statistical properties of the MIS estimator can be phrased in terms of the \Renyi divergence. We can prove that $0 \le \widehat{J}_t(\xs)  \le d_{\infty}(p_{\xs} \| \Phi_t)$ and the variance  can be bounded as $\Var[ \widehat{J}_t(\xs)] \le d_2(p_{\xs} \| \Phi_t)/(t-1)$~\cite{metelli2018policy, papini2019optimistic, metelli2020importance}. Since the variance of $\widehat{J}_t(\xs)$ scales with $d_2(p_{\xs} \| \Phi_t)/(t-1)$ instead of $1/T_t(\xs)$, as for $\widehat{J}_t^{\text{MC}}(\xs)$, we refer to $\eta_t(\xs) := (t-1)/d_{2}(p_{\xs} \| \Phi_t)$ as the \emph{effective number of trajectories}. It is worth noting that $\eta_t(\xs) \ge T_t(\xs)$ (Lemma~\ref{th:ess}); thus, thanks to the structure introduced by the mediator feedback, the MIS estimator variance is always smaller than the MC estimator variance.\footnote{The effective number of trajectories $\eta_t(\xs)$ is, in fact, the \emph{effective sample size} of $\widehat{J}_t(\xs)$~\cite{martino2017effective}.}

\paragraph{Truncated Multiple Importance Sampling Estimation}
The main limitation of the MIS estimator is that the importance weight $\omega_{\xs,t}$ displays a \emph{heavy-tail} behavior, preventing  exponential concentration, unless $d_{\infty}(p_{\xs} \| \Phi_t)$ is finite~\cite{metelli2018policy}. A common solution consists in \emph{truncating} the estimator~\cite{ionides2008truncated} at the cost of introducing a negative bias. Given a (time-variant and policy-dependent) truncation threshold $ M_t(\xs) < \infty$, the Truncated MIS (TMIS) was introduced by~\citet{papini2019optimistic}:
\begin{equation}
\widecheck{J}_t(\xs)  = \frac{1}{t-1} \sum_{i=1}^{t-1} \widecheck{\omega}_{\xs,t} (\outcome_i) {\func}(\tau_i),
\end{equation}
where $ \widecheck{\omega}_{\xs,t} (\outcome_i) = \min \left\{ M_t(\xs), \omega_{\xs,t} (\outcome_i) \right\} $. TMIS enjoys more desirable theoretical properties than plain MIS. While its variance scales similarly to $\widehat{J}_t(\xs)$ since
$\Var[ \widecheck{J}_t(\xs)] \le d_2(p_{\xs} \| \Phi_t)/(t-1)$, the range can be bounded as $0 \le \widecheck{J}_t(\xs) \le M_t(\xs)$. Thus, the range is controlled by $M_t(\xs)$ and no longer by the divergence $d_{\infty}(p_{\xs} \| \Phi_t)$, which may be infinite. Similarly, the bias can be bounded as $ J(\xs) -\E_{\outcome_i \sim p_{\xs_i}}[\widecheck{J}_t(\xs) ] \le d_2(p_{\xs} \| \Phi_t) / M_t(\xs) $ (see~\citet{papini2019optimistic} and Lemma~\ref{lemma:VarBiasMax} for details). If we are interested in minimizing the joint contribution of bias and variance, this suggests to increase $M_t(\xs)$ progressively over the rounds.

\section{Deterministic Algorithms}\label{sec:det}
In this section, we consider finite policy spaces ($|\Theta| < \infty$) and discuss algorithms for \problem that select policies deterministically, \ie $\xs_t$ is a deterministic function of history $\Hs_{t-1}$.

\paragraph{Follow The Leader}
The simplest algorithm accounting for the mediator feedback is Follow The Leader (FTL). It maintains a TMIS estimator $\widecheck{J}_t(\xs)$ and selects the policy with the highest estimated expected return, \ie $\xs_t \in \argmax_{\xs \in \DecisionSet} \widecheck{J}_t(\xs)$. This is a pure-exploitation algorithm, unsuited for bandit feedback. Surprisingly, under a strong form of mediator feedback, FTL enjoys \emph{constant} regret.
\begin{restatable}[]{thr}{ftl}\label{th:ftl}
	Let $\DecisionSet = [K]$, $v(\xs)=\max_{\xs'\in\DecisionSet}d_2(p_{\xs}\|p_{\xs'})$ for all $\vtheta \in \Theta$ and $v^{*}(\xs)=\max\{v(\xs),v(\xs^{*})\}$, where $\pi_{\xs^*}$ is an optimal policy. If $v\coloneqq \max_{\xs\in\DecisionSet}v(\xs)<\infty$, then, for any $\alpha>1$, the expected regret of FTL using TMIS with truncation ${M_t(\xs) = \sqrt{\frac{td_{2}(p_{\xs}\|\Phi_t)}{\alpha\log t}}}$ is bounded as:
	\begin{equation}
	\begin{aligned}
		\E R(n) & \leq \sum_{\xs\in\DecisionSet:\Delta(\xs)>0}
		\frac{48\alpha v^{*}(\xs)}{\Delta(\xs)}\log\frac{24\alpha v^{*}(\xs)}{\Delta(\xs)^2}\\
		& \quad
		+ \Delta(\xs_1) + \frac{2 K}{\alpha-1}\min\left\{1, \sqrt{2\log v}\right\}.
	\end{aligned}
	\end{equation}
\end{restatable}
We refer to the condition when all pairwise \Renyi divergences are finite (\ie $v < \infty$) as \emph{perfect mediator feedback}. In such case, we have the remarkable property that running \emph{any} policy in $\Pi_{\Theta}$ provides information for \emph{all} the others. Indeed, the effective number of trajectories satisfies $\eta_t(\xs) \ge (t-1)/v$ (Lemma~\ref{th:ess}). Unfortunately, when $v = \infty$, FTL degenerates to \emph{linear} regret (Fact~\ref{fact:ftl}).

\paragraph{UCB1}
We can always apply an algorithm for standard bandit feedback, like UCB1~\cite{lai1985asymptotically,auer2002finite}, to \problem with finite policy space, ignoring the mediator feedback. UCB1 maintains the sample mean $\widehat{J}_t^{\text{MC}}(\xs)$ of the observed returns for each $\xs\in\Theta$ and selects the one that maximizes $\widehat{J}_t^{\text{MC}}(\xs) + \sqrt{{(\alpha\log t)}/{T_t(\xs)}}$. The optimistic bonus favors policies that have been selected less often, in accordance with the OFU principle. Being designed for bandit feedback, UCB1 guarantees $\mathcal{O}(\Delta^{-1}\log n)$ regret~\cite{auer2002finite} even if $v=\infty$, but it cannot exploit mediator feedback when actually present.

In principle, we could employ FTL or UCB1 based on whether $v$ is finite or infinite. There are two reasons why this approach might be inappropriate. First, we would disregard the possibility to share information among pairs of policies with finite divergence, losing possible practical benefits (not captured by the current regret analysis). Second, even when $v<\infty$, the regret of FTL is $\mathcal{O}(v \Delta^{-1} \log( v \Delta^{-2}))$ that, at finite time, might be worse than $\mathcal{O}(\Delta^{-1} \log n)$, especially for large $v$. Note that deriving the conditions on $v$ so that the regret of UCB1 is smaller than that of FLT is not practical since it would require the knowledge of the gap $\Delta$.

\begin{algorithm}[t]
				\begin{small}
					\textbf{Input}: initial parameter $\xs_1$, $\alpha > 1$
					\begin{algorithmic}
					\State Execute $\pi_{\xs_1}$, observe $\outcome_1 \sim p_{\xs_1}$ and $ \func(\outcome_1)$
  					\For{$t = 2,\dots,n$}
						\State Compute expected return estimate $\widecheck{J}_t(\xs)$
						\State Compute index:$$\textstyle B_t(\xs) = \widecheck{J}_t(\xs) + (1+\sqrt{2})\sqrt{\frac{\alpha \log t}{\eta_t(\xs)}}$$
						\State Select $\xs_t \in \argmax_{\xs \in \DecisionSet} B_t(\xs)$
						\State Execute $\pi_{\xs_t}$, observe $\outcome_t \sim p_{\xs_t}$ and $ \func(\outcome_t)$
						\EndFor
					\end{algorithmic}
				\end{small}
\caption{OPTIMIST}\label{alg:optimist}
\end{algorithm}

\paragraph{OPTIMIST}
The difficulty in combining the advantages of FTL and UCB1 is overcome by OPTIMIST (Algorithm~\ref{alg:optimist}), an OFU-based algorithm introduced by~\citet{papini2019optimistic}.\footnote{We consider here a slight variant of OPTIMIST with an explicit exploration parameter $\alpha$ in place of the original confidence parameter $\delta$ from~\cite{papini2019optimistic}, since we focus on expected regret rather than high-probability regret.} It selects policies as to maximize an \emph{optimistic} TMIS expected return estimate that favors policies with a lower effective number of trajectories. In the original paper~\cite{papini2019optimistic}, OPTIMIST is only shown to enjoy sublinear regret in high probability under perfect mediator feedback ($v<\infty$). We show here that OPTIMIST actually enjoys constant regret under perfect mediator feedback (like FTL) without ever degenerating into linear regret (like UCB1).
\begin{restatable}{thr}{optimist}\label{th:optimist}
	Let $\DecisionSet= [K]$ and $v(\xs) = \max_{\xs'\in\DecisionSet}d_{2}(p_{\xs}\|p_{\xs'})$ for all $\xs\in\DecisionSet$ ($v(\xs)$ can be infinite). For any $\alpha>1$, the expected regret of OPTIMIST with truncation $M_t(\xs) = \sqrt{\frac{td_{2}(p_{\xs}\|\Phi_t)}{\alpha\log t}}$ is bounded as:\\
(a)~~if $v\coloneqq \max_{\xs\in\DecisionSet}v(\xs) < \infty$:
		\begin{align*}
			\E R(n) & \leq \sum_{\xs\in\DecisionSet:\Delta(\xs)>0}
				\frac{48\alpha v(\xs)}{\Delta(\xs)}\log\frac{24\alpha v(\xs)}{\Delta(\xs)^2}\\
			& \quad +\Delta(\xs_1)  + \frac{2K}{\alpha-1}\min\left\{1, \sqrt{2\log v}\right\};
		\end{align*}
(b)~~in any case:
		\begin{align*}
			& \E R(n) \leq \sum_{\xs\in\DecisionSet:\Delta(\xs)>0}\frac{24\alpha}{\Delta(\xs)} \log n + \frac{\alpha + 1}{\alpha-1}K,
		\end{align*}
		with an instance-independent expected regret of $ \E R(n)\leq 4\sqrt{6\alpha K n\log n} + {(\alpha + 1)K}/{(\alpha-1)}$.
\end{restatable}
Note also that the regret correctly goes to zero with the divergence (when $v=1$, all the policies are equivalent).
It is an interesting open problem whether better regret guarantees can be provided for the intermediate case, \ie when some (but not all) the \Renyi divergences are finite.

\begin{algorithm}[t]
 \begin{small}
 \textbf{Input}: initial parameter $\xs_1$, scale $a\ge 0$, translation $b \ge 0$, $\alpha > 1$
 \begin{algorithmic}
 \State Execute $\pi_{\xs_1}$, observe $\outcome_1 \sim p_{\xs_1}$ and $\func(\outcome_1)$
  \For{$t = 2,\dots,n$}
  	\State Compute expected return estimate $\widecheck{J}_t(\xs)$
  	\State Generate perturbation:  $$\textstyle U_t(\xs) = \frac{1}{\eta_t(\xs)} \sum_{l=1}^{a\eta_t(\xs)} \outcome_l + b\text{, with }\outcome_l \sim \mathrm{Ber}(1/2)$$
  	\State Select $\xs_t \in \argmax_{\xs \in \DecisionSet}\widecheck{J}_t(\xs) + U_t(\xs)$
  	\State Execute $\pi_{\xs_t}$, observe $\outcome_t \sim p_{\xs_t}$ and $\func(\outcome_t)$
  \EndFor
  \end{algorithmic}
  \end{small}
\caption{\algname}\label{alg:randomist}
\end{algorithm}

\section{Randomized Algorithms}\label{sec:rand}
In this section, we propose a novel algorithm for regret minimization in \problem that selects the policies with a randomized strategy. \algname (\algnameext, Algorithm~\ref{alg:randomist}) is based on PHE~\cite{kveton2019perturbed} and employs additional samples to \emph{perturb} the TMIS expected return estimate $\widecheck{J}_t(\xs)$, enforcing exploration.\footnote{In this sense, \algname, as well as PHE, resembles the Follow the Perturbed Leader~\cite{hannan1957approximation} strategy.} Clearly, \algname shares the randomized nature of exploration with the Bayesian approaches for bandits (\eg Thompson Sampling~\cite{thompson1933likelihood}) although no prior-posterior mechanism is explicitly implemented and no assumption (apart for boundedness) on the return distribution is needed.
At each round $t=2,\dots,n$, we update the TMIS expected return estimate for each policy $\widecheck{J}_t(\xs)$ and we generate the perturbation $U_t(\xs)$ that is obtained through $a \eta_t(\xs)$ \emph{pseudo-rewards} sampled from a Bernoulli distribution $\mathrm{Ber}(1/2)$. Then, we play the policy maximizing the \emph{perturbed estimated expected return}, \ie the sum of the estimated expected return $\widecheck{J}_t(\xs)$ and the perturbation $U_t(\xs)$. The two hyperparameters are the \emph{perturbation scale} $a>0$ and the \emph{perturbation translation} $b>0$. Informally, $a$ and $b$ are responsible for the amount of exploration: $a$ governs the variance of the perturbation, while $b$ (which is absent in PHE) accounts for the negative bias introduced by the TMIS estimator. We now present the properties of \algname with finite parameter space and propose an extension to deal with compact parameter spaces.

\begin{figure*}
\centering
\includegraphics[width=.95\textwidth]{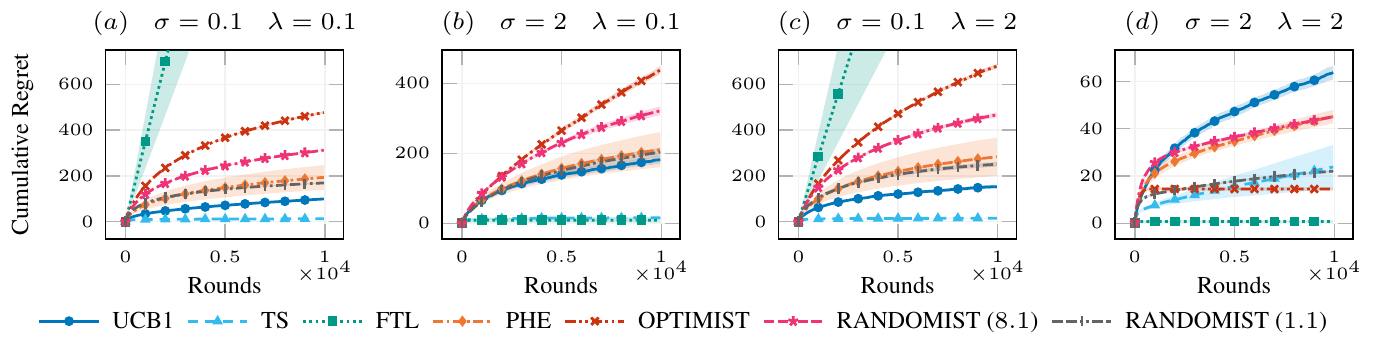}
\caption{Cumulative regret on the illustrative \problem for four values of $\sigma$ and $\lambda$. 20 runs, 95\% c.i.}
\label{fig:resIllustrative}
\end{figure*}

\paragraph{Finite Parameter Space}
If the policy space is finite, we can show that \algname enjoys guarantees similar to those of OPTIMIST on the expected regret.
{\thinmuskip=1mu
\medmuskip=1mu
\thickmuskip=1mu
\begin{restatable}{thr}{randomist}\label{th:randomist}
	Let $\DecisionSet=[K]$, $v(\xs) = \max_{x'\in\DecisionSet}d_{2}(p_{\xs}\|p_{\xs'})$ for all $\xs\in\DecisionSet$ ($v(\xs)$ can be infinite) and $v^*(\xs) = \max\{v(\xs), v(\xs^*)\}$ where $\pi_{\xs^*}$ is an optimal policy. For any $\alpha>1$, the expected regret of \algname with truncation $M_t(\xs) = \sqrt{\frac{td_{2}(p_{\xs}\|\Phi_t)}{\alpha\log t}}$ is bounded as follows:\\
(a)~~if $v\coloneqq \max_{\xs\in\DecisionSet}v(\xs) < \infty$, $b \le \sqrt{(\alpha \log t)/\eta_t(\xs)}$ and $a \ge 0$:
		\begin{align*}
			\E R(n) & \leq  \sum_{\substack{\xs\in\DecisionSet: \Delta(\xs)>0}}\!\!\!\!\!\!\!\!\!\!\!\!\!\!\!\!\!\!\!\!\!\!\!\!\!\!\!\!
				\frac{(188+32a)\alpha v^*(\xs)}{\Delta(\xs)}\log\frac{(94+16a)\alpha v^*(\xs)}{\Delta(\xs)^2} \\
			  & \quad + \Delta(\xs_1) + \frac{\alpha+3}{\alpha-1}\min\left\{1, \sqrt{2\log v}\right\} K;
		\end{align*}
(b)~~no matter the value of $v$, if $a > 8$ and $  J(\xs) - \E[\widecheck{J}_t(\xs)] \le b \le \sqrt{(\alpha \log t)/\eta_t(\xs)}$:
		\begin{align*}
			&\E R(n) \leq \sum_{\xs\in\DecisionSet:\Delta(\xs)>0}\frac{(52+110a)c\alpha}{\Delta(\xs)} \log n + 2\frac{\alpha+1}{\alpha-1}K,
		\end{align*}
		where $c = 2 + \frac{e^2 \sqrt{a}}{\sqrt{2 \pi}} \exp\big[\frac{16}{a-8} \big]  \Big( 1+ \sqrt{\frac{\pi a}{a-8}} \Big)$,
		with an instance-independent regret of $\E R(n)\leq 2 \sqrt{(52+110a)c\alpha Kn \log n} + 2\frac{\alpha+1}{\alpha-1}K$.
\end{restatable}
}
Under perfect mediator feedback \algname enjoys constant regret, like OPTIMIST, although with a dependence on $v^*(\xs)$, which involves the divergence \wrt an optimal policy. Moreover, in such case, since exploration is not needed, we could even set $a=b=0$ reducing \algname to FTL. Similarly to OPTIMIST, when we allow $v=\infty$, the regret becomes logarithmic and the hyperparameters $a$ and $b$ must be carefully set to enforce exploration.

\paragraph{Compact Parameter Space}
When the parameter space is a compact set, \ie $\DecisionSet = [-M,M]^d$, the $\argmax$ in Algorithm~\ref{alg:randomist} cannot be explicitly computed. However, the random variable $\xs\in \argmax_{\xs' \in \DecisionSet} \widecheck{J}_t(\xs) + U_t(\xs)$  can be seen as sampled from the distribution for $\xs$ of being the parameter in $\DecisionSet$ with the largest perturbed estimated expected return, whose p.d.f. is given by~\cite{d2017estimating}:
\begin{align}
\mathfrak{g}_t^*(\xs) &= g \Big( \widecheck{J}_t(\xs) + U_t(\xs) =   \sup_{\xs' \in \DecisionSet} \widecheck{J}_t(\xs') + U_t(\xs') | \Hs_{t-1}\Big) \notag \\
& = \int_{\mathbb{R}} \frac{g_{\xs}(y)}{G_{\xs}(y)} \Prodi_{\DecisionSet} G_{\xs'}(y) \de \xs' \de y,
\end{align}
where $\text{$ \scriptstyle\Prodi $}_{\DecisionSet} G_{\xs}(y) \de \xs = \exp \left( \int_{\DecisionSet} \log G_{\xs}(y) \de \xs \right)$ is the \emph{product integral}~\cite{davis1970concerning}, $g_{\xs}$ and $G_{\xs}$ are the p.d.f. and the c.d.f. of the random variable $\widecheck{J}_t(\xs) + U_t(\xs)$ conditioned to the history $\Hs_{t-1}$. The \emph{computation} of $\mathfrak{g}_t^*$ (even up to a constant) is challenging as the product integral requires a numerical integration over the parameter space $\DecisionSet$. Provided that an approximation (up to a constant) ${\mathfrak{g}}_t^\dagger$ of $\mathfrak{g}_t^*$ is available, we can use a Monte Carlo Markov Chain method~\cite{owen2013monte} to generate a sample $\xs_t \sim {\mathfrak{g}}_t^\dagger$. As a practical approximation, we consider the p.d.f. for $\xs$ of having a perturbed estimated expected return larger than that of the previously executed policies:\footnote{$ {\mathfrak{g}}_t^\dagger$ can be seen as obtained from  ${\mathfrak{g}}_t^*$ applying a quadrature with $\{\xs_1, \dots \xs_{t-1}\}$ as nodes for the inner integral.}
${\mathfrak{g}}_t^\dagger(\xs)  \propto \int_{\mathbb{R}} g_{\xs}(y) \prod_{i=1}^{t-1} G_{\xs_{i}}(y) \de y.$
%
%

Since $\mathcal{O}(d)$ iterations of MCMC are sufficient to generate a sample~\cite{beskos2009computational}, where $d$ is the dimensionality of $\DecisionSet$, and one evaluation of $ {\mathfrak{g}}_t^\dagger$ can be performed in time $\mathcal{O}(t^3)$, the per-round complexity of RANDOMIST is $\mathcal{O}(dt^3)$. This can be further reduced to $\mathcal{O}(dt^2)$ via clever caching (see Appendix~\ref{apx:expDetails}). OPTIMIST~\cite{papini2019optimistic} can also be applied to continuous parameter spaces, with an $\widetilde{\mathcal{O}}(\sqrt{vdn})$ high-probability regret bound. However, it is not clear how to perform the maximization step of OPTIMIST efficiently in this setting, since the optimistic index is non-differentiable and non-convex in the parameter variable.
Discretization is adopted in~\citep{papini2019optimistic}, leading to $\mathcal{O}(t^{1+d/2})$ time complexity, that is exponential in $d$.
The RANDOMIST variant proposed here, although heuristic, has only polynomial dependence on $d$, thus scaling more favorably to high-dimensional problems.

\section{Numerical Simulations}\label{sec:exp}
We present the numerical simulations, starting with an illustrative example and then moving to RL benchmarks.
For the RL experiments, similarly to~\citet{papini2019optimistic}, the evaluation is carried out in the parameter-based PO setting~\cite{sehnke2008policy}, where the policy parameters $\vtheta$ are sampled from a \emph{hyperpolicy} $\nu_{\vxi}$ and the optimization is performed in the space of \emph{hyperparameters} $\Xi$ (Appendix~\ref{apx:pbpo}). This setting is particularly convenient since the \Renyi divergence between hyperpolicies can be computed exactly (at least for Gaussians). 
Details and an additional experiment on the Cartpole domain are reported in Appendix~\ref{apx:expDetails}.

\paragraph{Illustrative Problems}
The goal of this experiment is to show the advantages of the additional structure offered by the mediator feedback over the bandit feedback. We design a class of 5-policy \problem problems, isomorphic to bandit problems, in which trajectories are collapsed to a single real action $\OutcomeSet = \mathbb{R}$ and $\func(\tau) = \max\{0, \min\{1, \tau/4 \} \}$. The policies are Gaussians $(\mathcal{N}(0,\sigma^2)$, $\mathcal{N}(1,\sigma^2)$, $\mathcal{N}(2,\sigma^2)$, $\mathcal{N}(2.95,\lambda^2)$, $\mathcal{N}(3,\sigma^2))$ defined in terms of the two values $\sigma,\lambda > 0$. The optimal policy is the fifth one and we have a near-optimal parameter, the fourth, with a different variance. Intuitively, we can tune the parameters $\sigma$ and $\lambda$ to vary the \Renyi divergences. We compare \algname with $a=8.1$ (as prescribed in Theorem~\ref{th:randomist}) and $a=1.1$, and $b= \sqrt{(\alpha \log t)/\eta_t(\xs)}$ for both cases, with OPTIMIST~\cite{papini2019optimistic}, FTL, UCB1~\cite{auer2002finite}, PHE~\cite{kveton2019perturbed}, and TS with Gaussian prior~\cite{agrawal2013further}. The cumulative regret is shown in Figure~\ref{fig:resIllustrative} for four combinations of $\sigma$ and $\lambda$. In (a) and (d) we are in a perfect mediator feedback, but in (a) $\log v \simeq 2.25$ and (d) $\log v \simeq 900$. Instead, in (b) or (c), we have $v = \infty$. We notice that FTL displays a \linebreak(near-)linear regret in (a) as expected since $v=\infty$ but also in (c) where $v$ is finite but very large. \algname with theoretical value of $a=8.1$ always displays a good behavior and better than OPTIMIST, except in (d) where the latter shows a remarkable constant regret. We also note that when the amount of information shared among parameters is small, UCB1 performs better than OPTIMIST as well as PHE over \algname. 
Furthermore, TS with Gaussian prior performs very well across the tasks, although it considers the bandit feedback. This can be explained since TS assumes the correct return distribution. It also suggests that \algname could be improved when coped with other perturbation distributions (\eg Gaussian). Finally, we observe that \algname with $a=1.1$, although violating the conditions of Theorem~\ref{th:randomist}, keeps showing a sublinear regret even in (b) and (c) when $v=\infty$.


\begin{figure}[t]
\includegraphics[width=0.48\textwidth]{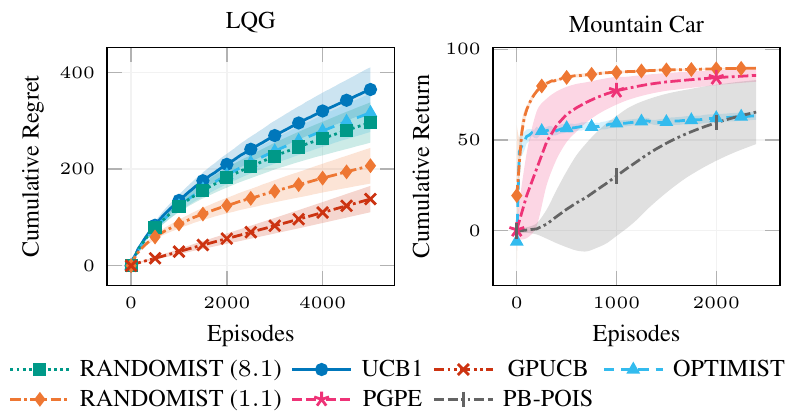}
\caption{Cumulative regret in the LQG (30 runs, 95\% c.i.) and cumulative return in the Mountain Car (5 runs, 95\% c.i.).}\label{fig:lqg}
\end{figure}

\paragraph{Linear Quadratic Gaussian Regulator}
The Linear Quadratic Gaussian Regulator~\citep[LQG,][]{dorato1995lqg} is a benchmark for continuous control. We consider the monodimensional case and a Gaussian hyperpolicy $\nu_{\xi} = \mathcal{N}(\xi, 0.15^2)$ where $\xi$ is the learned parameter. From $\nu_{\xi}$, we sample the gain $\theta$ of a deterministic linear policy: $a_h = \theta s_h$. This experiment aims at comparing \algname with UCB1~\cite{auer2002finite}, GPUCB~\cite{srinivas2010gaussian}, and OPTIMIST~\cite{papini2019optimistic} in a finite policy space by discretizing $[-1,1]$ in $K=100$ parameters. In Figure~\ref{fig:lqg}, we notice that OPTIMIST and \algname outperform UCB1. While \algname with $a=8.1$ and OPTIMIST have similar performance, \algname improves significantly when setting $a$ to $1.1$. As in ~\cite{papini2019optimistic}, the good performance of GPUCB is paired with a lack of theoretical guarantees due to the arbitrary choice of the GP kernel.

\paragraph{Mountain Car}
To test \algname in a continuous parameter space, we employ the approximation described above in the Mountain Car environment~\cite{sutton2018reinforcement}. We consider the setting of~\cite{papini2019optimistic}, employing PGPE~\cite{sehnke2008policy} and PB-POIS~\cite{metelli2018policy} as baselines. We use a Gaussian hyperpolicy $\nu_{\bm{\xi}} = \mathcal{N}(\bm{\xi}, \mathrm{diag}(0.15, 3)^2)$ with learned mean $\bm{\xi}$, from which we sample the parameters of a deterministic policy, linear in position and velocity. The exploration phase is performed by sampling from the approximate density $\mathfrak{g}_t^\dagger$, taking $10$ steps of the Metropolis-Hastings algorithm~\cite{owen2013monte} with Gaussian proposal $q_m = \mathcal{N}(\xs_m, \mathrm{diag}(0.15, 3)^2)$.
Figure~\ref{fig:lqg} shows that \algname outperforms both policy gradient baselines and OPTIMIST, in terms of learning speed and final performance.


\section{Related Works}\label{sec:related}
In this section, we revise the related literature, with attention to bandits with expert advice and to provably efficient PO. Additional comparisons are reported in Appendix~\ref{apx:relatedWorks}.
\paragraph{Mediator Feedback and Expert Advice}
A related formulation are the \emph{Bandits with Expert Advice}~\citep[BEA,][Section 4.2]{bubeck2012regret}, introduced as an approach to adversarial contextual bandits. To draw a parallelism with \problem, let $\OutcomeSet$ be the set of arms and $\DecisionSet = [K]$ the finite set of experts. At each step $t$, the agent receives \emph{advice} $p^t_{\xs}\in \mathscr{P}(\OutcomeSet)$ from each expert $\xs\in\DecisionSet$, selects one expert $\xs_t$, and pulls arm $\outcome_t\sim p_{\xs_t}^t$. The goal is to minimize the \emph{in-class} regret, competing with the best expert in hindsight. Differently from the trajectory distributions of \problem, expert advice can change with time.
A major concern of BEA, also relevant to PO, is the dependency of the regret on the number $K$ of experts (resp. policies). A na\"ive application of Exp3~\citep{auer2002nonstochastic} yields $\mathcal{O}(\sqrt{nK\log K})$ regret. Like our \problem algorithms, this is impractical when the experts are exponentially many. Exp4~\citep{auer2002nonstochastic} achieves  $\mathcal{O}(\sqrt{n|\OutcomeSet|\log K})$ regret, which scales well with $K$, but is vacuous in the case of infinite arms.~\citet{mcmahan2009tighter} replace $|\OutcomeSet|$ with the \emph{degree of agreement} of the experts, which has interesting similarities with our distributional-divergence approach.
\emph{Meta-bandit} approaches~\citep{agarwal2017corralling,pacchiano2020model} are so general that could be applied both to continuous-arm BEA and \problem, but also exhibit a superlogarithmic dependence on $K$. ~\citet{beygelzimer2011contextual} obtain $\widetilde{\mathcal{O}}(\sqrt{dn})$ 
regret competing with an infinite set of experts of VC-dimension $d$, mirrored in PO by OPTIMIST on compact spaces of dimension $d$~\citep[][Theorem 3]{papini2019optimistic}.

\paragraph{Provably Efficient PO} Recently, a surge of approaches to deal with PO in a theoretically sound way, with both stochastic or adversarial environments, has emerged. These works consider either \emph{full-information}, \ie the agent observes the whole reward function $\{\mathcal{R}(s_h,a)\}_{a \in \mathcal{A}}$ regardless the played action~\citep[\eg][]{rosenberg2019online, cai2019provably}, or the \emph{bandit feedback} (with a different meaning compared to the use we have made in this paper), in which only the reward of the chosen action is observed $\mathcal{R}(s_h,a_h)$~\citep[\eg][]{jin2019learning,efroni2020optimistic}. These methods are not directly comparable with the mediator feedback, although both settings exploit the structure of the PO problem. While with MF we explicitly model the policy space $\Pi_{\Theta}$, these methods search in the space of all Markovian stationary policies. Furthermore, they are limited to tabular MDPs, while MF can deal natively with continuous state-action spaces.

\section{Discussion and Conclusions}\label{sec:discussion}
We have deepened the understanding of policy optimization as an online learning problem with additional feedback. We believe that mediator feedback has potential applications even beyond PO. Indeed, the problem of optimizing over probability distributions also encompasses GANs and variational inference~\citep{chu2019probability} and, more generally, MF emerges in any Bayesian network in which we control the conditional distributions on some vertexes, via parameters $\xs$, while the other are fixed and independent from $\xs$.
%
%
Furthermore, we have introduced a novel randomized algorithm, RANDOMIST, and we have shown its advantages both in terms of computational complexity and performance. The algorithm could be improved by adopting a different perturbation, \eg Gaussian, as already hinted in~\cite{kveton2019garbage}. Further work is needed to match the theoretical regret lower bounds.
Currently, a major discrepancy is the use of the KL-divergence in the lower bounds instead of the larger \Renyi divergence required by algorithms based on IS. \HL{Moreover, the algorithm employs the ratio importance weight and, thus, it might suffer from the curse of horizon~\citep{liu2018breaking}.} Finally, the case of non-perfect mediator feedback could be related to graphical bandits~\cite{alon2017nonstochastic}, where finite \Renyi divergences are the edges of a directed feedback graph, in order to capture the actual difficulty of this intermediate case.

\clearpage

%

\section*{Acknowledgments}
This work has been partially supported by the Italian MIUR PRIN 2017 Project ALGADIMAR "Algorithms, Games, and Digital Markets".

{\small
\bibliography{pomits}

\begin{thebibliography}{81}
\providecommand{\natexlab}[1]{#1}
\providecommand{\url}[1]{\texttt{#1}}
\providecommand{\urlprefix}{URL }
\expandafter\ifx\csname urlstyle\endcsname\relax
  \providecommand{\doi}[1]{doi:\discretionary{}{}{}#1}\else
  \providecommand{\doi}{doi:\discretionary{}{}{}\begingroup
  \urlstyle{rm}\Url}\fi

\bibitem[{Abbasi-Yadkori, P{\'a}l, and
  Szepesv{\'a}ri(2011)}]{abbasi2011improved}
Abbasi-Yadkori, Y.; P{\'a}l, D.; and Szepesv{\'a}ri, C. 2011.
\newblock Improved algorithms for linear stochastic bandits.
\newblock In \emph{NeurIPS}.

\bibitem[{Abeille and Lazaric(2017{\natexlab{a}})}]{abeille2017linear}
Abeille, M.; and Lazaric, A. 2017{\natexlab{a}}.
\newblock Linear thompson sampling revisited.
\newblock \emph{Electronic Journal of Statistics} 11(2): 5165--5197.

\bibitem[{Abeille and Lazaric(2017{\natexlab{b}})}]{abeille2017thompson}
Abeille, M.; and Lazaric, A. 2017{\natexlab{b}}.
\newblock Thompson Sampling for Linear-Quadratic Control Problems.
\newblock In \emph{AISTATS}.

\bibitem[{Abeille and Lazaric(2018)}]{abeille2018improved}
Abeille, M.; and Lazaric, A. 2018.
\newblock Improved regret bounds for thompson sampling in linear quadratic
  control problems.
\newblock In \emph{ICML}.

\bibitem[{Agarwal et~al.(2017)Agarwal, Luo, Neyshabur, and
  Schapire}]{agarwal2017corralling}
Agarwal, A.; Luo, H.; Neyshabur, B.; and Schapire, R.~E. 2017.
\newblock Corralling a Band of Bandit Algorithms.
\newblock In \emph{{COLT}}.

\bibitem[{Agrawal and Goyal(2012)}]{agrawal2012analysis}
Agrawal, S.; and Goyal, N. 2012.
\newblock Analysis of Thompson Sampling for the Multi-armed Bandit Problem.
\newblock In \emph{{COLT}}.

\bibitem[{Agrawal and Goyal(2013{\natexlab{a}})}]{agrawal2013further}
Agrawal, S.; and Goyal, N. 2013{\natexlab{a}}.
\newblock Further Optimal Regret Bounds for Thompson Sampling.
\newblock In \emph{{AISTATS}}.

\bibitem[{Agrawal and Goyal(2013{\natexlab{b}})}]{agrawal2013thompson}
Agrawal, S.; and Goyal, N. 2013{\natexlab{b}}.
\newblock Thompson sampling for contextual bandits with linear payoffs.
\newblock In \emph{ICML}.

\bibitem[{Alon et~al.(2015)Alon, Cesa-Bianchi, Dekel, and
  Koren}]{alon2015online}
Alon, N.; Cesa-Bianchi, N.; Dekel, O.; and Koren, T. 2015.
\newblock Online learning with feedback graphs: Beyond bandits.
\newblock In \emph{COLT}.

\bibitem[{Alon et~al.(2017)Alon, Cesa{-}Bianchi, Gentile, Mannor, Mansour, and
  Shamir}]{alon2017nonstochastic}
Alon, N.; Cesa{-}Bianchi, N.; Gentile, C.; Mannor, S.; Mansour, Y.; and Shamir,
  O. 2017.
\newblock Nonstochastic Multi-Armed Bandits with Graph-Structured Feedback.
\newblock \emph{{SIAM} J. Comput.} 46(6): 1785--1826.

\bibitem[{Auer, Cesa{-}Bianchi, and Fischer(2002)}]{auer2002finite}
Auer, P.; Cesa{-}Bianchi, N.; and Fischer, P. 2002.
\newblock Finite-time Analysis of the Multiarmed Bandit Problem.
\newblock \emph{Mach. Learn.} 47(2-3): 235--256.

\bibitem[{Auer et~al.(2002)Auer, Cesa{-}Bianchi, Freund, and
  Schapire}]{auer2002nonstochastic}
Auer, P.; Cesa{-}Bianchi, N.; Freund, Y.; and Schapire, R.~E. 2002.
\newblock The Nonstochastic Multiarmed Bandit Problem.
\newblock \emph{{SIAM} J. Comput.} 32(1): 48--77.

\bibitem[{Beskos and Stuart(2009)}]{beskos2009computational}
Beskos, A.; and Stuart, A. 2009.
\newblock Computational complexity of Metropolis-Hastings methods in high
  dimensions.
\newblock In \emph{Monte Carlo and Quasi-Monte Carlo Methods 2008}, 61--71.

\bibitem[{Beygelzimer et~al.(2011)Beygelzimer, Langford, Li, Reyzin, and
  Schapire}]{beygelzimer2011contextual}
Beygelzimer, A.; Langford, J.; Li, L.; Reyzin, L.; and Schapire, R.~E. 2011.
\newblock Contextual Bandit Algorithms with Supervised Learning Guarantees.
\newblock In \emph{{AISTATS}}.

\bibitem[{Bubeck and Cesa{-}Bianchi(2012)}]{bubeck2012regret}
Bubeck, S.; and Cesa{-}Bianchi, N. 2012.
\newblock Regret Analysis of Stochastic and Nonstochastic Multi-armed Bandit
  Problems.
\newblock \emph{Foundations and Trends in Machine Learning} 5(1): 1--122.

\bibitem[{Bubeck, Eldan, and Lehec(2018)}]{bubeck2018sampling}
Bubeck, S.; Eldan, R.; and Lehec, J. 2018.
\newblock Sampling from a log-concave distribution with projected langevin
  monte carlo.
\newblock \emph{Discrete \& Computational Geometry} 59(4): 757--783.

\bibitem[{Bubeck, Perchet, and Rigollet(2013)}]{bubeck2013bounded}
Bubeck, S.; Perchet, V.; and Rigollet, P. 2013.
\newblock Bounded regret in stochastic multi-armed bandits.
\newblock In \emph{{COLT}}.

\bibitem[{Cai et~al.(2019)Cai, Yang, Jin, and Wang}]{cai2019provably}
Cai, Q.; Yang, Z.; Jin, C.; and Wang, Z. 2019.
\newblock Provably Efficient Exploration in Policy Optimization.
\newblock \emph{arXiv preprint arXiv:1912.05830} .

\bibitem[{Caron et~al.(2012)Caron, Kveton, Lelarge, and
  Bhagat}]{caron2012leveraging}
Caron, S.; Kveton, B.; Lelarge, M.; and Bhagat, S. 2012.
\newblock Leveraging Side Observations in Stochastic Bandits.
\newblock In \emph{{UAI}}.

\bibitem[{Casella and George(1992)}]{casella1992explaining}
Casella, G.; and George, E.~I. 1992.
\newblock Explaining the Gibbs sampler.
\newblock \emph{The American Statistician} 46(3): 167--174.

\bibitem[{Chapelle and Li(2011)}]{chapelle2011empirical}
Chapelle, O.; and Li, L. 2011.
\newblock An Empirical Evaluation of Thompson Sampling.
\newblock In \emph{NeurIPS}.

\bibitem[{Chen et~al.(2016)Chen, Wang, Yuan, and Wang}]{chen2016combinatorial}
Chen, W.; Wang, Y.; Yuan, Y.; and Wang, Q. 2016.
\newblock Combinatorial multi-armed bandit and its extension to
  probabilistically triggered arms.
\newblock \emph{JMLR} 17(1): 1746--1778.

\bibitem[{Cheng and Bartlett(2018)}]{cheng2018convergence}
Cheng, X.; and Bartlett, P. 2018.
\newblock Convergence of Langevin MCMC in KL-divergence.
\newblock In \emph{ALT}.

\bibitem[{Chowdhury and Gopalan(2017)}]{chowdhury2017kernelized}
Chowdhury, S.~R.; and Gopalan, A. 2017.
\newblock On kernelized multi-armed bandits.
\newblock In \emph{ICML}.

\bibitem[{Chowdhury and Gopalan(2019)}]{chowdhury2019online}
Chowdhury, S.~R.; and Gopalan, A. 2019.
\newblock Online Learning in Kernelized Markov Decision Processes.
\newblock In \emph{{AISTATS}}.

\bibitem[{Chu, Blanchet, and Glynn(2019)}]{chu2019probability}
Chu, C.; Blanchet, J.~H.; and Glynn, P.~W. 2019.
\newblock Probability Functional Descent: {A} Unifying Perspective on GANs,
  Variational Inference, and Reinforcement Learning.
\newblock In Chaudhuri, K.; and Salakhutdinov, R., eds., \emph{ICML}.

\bibitem[{Chung and Lu(2006)}]{chung2006old}
Chung, F.; and Lu, L. 2006.
\newblock Old and new concentration inequalities.
\newblock \emph{Complex Graphs and Networks} 107: 23--56.

\bibitem[{Cochran(1977)}]{cochran77sampling}
Cochran, W.~G. 1977.
\newblock \emph{Sampling Techniques, 3rd Edition}.
\newblock John Wiley.
\newblock ISBN 0-471-16240-X.

\bibitem[{Combes, Magureanu, and Proutiere(2017)}]{combes2017minimal}
Combes, R.; Magureanu, S.; and Proutiere, A. 2017.
\newblock Minimal exploration in structured stochastic bandits.
\newblock In \emph{NeurIPS}.

\bibitem[{Corless et~al.(1996)Corless, Gonnet, Hare, Jeffrey, and
  Knuth}]{corless1996lambertw}
Corless, R.~M.; Gonnet, G.~H.; Hare, D.~E.; Jeffrey, D.~J.; and Knuth, D.~E.
  1996.
\newblock On the LambertW function.
\newblock \emph{Advances in Computational mathematics} 5(1): 329--359.

\bibitem[{Cortes, Mansour, and Mohri(2010)}]{cortes2010learning}
Cortes, C.; Mansour, Y.; and Mohri, M. 2010.
\newblock Learning Bounds for Importance Weighting.
\newblock In \emph{NeurIPS}.

\bibitem[{Curtain(1997)}]{dorato1995lqg}
Curtain, R.~F. 1997.
\newblock Linear-quadratic control: An introduction.
\newblock \emph{Autom.} 33(5): 1004.

\bibitem[{Dani, Hayes, and Kakade(2008)}]{dani2008stochastic}
Dani, V.; Hayes, T.~P.; and Kakade, S.~M. 2008.
\newblock Stochastic Linear Optimization under Bandit Feedback.
\newblock In \emph{COLT}.

\bibitem[{Davis and Chatfield(1970)}]{davis1970concerning}
Davis, W.; and Chatfield, J. 1970.
\newblock Concerning product integrals and exponentials.
\newblock \emph{AMS} .

\bibitem[{Dean et~al.(2018)Dean, Mania, Matni, Recht, and Tu}]{dean2018regret}
Dean, S.; Mania, H.; Matni, N.; Recht, B.; and Tu, S. 2018.
\newblock Regret bounds for robust adaptive control of the linear quadratic
  regulator.
\newblock In \emph{NeurIPS}.

\bibitem[{Deisenroth, Neumann, and Peters(2013)}]{deisenroth2013survery}
Deisenroth, M.~P.; Neumann, G.; and Peters, J. 2013.
\newblock A Survey on Policy Search for Robotics.
\newblock \emph{Foundations and Trends in Robotics} 2(1-2): 1--142.

\bibitem[{D'Eramo et~al.(2017)D'Eramo, Nuara, Pirotta, and
  Restelli}]{d2017estimating}
D'Eramo, C.; Nuara, A.; Pirotta, M.; and Restelli, M. 2017.
\newblock Estimating the maximum expected value in continuous reinforcement
  learning problems.
\newblock In \emph{AAAI}.

\bibitem[{Eckles and Kaptein(2014)}]{eckles2014thompson}
Eckles, D.; and Kaptein, M. 2014.
\newblock Thompson sampling with the online bootstrap.
\newblock \emph{arXiv preprint arXiv:1410.4009} .

\bibitem[{Efroni et~al.(2020)Efroni, Shani, Rosenberg, and
  Mannor}]{efroni2020optimistic}
Efroni, Y.; Shani, L.; Rosenberg, A.; and Mannor, S. 2020.
\newblock Optimistic Policy Optimization with Bandit Feedback.
\newblock \emph{arXiv preprint arXiv:2002.08243} .

\bibitem[{Gil, Alajaji, and Linder(2013)}]{gil2013renyi}
Gil, M.; Alajaji, F.; and Linder, T. 2013.
\newblock R{\'e}nyi divergence measures for commonly used univariate continuous
  distributions.
\newblock \emph{Information Sciences} 249: 124--131.

\bibitem[{Grant and Leslie(2020)}]{grant2020thompson}
Grant, J.~A.; and Leslie, D.~S. 2020.
\newblock On Thompson Sampling for Smoother-than-Lipschitz Bandits.
\newblock \emph{arXiv preprint arXiv:2001.02323} .

\bibitem[{Hannan(1957)}]{hannan1957approximation}
Hannan, J. 1957.
\newblock Approximation to Bayes risk in repeated play.
\newblock \emph{Contributions to the Theory of Games} 3: 97--139.

\bibitem[{Ionides(2008)}]{ionides2008truncated}
Ionides, E.~L. 2008.
\newblock Truncated importance sampling.
\newblock \emph{JCGS} 17(2): 295--311.

\bibitem[{Jaksch, Ortner, and Auer(2010)}]{jaksch2010near}
Jaksch, T.; Ortner, R.; and Auer, P. 2010.
\newblock Near-optimal Regret Bounds for Reinforcement Learning.
\newblock \emph{J. Mach. Learn. Res.} 11: 1563--1600.

\bibitem[{Jin et~al.(2018)Jin, Allen-Zhu, Bubeck, and Jordan}]{jin2018q}
Jin, C.; Allen-Zhu, Z.; Bubeck, S.; and Jordan, M.~I. 2018.
\newblock Is q-learning provably efficient?
\newblock In \emph{NeurIPS}.

\bibitem[{Jin et~al.(2019)Jin, Jin, Luo, Sra, and Yu}]{jin2019learning}
Jin, C.; Jin, T.; Luo, H.; Sra, S.; and Yu, T. 2019.
\newblock Learning adversarial markov decision processes with bandit feedback
  and unknown transition.
\newblock \emph{arXiv preprint arXiv:1912.01192} .

\bibitem[{Kallus(2018)}]{kallus2018instrument}
Kallus, N. 2018.
\newblock Instrument-Armed Bandits.
\newblock In \emph{Algorithmic Learning Theory}, 529--546.

\bibitem[{Kaufmann, Korda, and Munos(2012)}]{kaufmann2012thompson}
Kaufmann, E.; Korda, N.; and Munos, R. 2012.
\newblock Thompson Sampling: An Asymptotically Optimal Finite-Time Analysis.
\newblock In \emph{{ALT}}.

\bibitem[{Kveton et~al.(2019{\natexlab{a}})Kveton, Szepesv{\'{a}}ri,
  Ghavamzadeh, and Boutilier}]{kveton2019perturbed}
Kveton, B.; Szepesv{\'{a}}ri, C.; Ghavamzadeh, M.; and Boutilier, C.
  2019{\natexlab{a}}.
\newblock Perturbed-History Exploration in Stochastic Multi-Armed Bandits.
\newblock In \emph{{IJCAI}}.

\bibitem[{Kveton et~al.(2019{\natexlab{b}})Kveton, Szepesv{\'{a}}ri, Vaswani,
  Wen, Lattimore, and Ghavamzadeh}]{kveton2019garbage}
Kveton, B.; Szepesv{\'{a}}ri, C.; Vaswani, S.; Wen, Z.; Lattimore, T.; and
  Ghavamzadeh, M. 2019{\natexlab{b}}.
\newblock Garbage In, Reward Out: Bootstrapping Exploration in Multi-Armed
  Bandits.
\newblock In \emph{{ICML}}.

\bibitem[{Lai and Robbins(1985)}]{lai1985asymptotically}
Lai, T.~L.; and Robbins, H. 1985.
\newblock Asymptotically efficient adaptive allocation rules.
\newblock \emph{Advances in applied mathematics} 6(1): 4--22.

\bibitem[{Lattimore and Munos(2014)}]{lattimore2014bounded}
Lattimore, T.; and Munos, R. 2014.
\newblock Bounded Regret for Finite-Armed Structured Bandits.
\newblock In \emph{NeurIPS}.

\bibitem[{Lattimore and Szepesv{\'a}ri(2018)}]{lattimore2018bandit}
Lattimore, T.; and Szepesv{\'a}ri, C. 2018.
\newblock Bandit algorithms .

\bibitem[{Liu et~al.(2018)Liu, Li, Tang, and Zhou}]{liu2018breaking}
Liu, Q.; Li, L.; Tang, Z.; and Zhou, D. 2018.
\newblock Breaking the Curse of Horizon: Infinite-Horizon Off-Policy
  Estimation.
\newblock In \emph{NeurIPS}, 5361--5371.

\bibitem[{Lu and Van~Roy(2017)}]{lu2017ensemble}
Lu, X.; and Van~Roy, B. 2017.
\newblock Ensemble sampling.
\newblock In \emph{NeurIPS}.

\bibitem[{Magureanu, Combes, and Proutiere(2014)}]{magureanu2014lipschitz}
Magureanu, S.; Combes, R.; and Proutiere, A. 2014.
\newblock Lipschitz bandits: Regret lower bounds and optimal algorithms.
\newblock In \emph{COLT}.

\bibitem[{Martino, Elvira, and Louzada(2017)}]{martino2017effective}
Martino, L.; Elvira, V.; and Louzada, F. 2017.
\newblock Effective sample size for importance sampling based on discrepancy
  measures.
\newblock \emph{Signal Processing} 131: 386--401.

\bibitem[{McMahan and Streeter(2009)}]{mcmahan2009tighter}
McMahan, H.~B.; and Streeter, M.~J. 2009.
\newblock Tighter Bounds for Multi-Armed Bandits with Expert Advice.
\newblock In \emph{{COLT}}.

\bibitem[{Metelli et~al.(2018)Metelli, Papini, Faccio, and
  Restelli}]{metelli2018policy}
Metelli, A.~M.; Papini, M.; Faccio, F.; and Restelli, M. 2018.
\newblock Policy Optimization via Importance Sampling.
\newblock In \emph{NeurIPS}.

\bibitem[{Metelli et~al.(2020)Metelli, Papini, Montali, and
  Restelli}]{metelli2020importance}
Metelli, A.~M.; Papini, M.; Montali, N.; and Restelli, M. 2020.
\newblock Importance Sampling Techniques for Policy Optimization.
\newblock \emph{JMLR} 21(141): 1--75.

\bibitem[{Osband, Russo, and Van~Roy(2013)}]{osband2013more}
Osband, I.; Russo, D.; and Van~Roy, B. 2013.
\newblock (More) efficient reinforcement learning via posterior sampling.
\newblock In \emph{NeurIPS}.

\bibitem[{Osband et~al.(2017)Osband, Russo, Wen, and Van~Roy}]{osband2017deep}
Osband, I.; Russo, D.; Wen, Z.; and Van~Roy, B. 2017.
\newblock Deep exploration via randomized value functions.
\newblock \emph{JMLR} .

\bibitem[{Owen(2013)}]{owen2013monte}
Owen, A.~B. 2013.
\newblock Monte Carlo theory, methods and examples.
\newblock \emph{Monte Carlo Theory, Methods and Examples} .

\bibitem[{Pacchiano et~al.(2020)Pacchiano, Phan, Abbasi{-}Yadkori, Rao,
  Zimmert, Lattimore, and Szepesv{\'{a}}ri}]{pacchiano2020model}
Pacchiano, A.; Phan, M.; Abbasi{-}Yadkori, Y.; Rao, A.; Zimmert, J.; Lattimore,
  T.; and Szepesv{\'{a}}ri, C. 2020.
\newblock Model Selection in Contextual Stochastic Bandit Problems.
\newblock \emph{CoRR} abs/2003.01704.

\bibitem[{Papini et~al.(2019)Papini, Metelli, Lupo, and
  Restelli}]{papini2019optimistic}
Papini, M.; Metelli, A.~M.; Lupo, L.; and Restelli, M. 2019.
\newblock Optimistic Policy Optimization via Multiple Importance Sampling.
\newblock In \emph{{ICML}}.

\bibitem[{Peng et~al.(2020)Peng, Coumans, Zhang, Lee, Tan, and
  Levine}]{peng2020learning}
Peng, X.~B.; Coumans, E.; Zhang, T.; Lee, T.-W.; Tan, J.; and Levine, S. 2020.
\newblock Learning Agile Robotic Locomotion Skills by Imitating Animals.
\newblock \emph{arXiv preprint arXiv:2004.00784} .

\bibitem[{Peters and Schaal(2008)}]{peters2008reinforcement}
Peters, J.; and Schaal, S. 2008.
\newblock Reinforcement learning of motor skills with policy gradients.
\newblock \emph{Neural Networks} 21(4): 682--697.

\bibitem[{Phan, Yadkori, and Domke(2019)}]{phan2019thompson}
Phan, M.; Yadkori, Y.~A.; and Domke, J. 2019.
\newblock Thompson Sampling and Approximate Inference.
\newblock In \emph{NeurIPS}.

\bibitem[{Puterman(1994)}]{puterman1994markov}
Puterman, M.~L. 1994.
\newblock \emph{Markov Decision Processes: Discrete Stochastic Dynamic
  Programming}.
\newblock Wiley.

\bibitem[{R{\'e}nyi(1961)}]{renyi1961measures}
R{\'e}nyi, A. 1961.
\newblock On measures of entropy and information.
\newblock In \emph{Proceedings of the Fourth Berkeley Symposium on Mathematical
  Statistics and Probability}.

\bibitem[{Rosenberg and Mansour(2019)}]{rosenberg2019online}
Rosenberg, A.; and Mansour, Y. 2019.
\newblock Online Convex Optimization in Adversarial Markov Decision Processes.
\newblock In \emph{ICML}.

\bibitem[{Rusmevichientong and Tsitsiklis(2010)}]{rusmevichientong2010linearly}
Rusmevichientong, P.; and Tsitsiklis, J.~N. 2010.
\newblock Linearly parameterized bandits.
\newblock \emph{Mathematics of Operations Research} 35(2): 395--411.

\bibitem[{Russo et~al.(2018)Russo, Van~Roy, Kazerouni, Osband, and
  Wen}]{russo2018tutorial}
Russo, D.~J.; Van~Roy, B.; Kazerouni, A.; Osband, I.; and Wen, Z. 2018.
\newblock A Tutorial on Thompson Sampling.
\newblock \emph{Foundations and Trends{\textregistered} in Machine Learning}
  11(1): 1--96.

\bibitem[{Sehnke et~al.(2008)Sehnke, Osendorfer, R{\"{u}}ckstie{\ss}, Graves,
  Peters, and Schmidhuber}]{sehnke2008policy}
Sehnke, F.; Osendorfer, C.; R{\"{u}}ckstie{\ss}, T.; Graves, A.; Peters, J.;
  and Schmidhuber, J. 2008.
\newblock Policy Gradients with Parameter-Based Exploration for Control.
\newblock In \emph{{ICANN}}.

\bibitem[{Srinivas et~al.(2010)Srinivas, Krause, Kakade, and
  Seeger}]{srinivas2010gaussian}
Srinivas, N.; Krause, A.; Kakade, S.~M.; and Seeger, M.~W. 2010.
\newblock Gaussian Process Optimization in the Bandit Setting: No Regret and
  Experimental Design.
\newblock In \emph{ICML}.

\bibitem[{Strehl and Littman(2008)}]{strehl2008analysis}
Strehl, A.~L.; and Littman, M.~L. 2008.
\newblock An analysis of model-based interval estimation for Markov decision
  processes.
\newblock \emph{JCSS} 74(8): 1309--1331.

\bibitem[{Sutton and Barto(2018)}]{sutton2018reinforcement}
Sutton, R.~S.; and Barto, A.~G. 2018.
\newblock \emph{Reinforcement learning: An introduction}.
\newblock MIT press.

\bibitem[{Tange(2011)}]{Tange2011a}
Tange, O. 2011.
\newblock GNU Parallel - The Command-Line Power Tool.
\newblock \emph{;login: The USENIX Magazine} 36(1): 42--47.

\bibitem[{Thompson(1933)}]{thompson1933likelihood}
Thompson, W.~R. 1933.
\newblock On the likelihood that one unknown probability exceeds another in
  view of the evidence of two samples.
\newblock \emph{Biometrika} 25(3/4): 285--294.

\bibitem[{Veach and Guibas(1995)}]{veach1995optimally}
Veach, E.; and Guibas, L.~J. 1995.
\newblock Optimally combining sampling techniques for Monte Carlo rendering.
\newblock In Mair, S.~G.; and Cook, R., eds., \emph{{SIGGRAPH}}, 419--428.
  {ACM}.

\bibitem[{Vinyals et~al.(2019)Vinyals, Babuschkin, Czarnecki, Mathieu, Dudzik,
  Chung, Choi, Powell, Ewalds, Georgiev, Oh, Horgan, Kroiss, Danihelka, Huang,
  Sifre, Cai, Agapiou, Jaderberg, and Silver}]{vinyals2019grandmaster}
Vinyals, O.; Babuschkin, I.; Czarnecki, W.; Mathieu, M.; Dudzik, A.; Chung, J.;
  Choi, D.; Powell, R.; Ewalds, T.; Georgiev, P.; Oh, J.; Horgan, D.; Kroiss,
  M.; Danihelka, I.; Huang, A.; Sifre, L.; Cai, T.; Agapiou, J.; Jaderberg, M.;
  and Silver, D. 2019.
\newblock Grandmaster level in StarCraft II using multi-agent reinforcement
  learning.
\newblock \emph{Nature} 575(7782): 350--354.

\end{thebibliography}
}

\clearpage
\onecolumn
\appendix
\setlength{\abovedisplayskip}{5pt}
\setlength{\belowdisplayskip}{5pt}
\setlength{\textfloatsep}{20pt}

\section*{Index of the Appendix}
In the following, we briefly recap the contents of the Appendix.
\begin{itemize}[leftmargin=*, label={--}]
	\item Appendix~\ref{apx:pbpo} provides the formulation of the policy optimization problem with mediator feedback in the parameter-based setting.
	\item Appendix~\ref{apx:relatedWorks} completes the review of the relevant literature begun in Section~\ref{sec:related}, focusing on the approaches that share connections with the mediator feedback and \algname.
	\item Appendix~\ref{apx:key} reports some central lemmas that are employed for the analysis of the TMIS estimator for expected return estimation.
	\item Appendix~\ref{apx:proofs} provides the proofs that are omitted in the main paper.
	\item Appendix~\ref{apx:auxiliaryLemmas} reports some auxiliary lemmas that are employed in the analysis.
	\item Appendix~\ref{apx:expDetails} presents the experimental setting in more detail, additional experimental results and a discussion about implementation issues.
\end{itemize}

\section{Parameter-based PO and Mediator Feedback}\label{apx:pbpo}
In \emph{parameter-based} policy optimization~\citep[PB-PO][]{sehnke2008policy} the policy parameters $\vtheta$ are sampled from a higher level distribution, called \emph{hyperpolicy}, and the learning process occurs in the space of hyperpolicy parameters, named \emph{hyperparameters}. More formally, we consider a space of parametric hyperpolicies $\mathcal{N}_{\Xi} = \{ \nu_{\vxi} \in \mathscr{P}(\Theta) \,:\, \vxi \in \Xi  \}$ where $\Xi \subseteq \mathbb{R}^d$ is the hyperparameter space. To each hyperpolicy $\nu_{\vxi}$ we can associate an index of performance:
$$J(\vxi) = \E_{\vtheta \sim \nu_{\vxi}}\left[J(\vtheta)\right] = \E_{\vtheta \sim \nu_{\vxi}}\left[ \E_{\tau \sim p_{\vtheta}} \left[\Rs(\tau)\right]\right].$$
The goal consists in finding an optimal hyperparameter, \ie any $\vxi^*$ maximizing $J(\vxi)$. At each round $t\in[n]$, we evaluate a hyperparameter $\vxi_t \in \Xi$ by sampling one (or more) policy parameters $\vtheta_t$, running policy $\pi_{\vtheta_t}$, collecting one (or more) trajectory $\tau_t$ and observing the corresponding return $\Rs(\tau_t)$. Then, based on the history of observations $\Hs_{t} = \{(\vxi_i,(\vtheta_i, \tau_i), \Rs(\tau_i))\}_{i=1}^t$, we update the hyperparameter $\vxi_{t}$ to get $\vxi_{t+1}$. Differently from the \emph{action-based} paradigm (AB-PO), in PB-PO deterministic policies are typically employed, since the stochasticity of the hyperpolicy is a sufficient source of exploration~\cite{sehnke2008policy}.

From an \emph{online learning} perspective, the goal of an agent consists in maximizing the sum of the expected payoffs over $n$ rounds or, equivalently, minimize the cumulative regret $R(n)$ \wrt to an optimal decision:
\begin{equation}
	\max_{\vxi_1, \dots \vxi_{n} \in \Xi} \sum_{t=1}^n J(\vxi_t) \quad \Leftrightarrow \quad \min_{\vxi_1, \dots \vxi_{n} \in  \Xi} R(n) =  \sum_{t=1}^n \Delta(\vxi_t),
\end{equation}
where $\Delta(\vxi) = J^* - J(\vxi)$ is the optimality gap of $\vxi \in \Xi$ and $J^* =  \sup_{\vxi \in \Xi} J(\vxi)$.

\section{Related Works}\label{apx:relatedWorks}
In this appendix, we revise the additional relevant literature, with particular attention to structured bandits, approximate Thompson sampling, and RL approaches with regret guarantees.

\paragraph{Structured Bandits and Feedback}
Although the formulation is quite different, the mediator feedback can be thought of as a way to endow a bandit with a particular structure. Numerous works have studied different kinds of structure (\eg linear~\cite{abbasi2011improved, dani2008stochastic, rusmevichientong2010linearly}, Lipschitz~\cite{magureanu2014lipschitz}). Of particular interest is~\cite{lattimore2014bounded}, in which \emph{general structures} are considered and constant problem-dependent regret results of order $\mathcal{O}(\Delta^{-1} \log \Delta^{-1})$ are derived for specific cases. Concerning the regret lower bounds, in \cite{bubeck2013bounded, lattimore2014bounded, combes2017minimal} several results are shown for different classes of structured bandits. Extensions of the bandit feedback in which, when an arm is pulled, the outcome of other arms is revealed (possibly with some probability) are typically based on a graph structure~\cite{alon2015online, chen2016combinatorial}. More specifically, in~\cite{caron2012leveraging} the notion of \emph{side-observation} is introduced to consider free extra information (passive or active) that allows achieving constant problem-dependent regret. This is quite similar to our mediator feedback, although we do not receive further fresh samples but we employ a single sample to update the estimates of multiple arms. More generally, we can see the mediator feedback (but also the side observations) as  something in between the bandit feedback and the full information (expert) feedback. Another related concept is that of instrument-armed bandits~\cite{kallus2018instrument}, where the reward of a decision (instrument) depends on an intermediate, observed variable (\eg compliance to medical prescription). Different definitions of regret can be adopted depending on whether one is interested in the causal relationship between the instrument and the reward. We adopt here the \emph{Intent-to-Treat Regret} formulation, since we are only interested in finding a good instrument (policy or hyperpolicy). In~\cite{kallus2018instrument}, this case is treated as a regular bandit problem, ignoring the mediator feedback.

\paragraph{Approximate Thompson Sampling}
Thompson Sampling~\citep[TS,][]{thompson1933likelihood, russo2018tutorial} is an effective methodology for randomized exploration in multi-armed bandits. The main bottlenecks of TS are the computation of and the sampling from the posterior distribution. Several works focused on the effect of sampling from \emph{approximate posteriors} with guarantees on the degradation of the (Bayesian) regret~\cite{lu2017ensemble,phan2019thompson}. Other works addressed the sampling issue by employing Monte Carlo Markov Chain~\cite[MCMC,][]{casella1992explaining} approaches with Laplace approximation~\cite{chapelle2011empirical}, Langevin Monte Carlo~\cite{bubeck2018sampling, cheng2018convergence}, and bootstrapping~\cite{eckles2014thompson}. Apart from the contextual bandits with linear payoff~\cite{agrawal2013thompson, abeille2017linear}, the case of infinite arm set has been addressed with TS only in a very limited number of works, deriving guarantees on the Bayesian regret under strong regularity conditions~\cite{grant2020thompson} (without proposing a sampling routine) or by employing adaptive discretizations in the context of GPs~\cite{chowdhury2017kernelized}.

\paragraph{Reinforcement Learning}
We have already surveyed the approaches to PO with regret guarantees in Section~\ref{sec:related}. Here, we focus on other provably efficient RL approaches. The majority of RL methods with theoretical guarantees has been developed in the context of tabular RL and are based on optimistic exploration~\cite[\eg][]{jaksch2010near, jin2018q, strehl2008analysis}. These approaches, typically, do no extend directly to continuous tasks and/or to a randomized form of exploration. Recently, a number of approaches have been proposed to apply \emph{posterior sampling} for solving MDPs, mainly with guarantees on the Bayesian regret~\cite[\eg][]{osband2013more, osband2017deep}, but also on the frequentist regret for some specific classes of continuous problems~\cite[\eg][]{abeille2017thompson,abeille2018improved, dean2018regret}.



\section{Key Lemmas on Off-Distribution Payoff Estimation}\label{apx:key}
In this appendix, we revise the key lemmas needed when using MIS and TMIS for off-policy expected return estimation. We start in Appendix~\ref{apx:biasVariance} with a result for bounding bias and variance for general truncation threshold $M$ and then we focus, in Appendix~\ref{apx:payoffOff}, on the specific threshold $M_t(\xs)$ used in the algorithms.

\subsection{Lemmas for General Truncation}\label{apx:biasVariance}
In this appendix, we consider the importance weights defined for general probability distributions $P,Q \in \mathscr{P}(\OutcomeSet)$ with $P \ll Q$:
\begin{align*}
	\omega_{P/Q}(\tau) = \frac{\de P}{\de Q}(\tau), \qquad \widecheck{\omega}_{P/Q}(\tau) = \min \left\{ M, \omega_{P/Q}(\tau) \right\},
\end{align*}
where $M < \infty$ is the truncation threshold. We start with an ancillary result, that extends Lemma 2 by~\citet{papini2019optimistic} in bounding the $\alpha$-moments and the bias of the truncated weight.

\begin{restatable}[]{lemma}{VarBiasMax}\label{lemma:VarBiasMax}
Let $P$ and $Q$ be probability measures on the measurable space $(\OutcomeSet,\mathscr{F})$ with $P \ll Q$. Then, for any $\alpha \in (0, \infty]$, the $\alpha$-moment of the truncated importance weight $\widecheck{\omega}_{P/Q}$ with threshold $M$ can be bounded for any $\beta \in [0, 1]$ as:
	\begin{equation}
	\E_{\tau \sim Q}[\widecheck{\omega}_{P/Q}(\tau)^\alpha]^{\frac{1}{\alpha}} \le M^{1-\beta}d_{\alpha\beta}(P \| Q)^{\beta-\frac{1}{\alpha}}.
	\end{equation}
	Furthermore, the bias of the truncated weight $\widecheck{\omega}_{P/Q}$  can be bounded for any $\beta \in [1, \infty]$ as:
	\begin{equation}
	\E_{\tau \sim Q}[{\omega}_{P/Q}(\tau) - \widecheck{\omega}_{P/Q}(\tau)]
	\le \left( \frac{d_{\beta}(P \| Q)}{M} \right)^{\beta-1}.
	\end{equation}
\end{restatable}
\begin{proof}
	Let us start with the first result. Consider the following derivation with $\beta \in [0,1]$:
	\begin{align}
		\E_{\tau \sim Q}[\widecheck{\omega}_{P/Q}(\tau)^\alpha] & = \E_{\tau \sim Q}\left[\min \left\{ M, {\omega}_{P/Q}(\tau) \right\}^{\alpha} \right] \notag \\
		& = \E_{\tau \sim Q}\left[\min \left\{ M, {\omega}_{P/Q}(\tau) \right\}^{\alpha(1-\beta)} \min \left\{ M, {\omega}_{P/Q}(\tau) \right\}^{\alpha\beta} \right] \notag \\
		& \le M^{\alpha(1-\beta)} \E_{\tau \sim Q}\left[ {\omega}_{P/Q}(\tau)^{\alpha\beta} \right]  \label{p:201} \\
		& = M^{\alpha(1-\beta)} d_{\alpha\beta}(P \| Q)^{\alpha\beta - 1}, \label{p:202}
	\end{align}
	where line~\eqref{p:201} is obtained by bounding the minimum and line~\eqref{p:202} comes from the definition of exponentiated \Renyi divergence. The result is obtained by taking the $\frac{1}{\alpha}$-power both sides.

	For the second result, we consider the following derivation for $\beta \in [1,\infty]$:
	\begin{align}
		\E_{\tau \sim Q}[{\omega}_{P/Q}(\tau) - \widecheck{\omega}_{P/Q}(\tau)] & = \E_{\tau \sim Q}\left[\left({\omega}_{P/Q}(\tau) - M\right) \Ind \left\{ {\omega}_{P/Q}(\tau) > M \right\} \right] \notag \\
		& \le \E_{\tau \sim Q}\left[{\omega}_{P/Q}(\tau) \Ind \left\{ {\omega}_{P/Q}(\tau) > M \right\} \right] \notag \\
		& \le \E_{\tau \sim Q}\left[{\omega}_{P/Q}(\tau)^\beta \right]^{\frac{1}{\beta}} \E_{\tau \sim Q}\left[ \Ind \left\{ {\omega}_{P/Q}(\tau) > M \right\}^{\frac{\beta}{\beta-1}} \right]^{\frac{\beta-1}{\beta}} \label{p:203}\\
		& = d_{\beta}(P \| Q)^{\frac{\beta-1}{\beta}} \Prob_{\tau \sim Q}\left({\omega}_{P/Q}(\tau) > M \right)^{\frac{\beta-1}{\beta}}, \label{p:204}
	\end{align}
	where line~\eqref{p:203} is an application of H\"older's inequality and line~\eqref{p:204} comes from the definition of exponentiated \Renyi divergence. Considering the probability we have for any $\gamma \in (0,\infty)$:
	\begin{align}
	\Prob_{\tau \sim Q}\left({\omega}_{P/Q}(\tau) > M \right) & = \Prob_{\tau \sim Q}\left({\omega}_{P/Q}(\tau)^\gamma > M^\gamma \right) \notag \\
	& \le \frac{\E_{\tau \sim Q}[\omega_{P/Q}(\tau)^\gamma]}{M^\gamma} \label{p:205} \\
	& = \frac{d_\gamma(P \| Q)^{\gamma-1}}{M^\gamma},\label{p:206}
	\end{align}
	where line~\eqref{p:205} follows from Markov's inequality and line~\eqref{p:206} from the definition  of exponentiated \Renyi divergence. By taking $\gamma=\beta$, we get the result.
\end{proof}

It is worth noting that, while for bounding the $\alpha$-moment of the non-truncated weight we need the $\alpha$-\Renyi divergence to be finite, for the truncated weight we can employ any $\alpha\beta$-\Renyi divergence, where $\alpha\beta \le \alpha$.

\subsection{Lemmas for $M_t(\xs)$ Truncation}\label{apx:payoffOff}
From now on, let $\widecheck{J}_t$ be the TMIS estimator for the expected return, in the case of finite policy space, at time $t$:
\begin{align}
\widecheck{J}_t(\xs) &= \frac{1}{t-1}\sum_{h=1}^{K}\sum_{l=1}^{T_t(\xs_h)}\min\left\{M_t(\xs), \frac{p_{\xs}(\outcome_{hl})}{\sum_{k=1}^{K}\frac{T_t(\xs_k)}{t-1}p_{\xs_h}(\outcome_{hl})}\right\}\func(\outcome_{hl}) \nonumber\\
&=\sum_{i=1}^{t-1}\min\left\{M_t(\xs), \frac{p_{\xs}(\outcome_i)}{\sum_{j=1}^{t-1}p_{\xs_j}(\outcome_i)}\right\}\func(\outcome_i),
\end{align}
where $K=|\DecisionSet|$, $T_t(\xs)$ is the number of executions of policy $\pi_{\xs}$ previous to time $t-1$, $\outcome_{hl}$ denotes the $l$-th trajectory from policy $\pi_{\xs_h}$, and $\outcome_{i}$ denotes the $i$-th trajectory overall. Note that, with slight abuse of notation, the underscript of trajectories is over the policy space in the first expression and over time in the second expression.
Also, let us fix the truncation threshold:
\begin{align}\label{p:truncation}
M_t(\xs) = \sqrt{\frac{(t-1)d_2(p_{\xs}\|\Phi_t)}{\alpha\log t}} = d_2(p_{\xs}\|\Phi_t)\sqrt{\frac{\eta_t(\xs)}{\alpha\log t}},
\end{align}
for some $\alpha>1$, where $\eta_t(\xs) = \frac{t-1}{d_{2}(p_{\xs}\|\Phi_t)}$ is the effective number of trajectories (or effective sample size) and ${\Phi_t = \sum_{h=1}^{K}\frac{T_t(\xs_h)}{t-1}p_{\xs_h} = \frac{1}{t-1}\sum_{i=1}^{t-1}p_{\xs_i}}$ is the mixture of the distributions of trajectories sampled previous to $t$. We sometimes abbreviate $M_t(\xs)$ and $\eta_t(\xs)$ as $M$ and $s$, respectively, when parameter and time are clear from context. When not specified, expected values are \wrt past history $\mathcal{H}_{t-1}$. We always assume $\func(\tau) \in [0,1]$. With little abuse of language, we will sometimes say \quotes{policy $\vtheta$} to mean \quotes{policy $\pi_{\vtheta}$}.

\begin{lemma}\label{th:basic}
The bias of $\widecheck{J}_t$ is bounded as follows:
\begin{align*}
	0\leq J - \E[\widecheck{J}_t(\xs)] \leq \sqrt{\frac{\alpha\log t}{\eta_t(\xs)}},
\end{align*}
the variance:
\begin{align*}
	\Var [\widecheck{J}_t(\xs)] \leq \frac{d_{2}(p_{\xs}\|\Phi_t)}{t-1},
\end{align*}
and the estimator itself:
\begin{align*}
	0\leq \widecheck{J}_t(\xs) \leq d_2(p_{\xs}\|\Phi_t)\sqrt{\frac{\eta_t(\xs)}{\alpha\log t}}.
\end{align*}
\end{lemma}
\begin{proof}
The last property is evident from the chosen truncation threshold~\eqref{p:truncation}, and the first two can be easily deduced from~\cite{papini2019optimistic}, Lemma 2.
\end{proof}

\begin{lemma}\label{th:ci}
	For all $\xs\in\DecisionSet$, $t\geq 1$ and $\epsilon\geq0$:
	\begin{align*}
	\Prob\left(\widecheck{J}_t(\xs) - J(\xs) > \epsilon \right)\leq
	\exp\left[-\frac{\epsilon^2 \eta_t(\xs)}{2\left(1+\frac{\epsilon}{3}\sqrt{\frac{\eta_t(\xs)}{\alpha\log t}}\right)}\right].
	\end{align*}
	Moreover, if $\epsilon\geq\sqrt{\frac{\alpha\log t}{\eta_t(\xs)}}$:
	\begin{align*}
	\Prob\left(J(\xs)-\widecheck{J}_t(\xs) > \epsilon\right) \leq \exp\left[-\frac{\eta_t(\xs)}{2}\left(\epsilon-\sqrt{\frac{\alpha\log t}{\eta_t(\xs)}}\right)^2\right].
	\end{align*}
\end{lemma}
\begin{proof}
 For the first concentration inequality:
	\begin{align}
	\Prob\left(\widecheck{J}_t(\xs) - J(\xs) > \epsilon \right)
	&\leq \Prob\left(\widecheck{J}_t(\xs) - \E[\widecheck{J}_t(\xs)] > \epsilon \right) \\
	&  \le \exp \left[ \frac{-\epsilon^2 (t-1)}{2  \left( d_{2}(p_{\xs} \| \Phi_t) + \frac{ \epsilon d_{2}(p_{\xs} \| \Phi_t)\sqrt{\frac{\eta_t(\xs)}{\alpha \log t} } }{3} \right) } \right]\label{p:1} \\
	& = \exp \left[ \frac{-\epsilon^2 \eta_t(\xs)}{2  \left(1  + \frac{ \epsilon}{3} \sqrt{\frac{\eta_t(\xs)}{\alpha \log t} } \right) } \right],
	\end{align}
	where we have used Lemma~\ref{th:basic} and \eqref{p:1} is from Theorem~\ref{th:chernoff}.
	Similarly, for the second concentration inequality we still use Theorem~\ref{th:chernoff} together with Lemma~\ref{th:basic}:
	\begin{align*}
	\Prob\left(J(\xs)-\widecheck{J}_t(\xs) > \epsilon\right)
	&= \Prob \left(  \E [\widecheck{J}_t(\xs)]  -\widecheck{J}_t(\xs) > \epsilon + \E [\widecheck{J}_t(\xs)]- J(\xs)  \right) \notag \\
	&\le \Prob \left(  \E [\widecheck{J}_t(\xs)]  -\widecheck{J}_t(\xs) > \epsilon -\sqrt{\frac{\alpha\log t}{\eta_t(\xs)}}  \right) \\
	& \le \exp \left[ - \frac{\left( \epsilon - \sqrt{\frac{\alpha \log t}{\eta_t(\xs)}} \right)^2 t}{2 d_2(p_{\xs} \| \Phi_t)} \right] \\
	& = \exp \left[ - \frac{1}{2} \left( \epsilon - \sqrt{\frac{\alpha \log t}{\eta_t(\xs)}} \right)^2 \eta_t(\xs) \right].\notag
	\end{align*}
\end{proof}

\begin{lemma}\label{th:ess}
The effective number of trajectories of a policy is always larger than the number of executions of that policy:
\begin{align*}
	\eta_t(\xs) \geq T_t(\xs).
\end{align*}
Moreover, if $v(\xs)=\sup_{x'\in\DecisionSet}d_2(p_{\xs}\|p_{\xs'})$ is finite:
\begin{align*}
	\eta_t(\xs) \geq \frac{t-1}{v(\xs)}.
\end{align*}
\end{lemma}
\begin{proof}
The first inequality is trivial if $T_t(\xs)=0$, so assume it is positive. From~\cite[][Theorem 5]{papini2019optimistic} we know that $d_2(p_{\xs}\|\Phi_t)$ is bounded by the harmonic mean of pairwise divergences:
\begin{align*}
	d_{2}(p_{\xs}\|\Phi_t) &\le \frac{t-1}{\sum_{i=1}^{t-1}\frac{1}{d_2(p_{\xs}\|p_{\xs_i})}}\\
	&=\frac{t-1}{T_t(\xs)+ \sum_{i=1}^{t-1}\Ind\{\xs_i\neq\xs\}\frac{1}{d_2(p_{\xs}\|p_{\xs_i})}} \le \frac{t-1}{T_t(\xs)}.\\
\end{align*}
Moreover, $d_2(p_{\xs}\|\Phi_t) \le v(\xs)$ since the harmonic mean is never larger than the maximum. The claims follow by definition of $\eta_t(\xs)$.
\end{proof}

\begin{lemma}\label{th:gap}
In \problem, the optimality gap of policy $\xs$ is bounded as:
\begin{align}
	\Delta(\xs) \le \sqrt{2\log v(\xs)},
\end{align}
where $v(\xs)=\sup_{\xs'\in\DecisionSet}d_{2}(p_{\xs}\|p_{\xs'})$.
\end{lemma}
\begin{proof}
	\begin{align}
		\Delta(\xs) &= |J(\xs^{*}) - J(\xs)|\label{p:9} \\
		&= \left|\int_{\OutcomeSet}\left(p_{\xs^{*}}(z) - p_{\xs}(z)\right)\func(z)\de z\right| \notag\\
		&\le \int_{\OutcomeSet}\left|p_{\xs^{*}}(z) - p_{\xs}(z)\right|\de z \notag\\
		&= 2D_{TV}(p_{\xs},p_{\xs^{*}}) \notag\\
		&\le \sqrt{2D_{KL}(p_{\xs}\|p_{\xs^{*}})} \label{p:10}\\
		&\le  \sqrt{2D_{2}(p_{\xs}\|p_{\xs^{*}})} \le \sqrt{2\log v(\xs)},\label{p:1111}
	\end{align}
	where $D_{TV}$ is the total variation distance, in~\eqref{p:9} we use $\Delta(\xs)\ge 0$, \eqref{p:10} is from Pinsker's inequality, and~\eqref{p:1111} comes from the monotonicity of the \Renyi divergence in the order.
\end{proof}

\section{Proofs and Derivations}\label{apx:proofs}
In this appendix, we report the proofs and the derivations we omitted in the main paper.

\subsection{Proofs of Section~\ref{sec:problem}}
\lbConstantRegret*
\begin{proof}
	To prove the lower bound, we consider a pair of MDPs $\nu_1$ and $\nu_2$ with horizon $2$ and $\mathcal{S} = \mathcal{A} = \mathbb{R}$. Thus, each trajectory is represented by the triple $\tau=(s,a,s')$. We take for the two MDPs the same reward function $\func(\tau)=s'$. The policy space is induced by $\DecisionSet=\{\xs_1,\xs_2\}$. Let $\Delta \in [0,1]$, for the first problem $\nu_1$ we select the trajectory distributions as follows:\footnote{The first factor is the initial-state distribution $\mu$ that is chosen equal for all the problems, the second factor is the policy $\pi_{\vtheta}$, and the third factor is the transition model $P$.}
	\begin{align*}
		& p_{\xs_1}^{\nu_1}(\tau) = \mu(s) \mathcal{N}(a|1,1)\mathcal{N}(s'|a\Delta,1),
		& p_{\xs_2}^{\nu_1}(\tau) = \mu(s) \mathcal{N}(a|0,1)\mathcal{N}(s'|a\Delta,1),
	\end{align*}
	leading to the expected returns $J^{\nu_1}(\xs_1)=\Delta$ and $J^{\nu_1}(\xs_2)=0$.	Instead, for the second problem $\nu_2$ we select:
	\begin{align*}
		& p_{\xs_1}^{\nu_2}(\tau) = \mu(s)\mathcal{N}(a|1,1)\mathcal{N}(s'|-a\Delta,1),
		& p_{\xs_2}^{\nu_2}(\tau) = \mu(s)\mathcal{N}(a|0,1)\mathcal{N}(s'|-a\Delta,1),
	\end{align*}
	leading to the expected returns $J^{\nu_2}(\xs_1)=-\Delta$ and $J^{\nu_2}(\xs_2)=0$. For $\nu_1$ the optimal decision is $\xs_1$, while for $\nu_2$ the optimal policy is $\xs_2$ and, for both, the gap is $\Delta$. Furthermore, notice that:
	\begin{equation}
	\KL(p_{\xs_1}^{\nu_1} \| p_{\xs_2}^{\nu_1}) = \int p_{\xs_1}^{\nu_1}(\outcome) \log \frac{p_{\xs_1}^{\nu_1}(\outcome)}{p_{\xs_2}^{\nu_1}(\outcome)} \de \outcome = \KL(\mathcal{N}(1,1) \| \mathcal{N}(0,1)) = \frac{1}{2}.
\end{equation}
Similar derivations lead to $\KL(p_{\xs_2}^{\nu_1} \| p_{\xs_1}^{\nu_1}) = \KL(p_{\xs_1}^{\nu_2} \| p_{\xs_2}^{\nu_2}) = \KL(p_{\xs_2}^{\nu_2} \| p_{\xs_1}^{\nu_2}) = \frac{1}{2}.$

Define a history of length $t$ generated by the interaction of a policy with a problem $\nu$ as $\Hs_t^{\nu} = ((\xs_i, \outcome_i, \func(\outcome_i)))_{i=1}^t$. Given two problems we define $\KL(\Hs_t^{\nu_1}\| \Hs_t^{\nu_2})$ as the KL-divergence between the distributions having generated the histories. Using standard derivations~\cite{bubeck2013bounded} we have:
\begin{align}
    \max\left\{ \E_{\nu_1} R(n),\E_{\nu_2} R(n) \right\} &\ge \frac{1}{2} \left( \E_{\nu_1} R(n) + \E_{\nu_2} R(n) \right) \notag \\
    & = \frac{\Delta}{2} \sum_{t=1}^n \left( \Prob_{\nu_1} \left(\xs_t = \xs_2 \right) + \Prob_{\nu_2} \left(\xs_t = \xs_1 \right) \right) \notag \\
    & \ge \frac{\Delta}{4} \sum_{t=1}^n \exp \left[- \KL(\Hs_t^{\nu_1}\| \Hs_t^{\nu_2}) \right],\label{p:101}
\end{align}
where we denoted with $\E_{\nu_1} $ (resp. $\E_{\nu_2} $) the expectation taken under the randomness of problem $\nu_1$ (resp. $\nu_2$) and we denoted with $\Prob_{\nu_1} \left(\xs_t = \xs_2 \right)$ (resp. $\Prob_{\nu_2} \left(\xs_t = \xs_1 \right)$) the probability of choosing decision $\xs_2$ (resp. $\xs_1$) at round $t$ in the problem $\nu_1$ (resp. $\nu_2$). The last passage follows from Lemma 4 of~\cite{bubeck2013bounded}. Recalling that we have selected the same reward function for both problems, and again with standard derivations, we have:
\begin{align*}
 \KL(\Hs_t^{\nu_1}\| \Hs_t^{\nu_2}) & = \E_{\nu_1}[T_t(\xs_1)] \KL(p_{\xs_1}^{\nu_1}\| p_{\xs_1}^{\nu_2}) + \E_{\nu_1}[T_t(\xs_2)] \KL(p_{\xs_2}^{\nu_1}\| p_{\xs_2}^{\nu_2}),
\end{align*}
Let us now compute the divergences:
\begin{align*}
 \KL(p_{\xs_1}^{\nu_1}\| p_{\xs_1}^{\nu_2}) & = \int p_{\xs_1}^{\nu_1}(\tau) \log \frac{p_{\xs_1}^{\nu_1}(\tau)}{p_{\xs_1}^{\nu_2}(\tau)} \de \tau \\
 & = \E_{a \sim \mathcal{N}(1,1)} \left[ \KL(\mathcal{N}(\Delta a,1) \| \mathcal{N}(-\Delta a,1)) \right] \\
 & = \E_{a \sim \mathcal{N}(1,1)} \left[2 \Delta^2 a^2 \right] = 4 \Delta^2.
\end{align*}
In a similar way, we can derive 
$\KL(p_{\xs_2}^{\nu_1}\| p_{\xs_2}^{\nu_2}) =\E_{a \sim \mathcal{N}(0,1)} \left[ \KL(\mathcal{N}(\Delta a,1) \| \mathcal{N}(-\Delta a,1)) \right] = 2 \Delta^2$.
 Thus, we have:
\begin{align*}
 \KL(\Hs_t^{\nu_1}\| \Hs_t^{\nu_2}) = 2 \Delta^2 \left( 2 \E_{\nu_1}[T_t(\xs_1)] + \E_{\nu_1}[T_t(\xs_2)] \right) \le 4 \Delta^2 t.
\end{align*}
Plugging this result into Equation~\eqref{p:101}:
\begin{align*}
	\max\left\{ \E_{\nu_1} R(n),\E_{\nu_2} R(n) \right\}  \ge \frac{\Delta}{4} \sum_{t=1}^n \exp \left[- 4 \Delta^2 t \right] \ge \frac{1}{32\Delta},
\end{align*}
where the last passage holds for sufficiently large $n$~\cite{lattimore2014bounded}.
\end{proof}

\lbLogRegret*
\begin{proof}
	The proofs follows the same steps of that of Theorem~\ref{thr:lbConstantRegret}, but with a different construction of the trajectory distributions. We still consider a pair of MDPs $\nu_1$ and $\nu_2$ with horizon $2$ defined over the policy space $\DecisionSet=\{\xs_1,\xs_2\}$ and having $\mathcal{S} = \mathcal{A} =\mathbb{R}$. We take for the two problems the same reward functions $\func(\tau)=a$. Let $\Delta \in [0,1]$, for the first problem $\nu_1$ we select the trajectory distributions as:
	\begin{align*}
		& p_{\xs_1}^{\nu_1}(\outcome) = \mu(s) \mathcal{N}(a|1,1)\mathcal{N}(s'|a\Delta,1),
		& p_{\xs_2}^{\nu_1}(\outcome) = \mu(s) \delta_{0}(a)\mathcal{N}(s'|a\Delta,1),
	\end{align*}
	where $\delta_{x}$ is the Dirac measure centered in $x$, leading to the expected returns $J^{\nu_1}(\xs_1)=\Delta$ and $J^{\nu_1}(\xs_2)=0$.	Instead, for the second problem $\nu_2$ we select:
	\begin{align*}
		& p_{\xs_1}^{\nu_2}(\outcome) = \mu(s) \mathcal{N}(a|1,1)\mathcal{N}(s'|-a\Delta,1),
		& p_{\xs_2}^{\nu_2}(\outcome) = \mu(s) \delta_{0}(a)\mathcal{N}(s'|-a\Delta,1),
	\end{align*}
	leading to the expected returns $J^{\nu_2}(\xs_1)=-\Delta$ and $J^{\nu_2}(\xs_2)=0$. For $\nu_1$ the optimal policy is $\xs_1$, while for $\nu_2$ the optimal policy is $\xs_2$ and, for both, the gap is $\Delta$. Differently from the proof of Theorem~\ref{thr:lbConstantRegret}, we considered Dirac deltas for the first factor of trajectory distribution of $\xs_2$ instead of normal distributions. This leads to:
	\begin{equation}
	\KL(p_{\xs_1}^{\nu_1} \| p_{\xs_2}^{\nu_1}) = \int p_{\xs_1}^{\nu_1}(\outcome) \log \frac{p_{\xs_1}^{\nu_1}(\outcome)}{p_{\xs_2}^{\nu_1}(\outcome_1)} \de \outcome = \KL(\mathcal{N}(1,1) \| \delta_{0}) = \infty.
\end{equation}
Similar derivations lead to $\KL(p_{\xs_2}^{\nu_1} \| p_{\xs_1}^{\nu_1}) = \KL(p_{\xs_1}^{\nu_2} \| p_{\xs_2}^{\nu_2}) = \KL(p_{\xs_2}^{\nu_2} \| p_{\xs_1}^{\nu_2}) = \infty.$

The analysis is now carried out \wrt the second problem $\nu_2$. First of all, we notice that:
\begin{align*}
\max\left\{ R_{\nu_1}(n),R_{\nu_2}(n) \right\} \ge R_{\nu_2}(n) \ge \Delta \E_{\nu_2}[T_n(\xs_1)].\label{p:103}
\end{align*}
Moreover, using standard derivations~\cite{bubeck2013bounded} we have:
\begin{align}
    \max\left\{ \E_{\nu_1} R(n),\E_{\nu_2} R(n) \right\} &\ge \frac{1}{2} \left( \E_{\nu_1} R(n) + \E_{\nu_2} R(n) \right) \notag \\
    & = \frac{\Delta}{2} \sum_{t=1}^n \left( \Prob_{\nu_1} \left(\xs_t = \xs_2 \right) + \Prob_{\nu_2} \left(\xs_t = \xs_1 \right) \right) \notag \\
    & \ge \frac{\Delta}{4} \sum_{t=1}^n \exp \left[- \KL(\Hs_t^{\nu_2} \| \Hs_t^{\nu_1}) \right]\\
    & \ge \frac{n \Delta}{4} \exp \left[- \KL(\Hs_n^{\nu_2} \| \Hs_n^{\nu_1}) \right],\label{p:102}
\end{align}
where the only difference with the proof of Theorem~\ref{thr:lbConstantRegret} is that we switched the roles of $\nu_1$ and $\nu_2$.
Recalling that we have selected the same reward function for both problems, and again with standard derivations, we have:
\begin{align*}
 \KL(\Hs_t^{\nu_2} \| \Hs_t^{\nu_1} ) & = \E_{\nu_2}[T_t(\xs_1)] \KL(p_{\xs_1}^{\nu_2}\| p_{\xs_1}^{\nu_1}) + \E_{\nu_2}[T_t(\xs_2)] \KL(p_{\xs_2}^{\nu_2}\| p_{\xs_2}^{\nu_1}),
\end{align*}
Let us now compute the divergences. For $\xs_1$, $\KL(p_{\xs_1}^{\nu_2}\| p_{\xs_1}^{\nu_1}) = 4 \Delta^2$ as in Theorem~\ref{thr:lbConstantRegret}. Instead, for the decision $\xs_2$:
\begin{align*}
 \KL(p_{\xs_2}^{\nu_2}\| p_{\xs_2}^{\nu_1}) & = \int p_{\xs_2}^{\nu_2}(\outcome) \log \frac{p_{\xs_2}^{\nu_2}(\outcome)}{p_{\xs_2}^{\nu_1}(\outcome)} \de \outcome \\
 & = \E_{a \sim \delta_{0}} \left[ \KL(\mathcal{N}(-a\Delta,1)) \| \mathcal{N}(a\Delta,1) \right] \\
 & = \E_{a \sim \delta_{0}} \left[2 a^2\Delta^2 \right] = 0.
\end{align*}
Thus, we have $ \KL(\Hs_t^{\nu_2}\| \Hs_t^{\nu_1}) = 4\Delta^2 \E_{\nu_2}[T_t(\xs_2)]$. Plugging this result into Equation~\eqref{p:102} and combining it with Equation~\ref{p:103}:
\begin{align*}
	 \max\left\{ \E_{\nu_1} R(n),\E_{\nu_2} R(n) \right\}  & \ge \max \left\{ \Delta \E_{\nu_2}[T_n(\xs_1)],  \frac{n \Delta}{4} \exp \left[- 4\Delta^2 \E_{\nu_2}[T_n(\xs_2)] \right]   \right\}\\
	& \ge \frac{\Delta }{2} \left(\E_{\nu_2}[T_n(\xs_1)]  + \frac{n }{4} \exp \left[- 4\Delta^2 \E_{\nu_2}[T_n(\xs_2)] \right]  \right) \\
	& \ge  \frac{\Delta }{2}  \min_{x \in [0,n]} \left\{ x + \frac{n }{4} \exp \left[- 4\Delta^2 x \right]  \right\} \\
	& \ge  \frac{1}{8\Delta}  \log (\Delta^2 n),
\end{align*}
where the last line follows by solving the optimization problem over $x$, simply by zeroing the derivative.
\end{proof}

\subsection{Proofs of Section~\ref{sec:det}}
\begin{fact}\label{fact:ftl}
	There exist an MDP and a parameter space $\Theta=\{\vtheta_1,\vtheta_2\}$ with $d_{2}(\xs_1 \| \xs_2)= \infty$ or $d_{2}(\xs_2 \| \xs_1)= \infty$ such that the expected regret of FTL is at least $\E R(n) \ge \frac{1}{16} (n-1)$.
\end{fact}

\begin{proof}
We consider a version of the FTL algorithm in which all policies are played once at the beginning. Let $\DecisionSet = \{\xs_1, \xs_2\}$ and $\OutcomeSet = \mathbb{R}$. We consider the following trajectory distributions:
	\begin{equation}
		p_{\xs_1} = \mathrm{Uni}([0,1]), \qquad \qquad p_{\xs_2} = \delta_{1/4},
	\end{equation}
	where we denoted with $\mathrm{Uni}$ the uniform distribution. Finally, we select as reward function $\func(\tau) = \tau$. Clearly, the optimal policy is $\xs_1$ having expected return $1/2$, while $\xs_2$ has expected return $1/4$. Notice that $d_2(p_{\xs_1} \| p_{\xs_2} ) = \infty$. Consider the bad event $E$ in which, when pulled, $\xs_1$ provides a reward $\func(\tau)$ that is smaller than $1/4$. This event has finite probability $\Prob(E) = \Prob ( \tau < 1/4 | \tau \sim  \mathrm{Uni}([0,1])) = 1/4$. After the initial play of the two policies, the estimates based on TMIS are just the on-policy ones. Indeed, on event $E$:
\begin{align*}
	& \widecheck{J}_2(\xs_1) = \frac{1}{2} \left( \frac{p_{\xs_1}(\tau) \tau}{\frac{1}{2} p_{\xs_1}(\tau) +  \frac{1}{2} p_{\xs_2}(\tau)}  +  \frac{p_{\xs_1}(1/4) 1/4}{\frac{1}{2} p_{\xs_1}(1/4) +  \frac{1}{2} p_{\xs_2}(1/4)}  \right) = \tau \\
	& \widecheck{J}_2(\xs_2) = \frac{1}{2} \left( \frac{p_{\xs_2}(\tau) \tau}{\frac{1}{2} p_{\xs_1}(\tau) +  \frac{1}{2} p_{\xs_2}(\tau)}  +  \frac{p_{\xs_2}(1/4) 1/4}{\frac{1}{2} p_{\xs_1}(1/4) +  \frac{1}{2} p_{\xs_2}(1/4)}  \right) = 1/4.
	\end{align*}
	Since $\widecheck{J}_2(\xs_1) < \widecheck{J}_2(\xs_2)$, FTL will play $\xs_2$ at round 3. Moreover, since the samples from $\xs_2$ do not change the estimate $\widecheck{J}_2(\xs_1)$, FTL will consistently play $\xs_2$ suffering a regret of  $\frac{1}{4} (n-1)$. Thus:
	\begin{equation}
		\E R(n) \ge \E[R(n) | E] \Prob(E) = \frac{1}{16} (n-1).
	\end{equation}
\end{proof}

We first prove Theorem~\ref{th:optimist} on OPTIMIST (Algorithm~\ref{alg:optimist}), then prove Theorem~\ref{th:ftl} on FTL as a variant.
Before proceeding, a clarification on the initial executions performed by OPTIMIST is due.
\begin{remark}\label{rem:roundrobin}
	We assume the expected reward estimators are initialized to an infinite value, \ie $\widecheck{J}_1(\xs) \gets +\infty$ for all $\xs\in\DecisionSet$, as is customary in OFU algorithms. Until there are infinite-valued estimates, one of the corresponding policies must necessarily be executed. We refer to this initial phase as \emph{Round-Robin regime}. Note that the estimator of $\xs$ is updated (becomes finite) once a policy $\xs'$ at a finite \Renyi divergence, \ie $d_2(p_{\xs}\|p_{\xs'})<\infty$, is executed ($\xs'$ can be $\xs$ itself).
	We distinguish two cases:
	\begin{itemize}
		\item[(a)] If all the pairwise \Renyi divergences are finite (perfect mediator feedback), the initial policy $\xs_1$ is executed, then all estimators are updated and the Round-Robin regime immediately ends.
		\item[(b)] If some \Renyi divergences are infinite, let us call \emph{bad} a policy such that $v(\xs)=\max_{\xs'\in\DecisionSet}d_{2}(p_{\xs}\|p_{\xs'})=\infty$. In this case, we first execute all the bad policies once in a Round-Robin fashion. After that, all expected return estimates must be finite. Notice that $\xs_1$ need not be executed in this case unless it is itself a bad policy.
	\end{itemize}
\end{remark}

Hence, OPTIMIST only needs a \emph{partial} initial Round-Robin, compared to the full Round-Robin of UCB1. Taking this into account, we now bound the expected regret:

\optimist*
\begin{proof}
We bound the expected number of executions $\E T_n(\xs)$ of policy $\xs\in\DecisionSet$. The expected regret is then:
\begin{align}\label{p:regret}
\E R(n) = \sum_{x\in\DecisionSet:\Delta(x)>0}\E[T_n(\xs)]\Delta(\xs).
\end{align}
Fix a policy $\xs\in\DecisionSet$ and consider the following \quotes{good} events:
\begin{align*}
& E_t = \left\{ \widecheck{J}_{t}(\xs) \le J(\xs) + (1+\sqrt{2})\sqrt{ \frac{\alpha \log t}{\eta_t(\xs)}} \right\},
& F_t = \left\{ \widecheck{J}_{t}(\xs^*) \ge J(\xs^*) - (1+\sqrt{2})\sqrt{ \frac{\alpha \log t}{\eta_t(\xs^{*})}}\right\}.
\end{align*}
We will make sure that these events are well defined, \ie $\eta_t(\xs)>0$ always.
By Lemma~\ref{th:ci}:
\begin{align}
	\Prob(\overline{E_t}) &= \Prob\left(\widecheck{J}_t(\xs) - J(\xs) > (1+\sqrt{2})\sqrt{\frac{\alpha\log t}{\eta_t(\xs)}}\right) \\
	&\le \exp\left[-\frac{3(8+5\sqrt{2})\alpha\log t}{28}\right] \le t^{-\alpha},
\end{align}
and also:
\begin{align}
	\Prob(\overline{F_t}) &= \Prob\left(J(\xs^{*}) - \widecheck{J}_t(\xs^{*}) > (1+\sqrt{2})\sqrt{\frac{\alpha\log t}{\eta_t(\xs)}}\right)\\
	&\le \exp\left[-\alpha\log t\right] \le t^{-\alpha}.
\end{align}
Under $E_t\cap F_t$:
\begin{align}
	J_t(\xs_t) + 2(1+\sqrt{2})\sqrt{\frac{\alpha\log t}{\eta_t(\xs_t)}}
	&\ge \widecheck{J}_t(\xs_t) + (1+\sqrt{2})\sqrt{\frac{\alpha\log t}{\eta_t(\xs_t)}} \label{p:3}\\
	&\ge \widecheck{J}_t(\xs^{*}) + (1+\sqrt{2})\sqrt{\frac{\alpha\log t}{\eta_t(\xs^{*})}}\label{p:4} \\
	&\ge J(\xs^{*})\label{p:5},
\end{align}
where~\eqref{p:3} is from $E_t$, \eqref{p:4} is from the policy selection rule and~\eqref{p:5} is from $F_t$. Rearranging:
\begin{align}
	\eta_t(\xs_t) \le \frac{4(1+\sqrt{2})^2\alpha\log t}{\Delta(\xs_t)^2} \le \frac{24\alpha\log t}{\Delta(\xs_t)^2}.
\end{align}
Let $m=\max\{1, |\{\xs\in\DecisionSet\mid v(\xs)=\infty\}|\}$. Hence:
\begin{align}
 \E[T_n(\xs)] &= \E\left[\sum_{t=m+1}^n \Ind\{\xs_t=\xs, E_t\cap F_t\} + \sum_{t=m+1}^n \Ind\{\xs_t=\xs, \overline{E_t}\cup \overline{F_t}\}\right] + \Ind\{v(\xs)=\infty \lor \xs=\xs_1\}\\
 &\le \E\left[\sum_{t=m+1}^{n}\Ind\left\{\xs_t=\xs, \eta_t(\xs) \leq \frac{24\alpha\log t}{\Delta(\xs)^2}\right\}\right] + \sum_{t=m+1}^{n}\Prob(\overline{E_t}\cup \overline{F_t})+ \Ind\{v(\xs)=\infty \lor \xs=\xs_1\},\label{p:tn}
\end{align}
where the third term is due to the initial round-robin (see Remark~\ref{rem:roundrobin}).
We first bound the second term:
\begin{align}
\sum_{t=m+1}^{n}\Prob(\overline{E_t}\cup \overline{F_t})
&\le \sum_{t=m+1}^{n}[\Prob(\overline{E_t})+ \Prob(\overline{F_t})] \\
&\le 2\sum_{t=m+1}^{n}t^{-\alpha}  \le 2\int_{1}^{\infty}t^{-\alpha}\de t \le \frac{2}{\alpha - 1}.\label{p:14}
\end{align}

For the first term of~\eqref{p:tn}, we bound it differently depending on whether $v(\xs)$ is finite or not.

If $v(\xs)<\infty$, we use $\eta_t(\xs)\geq \frac{t-1}{v(\xs)}>0$ from Lemma~\ref{th:ess}:
\begin{align}
\sum_{t=m+1}^{n}\Ind\left\{\xs_t=\xs, \eta_t(\xs) \leq \frac{24\alpha\log t}{\Delta(\xs)^2}\right\}
&\le \sum_{t=m+1}^{n}\Ind\left\{t \leq \frac{24\alpha v(\xs)}{\Delta(\xs)^2}\log t + 1\right\} \\
&\le \sum_{t=m+1}^{n}\Ind\left\{t \leq \frac{48\alpha v(\xs)}{\Delta(\xs)^2}\log \frac{24\alpha v(\xs)}{\Delta(\xs)^2} + 1\right\}\label{p:6} \\
&= \sum_{t=m}^{n-1}\Ind\left\{t \leq \frac{48\alpha v(\xs)}{\Delta(\xs)^2}\log \frac{24\alpha v(\xs)}{\Delta(\xs)^2}\right\}\\
&\le \sum_{t=1}^{n-1}\Ind\left\{t \leq \frac{48\alpha v(\xs)}{\Delta(\xs)^2}\log \frac{24\alpha v(\xs)}{\Delta(\xs)^2}\right\}\\
&\le \frac{48\alpha v(\xs)}{\Delta(\xs)^2}\log \frac{24\alpha v(\xs)}{\Delta(\xs)^2},\label{p:7}
\end{align}
where~\eqref{p:6} is from Lemma~\ref{lemma:Functionf}.

Even if $v(\xs)=\infty$, we can still use $\eta_t(\xs)\geq T_t(\xs)$ from Lemma~\ref{th:ess} (in this case, $\eta_t(\xs)>0$ is guaranteed by the initial Round-Robin execution):
\begin{align}
\sum_{t=m+1}^{n}\Ind\left\{\xs_t=\xs, \eta_t(\xs) \leq \frac{24\alpha\log t}{\Delta(\xs)^2}\right\}
&\le \sum_{t=1}^{n}\Ind\left\{\xs_t=\xs,T_t(\xs) \leq \frac{24\alpha\log t}{\Delta(\xs)^2}\right\} \\
&\le \sum_{t=1}^{n}\Ind\left\{\xs_t=\xs,T_t(\xs) \leq \frac{24\alpha\log n}{\Delta(\xs)^2}\right\} \\
&\le \frac{24\alpha}{\Delta(\xs)^2}\log n.\label{p:8}
\end{align}

Statement (a) is obtained by using~\eqref{p:7} for all policies. From~\eqref{p:regret}:
\begin{align}
	\E R(n) &\le \sum_{x\in\DecisionSet:\Delta(x)>0}\frac{48\alpha v(\xs)}{\Delta(\xs)}\log \frac{24\alpha v(\xs)}{\Delta(\xs)^2} + \Delta(\xs_1) + \frac{2}{\alpha-1}\sum_{x\in\DecisionSet:\Delta(x)>0}\Delta(\xs) \\
	&\le \sum_{x\in\DecisionSet:\Delta(x)>0}\frac{48\alpha v(\xs)}{\Delta(\xs)}\log \frac{24\alpha v(\xs)}{\Delta(\xs)^2}+ \Delta(\xs_1) + \frac{2}{\alpha-1}\min\left\{1,\sqrt{2\log v}\right\}K,
\end{align}
where the last inequality is by combining Lemma~\ref{th:gap} with the trivial $\Delta(\xs)\le 1$. Note that the $\Ind\{v(\xs)=\infty \lor \xs=\xs_1\}$ terms from~\eqref{p:tn} amount to the unavoidable $\Delta(\xs_1)$ in this case.

Similarly, (b) is obtained by using~\eqref{p:8} for all policies. The $\Ind\{v(\xs)=\infty \lor \xs=\xs_1\}$ terms from~\eqref{p:tn} amount to an additional $K$ regret in the worst case.

The instance-independent regret is obtained from (b) by a standard reduction (see, \eg Theorem~3 from~\cite{kveton2019perturbed}).
\end{proof}

We now prove the regret bound for Follow The Leader (FTL), reported in Algorithm~\ref{alg:ftl} for completeness.

\begin{figure}[h]
			\fbox{\parbox{\linewidth}{
					\begin{small}
						\textbf{Input}: initial policy parameters $\xs_1$, $\alpha>1$
						\begin{algorithmic}
							\State Execute $\pi_{\xs_1}$, observe $\outcome_1 \sim p_{\xs_1}$ and $ \func(\outcome_1)$
							\For{$t = 2,\dots,n$}
							\State Compute expected return estimate $\widecheck{J}_t(\xs)$
							\State Select $\xs_t \in \argmax_{\xs \in \DecisionSet} \widecheck{J}_t(\xs)$
							\State Execute $\pi_{\xs_t}$, observe $\outcome_t \sim p_{\xs_t}$ and $\func(\outcome_t)$
							\EndFor
						\end{algorithmic}
					\end{small}
				}
			}
			\captionof{algorithm}{Follow The Leader (FTL)}
			\label{alg:ftl}
\end{figure}

\ftl*
\begin{proof}
	The proof is similar to that of Theorem~\ref{th:optimist}. We replace the argument in~\eqref{p:5} with the following. Under $E_t\cap F_t$:
	\begin{align}
		\Delta(\xs_t) &= J(\xs^{*}) - J(\xs_t) \\
		&=J(\xs^{*}) - \widecheck{J}(\xs^{*}) + \widecheck{J}(\xs^{*}) - J(\xs_t) \\
		&\le J(\xs^{*}) - \widecheck{J}(\xs^{*}) + \widecheck{J}(\xs_t) - J(\xs_t) \label{p:13}\\
		&\le (1+\sqrt{2})\sqrt{\frac{\alpha\log t}{\eta_t(\xs^*)}} + (1+\sqrt{2})\sqrt{\frac{\alpha\log t}{\eta_t(\xs_t)}} \label{p:11}\\
		&\le (1+\sqrt{2})\sqrt{\frac{\alpha v(\xs^{*})\log t}{t-1}} + (1+\sqrt{2})\sqrt{\frac{\alpha v(\xs_t)\log t}{t-1}}, \label{p:12} \\
		&\le 2(1+\sqrt{2})\sqrt{\frac{\alpha v^{*}(\xs_t)\log t}{t-1}},
	\end{align}
	where~\eqref{p:13} is by the policy selection rule, \eqref{p:11} is from $E_t\cap F_t$, and~\eqref{p:12} is from Lemma~\ref{th:ess}. Rearranging:
	\begin{align}
		t \leq 4(1+\sqrt{2})^2\frac{\alpha v^{*}(\xs) }{\Delta(\xs)^2}\log t + 1 \leq 24\frac{\alpha v^{*}(\xs)}{\Delta(\xs)^2}\log t + 1.
	\end{align}
	Hence:
	\begin{align}
	\E[T_n(\xs)] &= \E\left[\sum_{t=2}^n \Ind\{\xs_t=\xs, E_t\cap F_t\} + \sum_{t=2}^n \Ind\{\xs_t=\xs, \overline{E_t}\cup \overline{F_t}\}\right] + \Ind\{\xs=\xs_1\} \\
	&\le \E\left[\sum_{t=2}^{n}\Ind\left\{\xs_t=\xs, t \le 24\frac{\alpha v^{*}(\xs) }{\Delta(\xs_t)^2}\log t + 1 \right\}\right] + \sum_{t=2}^{n}\Prob(\overline{E_t}\cup \overline{F_t})+ \Ind\{\xs=\xs_1\}.
	\end{align}
	We bound the second term as in~\eqref{p:14} from the proof of Theorem~\ref{th:optimist}. For the first term:
	\begin{align}
		\sum_{t=2}^{n}\Ind\left\{\xs_t=\xs, t \le 24\frac{\alpha v^{*}(\xs) }{\Delta(\xs_t)^2}\log t + 1 \right\}
		&= \sum_{t=2}^{n}\Ind\left\{t \le 24\frac{\alpha v^{*}(\xs) }{\Delta(\xs)^2}\log t + 1 \right\}\\
		& \le \sum_{t=2}^{n}\Ind\left\{t \le 48\frac{\alpha v^{*}(\xs) }{\Delta(\xs)^2}\log 24\frac{\alpha v^{*}(\xs) }{\Delta(\xs)^2} + 1 \right\}\label{p:15} \\
		& \le \sum_{t=1}^{n-1}\Ind\left\{t \le 48\frac{\alpha v^{*}(\xs) }{\Delta(\xs)^2}\log 24\frac{\alpha v^{*}(\xs) }{\Delta(\xs)^2} \right\}\\
		&\le 48\frac{\alpha v^{*}(\xs) }{\Delta(\xs)^2}\log 24\frac{\alpha v^{*}(\xs) }{\Delta(\xs)^2},
	\end{align}
	where~\eqref{p:15} is from Lemma~\ref{lemma:Functionf}.
	We then proceed as for the proof of statement (a) from Theorem~\ref{th:optimist}.
\end{proof}

\subsection{Proofs of Section~\ref{sec:rand}}
In order to prove the results on the regret of \algname, we adopt an approach analogous to that of~\cite{kveton2019garbage, kveton2019perturbed}.

\subsubsection{General Randomized Exploration with Shared History}
We start analyzing a more general algorithm that we call \emph{General Randomized Exploration with Shared History} (GRE-SH, Algorithm~\ref{alg:gresh}). GRE-SH is the adaptation of the \emph{General Randomized Exploration} (Algorithm 1 of~\cite{kveton2019garbage}) to the mediator feedback setting.

\begin{figure}[h!]
\fbox{\parbox{\linewidth}{
 \begin{small}
 \textbf{Input}: initial policy parameters $\xs_1$
 \begin{algorithmic}
 \State Execute $\pi_{\xs_1}$, observe $\outcome_1 \sim p_{\xs_1}$ and $\func(\outcome_1)$
 \State Initialize $\Hs_{1} = \{ (\xs_1, \outcome_1, \func(\outcome_1))\}$
  \For{$t = 2,\dots,n$}
  	\State Draw $\theta_t(\xs) \sim q_{\xs}(\Hs_{t-1})$
  	\State Select $\xs_t \in \argmax_{\xs \in \DecisionSet} \theta_t(\xs)$
  	\State Execute $\pi_{\xs_t}$, observe $\outcome_t \sim p_{\xs_t}$ and $y_t = \func(\outcome_t)$
  	\State Update $\Hs_{t} = \Hs_{t-1} \cup \{ (\xs_t, \outcome_t, \func(\outcome_t))\}$
  \EndFor
  \end{algorithmic}
  \end{small}
}
}
\captionof{algorithm}{GRE-SH}\label{alg:gresh}
\end{figure}

For the sake of the analysis, let us define for any $t \in [n]$ and $\zeta \in \mathbb{R}$:
\begin{equation}
    Q_{t}(\xs, \zeta) = \Prob \left( \eta_t(\xs) \ge \zeta | \eta_t(\xs) \sim q_{\xs}(\Hs_{t-1}) ,\, \Hs_{t-1} \right).
\end{equation}

W.l.o.g. we will assume that the optimal policy $\xs^*$ is unique. We can now provide the following result for Algorithm~\ref{alg:gresh}.

\begin{thr}\label{th:gres}
For any tunable parameters $(\zeta(\xs))_{\xs \in \DecisionSet \setminus \{\xs^*\}} \in \mathbb{R}^{K-1}$ and $\alpha > 1$, the expected $n$-round regret of Algorithm~\ref{alg:gresh} can be bounded from above as:
\begin{equation*}
    \E R(n) = \sum_{\xs \in \DecisionSet \setminus \{\xs^*\}} \Delta(\xs) \E[T_{n}(\xs)] \le \Delta(\xs_1) + \sum_{\xs \in \DecisionSet \setminus \{\xs^*\}} \Delta(\xs) (a(\xs) + b(\xs))
\end{equation*}
where:
\begin{align*}
    & a(\xs) =\E \left[\sum_{t=2}^n \min \left\{ \left(\frac{1}{Q_{t}(\xs^*, \zeta(\xs))} -1 \right) \Prob (\xs_t = \xs^* | \Hs_{t-1}) , 1 \right\} \right]\\
    & b(\xs) = \E \left[ \sum_{t=2}^n  \Ind \left\{ Q_{t}(\xs, \zeta(\xs)) > t^{-\alpha} \right\} \Prob (\xs_t = \xs | \Hs_{t-1}) \right] + \frac{1}{\alpha-1}. \\
\end{align*}
\end{thr}

\begin{proof}
 We extend the proof of Theorem 1 of~\cite{kveton2019garbage}. Our goal is to bound the expected number of execution for each suboptimal policy $\xs \in \DecisionSet \setminus \{\xs^*\}$. Let us fix $\xs$ and consider the event:
 \begin{equation}
 	E_t(\xs) = \left\{ \eta_t(\xs) \le \zeta(\xs) \right\}.
 \end{equation}
  We proceed to the decomposition:
    \begin{align*}
        \E[T_{n}(\xs)] &= \E \left[\sum_{t=1}^n \Ind\{\xs_t =\xs \} \right]  \\
        & = \E \left[\sum_{t=2}^n \Ind\left\{\xs_t = \xs, E_t(\xs) \right\} \right] + \E \left[\sum_{t=2}^n \Ind\left\{\xs_t = \xs, \overline{E_t(\xs) } \right\} \right] + \Ind\{ \xs = \xs_1 \}.
    \end{align*}
    To derive the expression of the term $b(\xs)$, let us consider the set of rounds $\mathcal{T} = \left\{t =2,\dots,n \,:\, Q_{t}(\xs, \zeta(\xs)) > t^{-\alpha} \right\}$. We have:
    \begin{align*}
    \E \left[\sum_{t=2}^n \Ind\left\{\xs_t =\xs,  \overline{E_t(\xs)} \right\} \right] & =  \E \left[\sum_{t \in \mathcal{T}} \Ind\left\{\xs_t =\xs,  \overline{E_t(\xs)} \right\} \right] +  \E \left[\sum_{t \notin \mathcal{T}} \Ind\left\{\xs_t=\xs,  \overline{E_t(\xs) } \right\} \right] \\
    & \le \E \left[\sum_{t \in \mathcal{T}} \Ind\left\{\xs_t =\xs \right\} \right] +  \E \left[\sum_{t \notin \mathcal{T}} \Ind\left\{ \overline{E_t(\xs) } \right\} \right] \\
    & = \E \left[\sum_{t=2}^n \Ind \left\{ \xs_t =\xs, Q_{t}(\xs, \zeta(\xs)) > t^{-\alpha} \right\} \right] + \E \left[\sum_{t \notin \mathcal{T}} t^{-\alpha} \right] \\
    & \le \E \left[\sum_{t=2}^n \Ind \left\{ Q_{t}(\xs, \zeta(\xs)) > t^{-\alpha} \right\} \Prob \left( \xs_t =\xs  | \Hs_{t-1} \right)  \right] + \frac{1}{\alpha-1}. \\
\end{align*}
where we note $\E \left[ \Ind\left\{ \overline{E_t(\xs)} \right\} \right] = \E \left[ \E \left[ \Ind\left\{ E_t(\xs) \right\}  \rvert \Hs_{t-1} \right] \right] =  \E \left[ \Prob \left( \overline{E_t(\xs) }\rvert \Hs_{t-1} \right)  \right] = \E \left[ Q_{t}(\xs,\zeta(\xs)) \right]$ and that $\E \left[ \Ind \left\{ \xs_t =\xs, Q_{t}(\xs, \zeta(\xs)) > t^{-\alpha} \right\} \right] = \E \left[ \Ind \left\{ Q_{t}(\xs, \zeta(\xs)) > t^{-\alpha} \right\}\E \left[ \Ind \left\{ \xs_t =\xs \right\} | \Hs_{t-1} \right]  \right] =\E \left[ \Ind \left\{ Q_{t}(\xs, \zeta(\xs)) > t^{-\alpha} \right\}\Prob \left( \xs_t =\xs  | \Hs_{t-1} \right)  \right] $. Finally, we bounded the summation with the integral: $\sum_{t \notin \mathcal{T}} t^{-\alpha} \le \sum_{t =2}^{\infty} t^{-\alpha} \le  \int_{x=1}^{\infty} x^{-\alpha} \de x = \frac{1}{\alpha-1}$ when $\alpha > 1$.

To derive the term $a(\xs)$, we need an auxiliary lemma, similar to Lemma 1 of~\cite{agrawal2013further}.
\begin{lemma}\label{lemma:auxProb}
    For all $t \in [n]$ and for all $\xs \in \DecisionSet \setminus \{\xs^*\}$ it holds that:
    \begin{equation}
        \Prob \left( \xs_t = \xs, E_t(\xs) | \Hs_{t-1} \right) \le \left( \frac{1}{Q_{t}(\xs, \zeta(\xs))} -1 \right) \Prob \left( \xs_t=\xs^*, E_t(\xs)| \Hs_{t-1} \right).
    \end{equation}
\end{lemma}

\begin{proof}
    Let us consider the derivation:
    \begin{align*}
         \Prob \left( \xs_t = \xs | E_t(\xs), \Hs_{t-1} \right) & = \Prob \left( \forall \xs' \neq \xs \,:\, \eta_t(\xs) \ge \eta_t(\xs') | E_t(\xs), \Hs_{t-1} \right) \\
         & \le \Prob \left( \forall \xs' \in \DecisionSet \,:\, \eta_t(\xs') \le \zeta(\xs) | E_t(\xs), \Hs_{t-1} \right) \\
         & = \Prob \left( \eta_t(\xs^*) \le \zeta(\xs) | \Hs_{t-1} \right)  \Prob \left( \forall \xs' \neq \xs^* \,:\, \eta_t(\xs') \le \zeta(\xs) | E_t(\xs), \Hs_{t-1} \right) \\
         & = \left( 1-Q_{t}(\xs^*, \zeta(\xs)) \right) \Prob \left( \forall \xs' \neq \xs^* \,:\, \eta_t(\xs') \le \zeta(\xs) | E_t(\xs), \Hs_{t-1} \right),
    \end{align*}
    where we exploited the fact that given $\Hs_{t-1}$ the events $E_t(\xs)$ and $ \eta_t(\xs^*) \le \zeta(\xs) $ are independent. Moreover, we have:
    \begin{align*}
        \Prob \left( \xs_t=\xs^*   | E_t(\xs), \Hs_{t-1} \right) & = \Prob \left( \forall \xs' \neq \xs^* \,:\, \eta_t(\xs^*) \ge \eta_t(\xs') | E_t(\xs), \Hs_{t-1} \right) \\
        & \ge \Prob \left( \forall \xs' \neq \xs^* \,:\,  \eta_t(\xs^*) > \zeta(\xs) \ge  \eta_t(\xs') | E_t(\xs), \Hs_{t-1} \right) \\
        & = \Prob \left( \eta_t(\xs^*) > \zeta(\xs) | \Hs_{t-1} \right) \Prob \left( \forall \xs' \neq \xs^*\,:\,\eta_t(\xs')\le \zeta(\xs) |  E_t(\xs), \Hs_{t-1} \right) \\
        & = Q_{t}(\xs^*, \zeta(\xs)) \Prob \left( \forall \xs' \neq \xs^* \,:\,\eta_t(\xs')\le \zeta(\xs) |  E_t(\xs), \Hs_{t-1} \right).
    \end{align*}
    Putting together these two inequalities and using the rule of the conditional probability, we get the result.
\end{proof}

Using Lemma~\ref{lemma:auxProb}, we have:
\begin{align*}
     \E \left[\sum_{t=2}^n \Ind\left\{\xs_t=\xs,  E_t(\xs) \right\} \right] & = \E \left[\sum_{t=m+1}^n \Prob \left( \xs_t=\xs, E_t(\xs) | \Hs_{t-1} \right) \right] \\
     & \le \E \left[\sum_{t=2}^n \min \left\{ \left( \frac{1}{Q_{t}(\xs, \zeta(\xs))} -1 \right) \Prob \left( \xs_t=\xs^*, E_t(\xs)| \Hs_{t-1} \right)  , 1 \right\} \right]\\
     & \le \E \left[\sum_{t=2}^n \min \left\{ \left( \frac{1}{Q_{t}(\xs, \zeta(\xs))} -1 \right) \Prob \left( \xs_t=\xs^*| \Hs_{t-1} \right)  , 1 \right\} \right],
\end{align*}
where we simply observed that $ \Prob \left( \xs_t=\xs^*, E_t(\xs)| \Hs_{t-1} \right) \le  \Prob \left( \xs_t=\xs^*| \Hs_{t-1} \right)$.
\end{proof}

\subsubsection{Proof of Theorem~\ref{th:randomist}}
Recalling that Algorithm~\ref{alg:randomist} falls in the GRE-SH case, we can now proceed with the proof of Theorem~\ref{th:randomist}.

\randomist*
\begin{proof}
	For the sake of the proof, we denote $\eta_t(\xs) = \widecheck{J}_t(\xs) + U_t(\xs)$. We apply Theorem~\ref{th:gres} with the choice $ \zeta(\xs) = J(\xs) +\frac{a}{2}+ \frac{\Delta(x)}{2}$

	\paragraph{Upper Bound on $b(\xs)$} Let us start with rewriting $Q_t(\xs, \zeta(\xs))$ as:
	\begin{equation*}
        Q_{t}(\xs,\zeta(\xs)) = \Prob \left( \eta_t(\xs) > J(\xs) + \frac{a}{2}+ \frac{\Delta(x)}{2} | \Hs_{t-1}\right).
    \end{equation*}
    We ignore the dependence on $\zeta(\xs)$ whenever clear from the context, thus $Q_t(\xs) = Q_{t}(\xs,\zeta(\xs))$. We start with bounding the term $b(\xs)$. Let us consider the event:
    \begin{equation*}
    E_{t}(\xs) = \left\{ \widecheck{J}_t(\xs) - J(\xs) \le \frac{\Delta(\xs)}{4} \right\}.
    \end{equation*}
    We can bound the probability that event $E_{t}(\xs)$ does not occur, by means of the inequalities of Lemma~\ref{th:ci}:
    \begin{align*}
        \Prob \left(\overline{ E_{t}(\xs)} \right) & = \Prob \left( \widecheck{J}_t(\xs) - J(\xs) > \frac{\Delta(\xs)}{4} \right) \le \exp\left[ - \frac{\Delta(\xs)^2 \eta_t(\xs)}{32 \left(1 + \frac{\Delta(\xs)}{12} \sqrt{\frac{\eta_t(\xs)}{\alpha \log t}} \right)} \right] \le t^{-\alpha},
    \end{align*}
    provided that $\eta_t(\xs) \ge \frac{32 \left(\sqrt{19}+10\right) \alpha }{9 \Delta(\xs) ^2} \log t \simeq \frac{52 \alpha   }{\Delta(\xs) ^2} \log t $. Under event $E_{t}(\xs)$, we can bound the probability $Q_t(\xs)$ by means of H\"oeffding's inequality and recalling that in any case $b \le \sqrt{\frac{\alpha \log t}{\eta_t(\xs)}}$:
    \begin{align*}
        Q_{t}(\xs) & = \Prob \left( \widecheck{J}_t(\xs) + U_t(\xs) > J(\xs) + \frac{a}{2}+ \frac{\Delta(\xs)}{2} \right)\\
        & \le \Prob \left(  U_t(\xs) -\frac{a}{2} >  \frac{\Delta(\xs)}{4} \right) \\
        & = \Prob \left( \frac{1}{a\eta_t(\xs)} \sum_{l=1}^{a\eta_t(\xs)} z_l - \frac{1}{2} >  \frac{\Delta(\xs)}{4a} - \frac{b}{a} \right) \\
        & \le \Prob \left( \frac{1}{a\eta_t(\xs)} \sum_{l=1}^{a\eta_t(\xs)} z_l - \frac{1}{2} >  \frac{\Delta(\xs)}{4a} - \frac{1}{a} \sqrt{\frac{\alpha \log t}{\eta_t(\xs)}} \right) \\
        & \le \exp \left[ - \frac{2}{a} \left( \frac{\Delta(\xs)}{4} - \sqrt{\frac{\alpha \log t}{\eta_t(\xs)}} \right)^2 \eta_t(\xs) \right] \le t^{-\alpha},
    \end{align*}
    where we have to enforce the following two conditions:
    \begin{align*}
    	& \frac{\Delta(\xs)}{4} -  \sqrt{\frac{\alpha \log t}{\eta_t(\xs)}} > 0 \\
    	& \frac{2}{a} \left( \frac{\Delta(\xs)}{4} -  \sqrt{\frac{\alpha \log t}{\eta_t(\xs)}} \right)^2 \eta_t(\xs) \ge \alpha \log t.
    \end{align*}
    The second condition leads to:
    \begin{align*}
    & \frac{2}{a} \left( \frac{\Delta(\xs)}{4} - { \sqrt{\frac{\alpha \log t}{\eta_t(\xs)}}} \right)^2 s \ge \alpha \log t \; \implies \; \frac{\Delta(\xs)}{4} - \sqrt{\frac{ \alpha \log t}{\eta_t(\xs)}} > \sqrt{\frac{a \alpha \log t}{2\eta_t(\xs)}} \\
    & \qquad \implies \; \eta_t(\xs) > \frac{16\alpha}{\Delta(\xs)^2} \left( 1 + \sqrt{\frac{a}{2}} \right)^2 \log t \simeq \frac{16(2+a)\alpha}{\Delta(\xs)^2} \log t.
    \end{align*}
	Combining the two conditions, we require $\eta_t(\xs) \ge \frac{(52 + 16a)\alpha}{\Delta(\xs)^2}\log t$. If $v(\xs) < \infty$, we use $\eta_t(\xs) \ge \frac{t-1}{v(\xs)}$ and we apply Lemma~\ref{lemma:Functionf} to get the following condition on the number of rounds that we denote with $t^\dagger$:
	\begin{equation}
		t \ge v(\xs) \frac{2 (52 + 16a) \alpha}{\Delta(\xs)^2} \log \left[ v(\xs) \frac{(52 + 16a)\alpha}{ \Delta(\xs)^2} \right] + 1 := t^\dagger.
	\end{equation}
	Now, we bound the term $b(\xs)$ when $v(\xs) < \infty$:
	\begin{align}
        b(\xs) & = \E \left[ \sum_{t=2}^n  \Ind \left\{ Q_{t}(\xs ) > t^{-\alpha} \right\} \Prob (\xs_t = \xs | \Hs_{t-1}) \right] + \frac{1}{\alpha-1} \\
        & \le \E \left[ \sum_{t=2}^n  \Ind \left\{ Q_{t}(\xs) > t^{-\alpha} \right\} \right] + \frac{1}{\alpha-1} \\
        & = \sum_{t=2}^n  \Prob \left( Q_{t}(\xs) > t^{-\alpha} \right) + \frac{1}{\alpha-1}.\label{eq:b1}
    \end{align}
    If $t \le t^\dagger$, we bound trivially $\Prob \left( Q_{t}(\xs, ) > t^{-\alpha} \right) \le 1$. Otherwise, we have:
    \begin{align*}
        \Prob \left( Q_{t}(\xs) > t^{-\alpha} \right) & = \Prob \left( Q_{t}(\xs ) > t^{-\alpha} |E_{t}(\xs) \right) \Prob(E_{t}(\xs)) + \Prob \left( Q_{t}(\xs) > t^{-\alpha} | \overline{E_{t}(\xs)} \right) \Prob(\overline{E_{t}(\xs)}) \\
        & \le \Prob \left( Q_{t}(\xs ) > t^{-\alpha} |E_{t}(\xs) \right) + \Prob(\overline{E_{t}(\xs)}) \le 0 + t^{-\alpha},
    \end{align*}
    where we exploited that under $E_{t}(\xs)$, we have $Q_{t}(\xs ) \le t^{-\alpha}$. Thus, we have:
    \begin{align*}
        b(\xs) & \le \sum_{t=2}^{\lfloor t^\dagger \rfloor} 1 + \sum_{t = \lfloor t^\dagger \rfloor + 1}^{n} t^{-\alpha} + \frac{1}{\alpha-1}  \le t^\dagger - 1 + \frac{2}{\alpha-1},
    \end{align*}
    where we bounded the summation with the integral, recalling that it must be $\lfloor t^\dagger \rfloor + 1 \ge 2$.

	Instead, when $v(\xs) = \infty$, we use $\eta_t(\xs) \ge T_t(\xs)$, leading to the condition on the number of executions of policy $\xs$, that we denote with $s^\ddagger$:
	\begin{equation}
	T_t(\xs) \ge \frac{(52 + 16a)\alpha}{ \Delta(\xs)^2}\log n := s^\ddagger.
	\end{equation}
	To get the second bound on $b(\xs)$, \ie the bound when $v(\xs) = \infty$, we need some further manipulations.
	We denote with $t_s(\xs)$ the random round in which policy $\xs$ is executed for the $s$-th time, with $t_0(\xs) = 0$.
    \begin{align*}
        b(\xs) & = \E \left[ \sum_{t=2}^n  \Ind \left\{ Q_{t}(\xs ) > t^{-\alpha} \right\} \Prob (\xs_t = \xs | \Hs_{t-1}) \right] + \frac{1}{\alpha-1} \\
       & = \E \left[ \sum_{s=0}^{n-1} \sum_{t=t_s(\xs)+1}^{t_{s+1}(\xs)}  \Ind \left\{ Q_{t}(\xs ) > t^{-\alpha} \right\} \Prob (\xs_t = \xs | \Hs_{t-1}) \right] + \frac{1}{\alpha-1} \\
      & = \E \left[ \sum_{s=0}^{\lfloor s^\ddagger \rfloor} \sum_{t=t_s(\xs)+1}^{t_{s+1}(\xs)}  \Ind \left\{ Q_{t}(\xs ) > t^{-\alpha} \right\} \Prob (\xs_t = \xs | \Hs_{t-1}) \right] \\
      & \quad + \E \left[ \sum_{s=\lfloor s^\ddagger \rfloor+1}^{n-1} \sum_{t=t_s(\xs)+1}^{t_{s+1}(\xs)}  \Ind \left\{ Q_{t}(\xs ) > t^{-\alpha} \right\} \Prob (\xs_t = \xs | \Hs_{t-1}) \right] + \frac{1}{\alpha-1} \\
      & \le \E \left[ \sum_{s=0}^{\lfloor s^\ddagger \rfloor}
      \sum_{t=t_s(\xs)+1}^{t_{s+1}(\xs)} \Prob (\xs_t = \xs | \Hs_{t-1}) \right]  + \E \left[ \sum_{t=t_{\lfloor s^\ddagger \rfloor+1}(\xs)+1}^n  \Ind \left\{ Q_{t}(\xs ) > t^{-\alpha} \right\}  \right] +\frac{1}{\alpha-1} \\
      & =  \sum_{s=0}^{\lfloor s^\ddagger \rfloor}
      1  + \sum_{t=t_{\lfloor s^\ddagger \rfloor+1}(\xs)+1}^n  \Prob \left( Q_{t}(\xs ) > t^{-\alpha} \right) + \frac{1}{\alpha-1}. \label{eq:b2}
    \end{align*}
    Now, for $t \ge t_{\lfloor s^\ddagger \rfloor+1}(\xs)+1$ we know that policy $\xs$ was executed at least $s^\ddagger$ times, \ie $T_t(\xs) \ge s^\ddagger$. Thus, we have, similarly as before:
	\begin{align*}
        \Prob \left( Q_{t}(\xs ) > t^{-\alpha} \right) \le \Prob \left( Q_{t}(\xs ) > t^{-\alpha}  | E_t(\xs) \right)+ \Prob \left( \overline{E_t(\xs)} \right) \le 0 + t^{-\alpha},
    \end{align*}
	where we exploited that under event $E_t(\xs)$ we have that $Q_{t}(\xs ) \le t^{-\alpha} $ and the upper bound on the probability that event $E_t(\xs)$ does not occur. Thus, we have:
    \begin{align*}
        b(\xs) & \le \sum_{s=0}^{\lfloor s^\ddagger \rfloor}
      1  +  \sum_{t=t_{\lfloor s^\ddagger \rfloor+1}(\xs)+1}^n  t^{-\alpha}  + \frac{1}{\alpha-1}  \le s^\ddagger + \frac{2 }{\alpha-1},
    \end{align*}
    where we bounded the summation with the integral, recalling that it must be $t_{\lfloor s^\ddagger \rfloor+1}(\xs)+1 \ge 2$.

    \paragraph{Upper Bound on $a(\xs)$} Fix a suboptimal policy $\xs \neq \xs^*$. We can rewrite the expression of $Q_{t}(\xs^*,\zeta(\xs))$ for the choice $\zeta(\xs) = J(\xs) +\frac{a}{2}+ \frac{\Delta(x)}{2} =  J(\xs^*) - \frac{a}{2}+ \frac{\Delta(x)}{2} $:
    \begin{equation}
        Q_{t}(\xs^*,\zeta(\xs)) = \Prob \left( \eta_t(\xs^*) > J(\xs^*) + \frac{a}{2}- \frac{\Delta(x)}{2} | \Hs_{t-1}\right)
    \end{equation}
    Again, we discard the dependence on $\zeta(\xs)$, \ie $Q_{t}(\xs^*) = Q_{t}(\xs^*,\zeta(\xs))$. Large part of the derivation exploits tools similar to those employed for $b(\xs)$. We redefine event $E_t(\xs)$ as:
    \begin{align*}
    E_t(\xs) = \left\{ J(\xs^*) - \widecheck{J}_t(\xs^*) \le \frac{\Delta(\xs)}{4} \right\}.
    \end{align*}
    We now bound the probability that event $E_t(\xs)$ does not occur:
     \begin{align*}
        \Prob \left(\overline{ E_{t}(\xs)} \right) & = \Prob \left(J(\xs^*) - \widecheck{J}_t(\xs^*) > \frac{\Delta(\xs)}{4} \right) \le \exp\left[ -\frac{1}{2} \left( \frac{\Delta(\xs)}{4} - \sqrt{\frac{\alpha \log t}{\eta_t(\xs^*)}} \right)^2 \eta_t(\xs^*)\right] \le t^{-\alpha}.
    \end{align*}
    For this, we have to enforce two conditions. The first one for the bias and the second one for fulfilling the inequality.
    \begin{align*}
     & \frac{\Delta(\xs)}{4} - \sqrt{\frac{\alpha \log t}{\eta_t(\xs^*)}}  > 0 \\
     & \frac{1}{2} \left( \frac{\Delta(\xs)}{4} - \sqrt{\frac{\alpha \log t}{\eta_t(\xs^*)}} \right)^2 \eta_t(\xs^*) > \alpha \log t.
    \end{align*}
    Let us consider the following derivation in which we retain only the positive solution:
    \begin{align*}
    & \frac{1}{2} \left( \frac{\Delta(\xs)}{4} - \sqrt{\frac{\alpha \log t}{\eta_t(\xs^*)}} \right)^2 \eta_t(\xs^*) > \alpha \log t \; \implies \; \frac{\Delta(\xs)}{4} - \sqrt{\frac{\alpha \log t}{\eta_t(\xs^*)}} > \sqrt{\frac{2 \alpha \log t}{\eta_t(\xs^*)}} \; \\
    & \qquad \implies \; \eta_t(\xs^*) > \frac{16(1+\sqrt{2})^2 \alpha}{\Delta(\xs)^2} \log t.
    \end{align*}
    Notice that this condition implies the first one on the bias. Consequently, we enforce $\eta_t(\xs^*) > \frac{16(1+\sqrt{2})^2 \alpha}{\Delta(\xs)^2} \log t \simeq \frac{94 \alpha}{\Delta(\xs)^2} \log t$. Similarly for $Q_t(\xs^*)$ under event $E_t(\xs)$:
     \begin{align*}
        Q_t(\xs^*) & = \Prob \left( \widecheck{J}_t(\xs^*) + U_t(\xs^*) >  J(\xs) + \frac{a}{2} - \frac{\Delta(x)}{2} \right) \\
        & \ge \Prob \left( U_t(\xs^*) >  \frac{a}{2} - \frac{\Delta(x)}{4} \right)  \\
        & = 1 - \Prob \left( \frac{a}{2} - U_t(\xs^*) >  \frac{\Delta(x)}{4} \right) \\
        & = 1 - \Prob \left( \frac{1}{2} - \frac{1}{a\eta_t(\xs^*)} \sum_{l=1}^{a\eta_t(\xs^*)} z_l - {\frac{1}{a} \sqrt{\frac{\alpha \log t}{s}}} >  \frac{\Delta(x)}{4a} \right) \\
        & \ge 1 - \Prob \left( \frac{1}{2} - \frac{1}{a\eta_t(\xs^*)} \sum_{l=1}^{a\eta_t(\xs^*)} z_l  >  \frac{\Delta(x)}{4a} \right) \\
        & \ge 1 - t^{-\alpha},
    \end{align*}
 provided that $\eta_t(\xs^*) \ge \frac{8a\alpha}{\Delta(\xs)^2}\log t$ (using H\"oeffding's inequality).
 Moreover, we have:
    \begin{equation*}
        \frac{1}{Q_t(\xs^*)} - 1 \le \frac{t^{\alpha}}{t^{\alpha}-1} - 1 = \frac{1}{t^{\alpha}-1} \le \frac{1}{(t-1)^{\alpha}},
    \end{equation*}
    for $t \ge 2$.
     Putting together these conditions we require $\eta_t(\xs^*) \ge \frac{94 a\alpha}{\Delta(\xs)^2}\log t$. Those conditions, lead to the very similar requirements on the rounds and on the number of executions:
  \begin{align}
        & t \ge  v(\xs^*) \frac{188 a\alpha}{\Delta(\xs)^2} \log \left[ v(\xs^*) \frac{94 a\alpha}{\Delta(\xs)^2} \right] + 1 := t^\dagger,\\
        & T_t(\xs^*) \ge \frac{94 a\alpha}{\Delta(\xs)^2}\log n := s^\ddagger.
    \end{align}
    We now proceed at bounding $a(\xs)$. If $v(\xs^*) < \infty$, we have:
    \begin{align*}
        a(\xs) & =\E \left[\sum_{t=2}^n \min \left\{ \left(\frac{1}{Q_{t}(\xs^*)} -1 \right) \Prob (\xs_t = \xs^* | \Hs_{t-1}) , 1 \right\} \right] \le \E \left[\sum_{t=2}^n \min \left\{ \frac{1}{Q_{t}(\xs^*)} -1  , 1 \right\} \right].\label{eq:a1}
    \end{align*}
    If the round index is smaller than $t \le t^\dagger$, we bound the $\min$ with $1$. Otherwise, we proceed to the following decomposition, based on whether event $ E_t(\xs)$ occurs:
    \begin{align*}
        \E  \left[ \min \left\{ \frac{1}{Q_{t}(\xs^*)} -1  , 1 \right\} \right] & = \E  \left[\min \left\{ \frac{1}{Q_{t}(\xs^*)} -1  , 1 \right\} |  E_t(\xs) \right] \Prob \left( E_t(\xs) \right) \\
        & \quad +  \E  \left[ \min \left\{ \frac{1}{Q_{t}(\xs^*)} -1  , 1 \right\} |  \overline{E_t(\xs)} \right] \Prob \left( \overline{E_t(\xs)} \right) \\
        & \le \E  \left[ \min \left\{ \frac{1}{Q_{t}(\xs^*)} -1  , 1 \right\} |  E_t(\xs) \right] + \Prob \left( \overline{E_t(\xs)} \right) \\
        & \le \frac{1}{(t-1)^{\alpha}} + t^{-\alpha}.
    \end{align*}
    Putting all together, we have:
    \begin{align*}
        a(\xs) & \le \sum_{t=2}^{\lfloor t^\dagger \rfloor} 1 +  \sum_{\lfloor t^\dagger \rfloor + 1}^{n}  \left( \frac{1}{(t-1)^{\alpha}} + t^{-\alpha} \right)  \le t^{\dagger} -1+ \frac{\alpha+1}{\alpha - 1},
    \end{align*}
    where we bounded the summations with the integrals, recalling that $\lfloor t^\dagger \rfloor + 1 \ge 2$.

    \begin{remark} It is worth noting that in this derivation of the bound on the term $a(\xs)$ we did not exploit the properties of the perturbation distribution. This is justified by the fact that we are considering the case $v(\xs)< \infty$ and each sample for each policy is informative for all the policies. Indeed, FTL enjoys constant regret in this setting.
    \end{remark}

    For the case $v(\xs^*) = \infty$, we need additional manipulations on the term $a(\xs)$:
    \begin{align*}
        a(\xs) & =\E \left[\sum_{t=2}^n \min \left\{ \left(\frac{1}{Q_{t}(\xs^*)} -1 \right) \Prob (\xs_t = \xs^* | \Hs_{t-1}) , 1 \right\} \right]\\
        & =\E \left[\sum_{s=0}^{n-1} \sum_{t=t_s(\xs^*)+1}^{t_{s+1}(\xs^*)} \min \left\{ \left(\frac{1}{Q_{t}(\xs^*)} -1 \right) \Prob (\xs_t = \xs^* | \Hs_{t-1}) , 1 \right\} \right]\\
        & =\E \left[\sum_{s=0}^{\lfloor s^\ddagger \rfloor} \sum_{t=t_s(\xs^*)+1}^{t_{s+1}(\xs^*)} \min \left\{ \left(\frac{1}{Q_{t}(\xs^*)} -1 \right) \Prob (\xs_t = \xs^* | \Hs_{t-1}) , 1 \right\} \right]\\
        & \quad + \left[\sum_{s=\lfloor s^\ddagger \rfloor + 1}^{n-1} \sum_{t=t_s(\xs^*)+1}^{t_{s+1}(\xs^*)} \min \left\{ \left(\frac{1}{Q_{t}(\xs^*)} -1 \right) \Prob (\xs_t = \xs^* | \Hs_{t-1}) , 1 \right\} \right]\\
        & \le \E \left[ \E \left[ \sum_{s=0}^{\lfloor s^\ddagger \rfloor}  \sum_{t=t_s(\xs^*)+1}^{t_{s+1}(\xs^*)} \left( \frac{1}{Q_{t}(\xs^*)} -1 \right)  \Ind \left\{ \xs_t = \xs^*  \right\}  | \Hs_{t-1} \right]  \right] \\
        & \quad + \E \left[ \sum_{t=t_{\lfloor s^\ddagger \rfloor}(\xs^*)+1}^{n} \min \left\{ \left(\frac{1}{Q_{t}(\xs^*)} -1 \right) , 1 \right\} \right]\\
        & = \E \left[\sum_{s=0}^{\lfloor s^\ddagger \rfloor} \left( \frac{1}{Q_{t_{s(\xs^*)+1}}(\xs^*)} -1 \right) \right] + \E \left[ \sum_{t=t_{\lfloor s^\ddagger \rfloor}(\xs^*)+1}^{n} \min \left\{ \left(\frac{1}{Q_{t}(\xs^*)} -1 \right) , 1 \right\} \right].
    \end{align*}
    Now, for $s \le s^\ddagger$ we apply the upper bound that will be proved later in Lemma~\ref{lemma:antiConc}:
    \begin{equation}\label{eq:expl}
        \E \left[ \frac{1}{Q_{t}(\xs^*)} \right] = \E \left[ \frac{1}{\Prob \left(\eta_t(\xs^*) \ge J(\xs^*) + \frac{a}{2} - \frac{\Delta(\xs)}{2} | \Hs_{t-1}\right)} \right] \le c, \quad \forall t \ge 1.
    \end{equation}
    Instead, for $t \ge t_{\lfloor s^\ddagger \rfloor}(\xs^*)+1$, we know that the optimal policy $\xs^*$ was executed at least $s^\ddagger$ times, \ie $T_t(\xs) \ge s^\ddagger$. Therefore:
    \begin{align*}
        \E \left[  \min \left\{ \left(\frac{1}{Q_{t}(\xs^*)} -1 \right) , 1 \right\} \right] & = \E \left[  \min \left\{ \left(\frac{1}{Q_{t}(\xs^*)} -1 \right) , 1 \right\} | E_t(\xs^*) \right] \Prob \left(E_t(\xs^*) \right) \\
        & \quad + \E \left[  \min \left\{ \left(\frac{1}{Q_{t}(\xs^*)} -1 \right) , 1 \right\} | \overline{E_t(\xs^*)} \right]\Prob \left(\overline{E_t(\xs^*) }\right) \\
        & \le \frac{1}{(t-1)^\alpha} + t^\alpha.
    \end{align*}
    Putting all together, we have:
    \begin{align*}
        a(\xs) & \le 1 + \sum_{s=0}^{\lfloor s^\ddagger \rfloor} c +  \sum_{t_{\lfloor s^\ddagger \rfloor}(\xs^*)+1}^{n}  \left( \frac{1}{(t-1)^{\alpha}} + t^{-\alpha} \right)  \\
        & \le c s^\ddagger  + \frac{\alpha+1}{\alpha - 1}.
    \end{align*}

   \paragraph{Putting together $a(\xs)$ and $b(\xs)$}
   Consider the case in which $v(\xs) < \infty$:
   \begin{align*}
   		a(\xs) + b(\xs) &  \le v(\xs^*) \frac{188 a\alpha}{\Delta(\xs)^2} \log \left[ v(\xs^*) \frac{94 a\alpha}{\Delta(\xs)^2} \right] + \frac{\alpha+1}{\alpha - 1} \\
   		& \quad +   v(\xs) \frac{2 (52 + 16a) \alpha}{\Delta(\xs)^2} \log \left[ v(\xs) \frac{(52 + 16a)\alpha}{ \Delta(\xs)^2} \right] + \frac{2}{\alpha-1} \\
   		& \le    v^*(\xs) \frac{(188+32a) \alpha}{\Delta(\xs)^2} \log \left[ v^*(\xs) \frac{(94 + 16a)\alpha}{ \Delta(\xs)^2} \right] + \frac{\alpha+3}{\alpha-1},
   \end{align*}
   where we exploited the definition of $v^*(\xs) = \max\{v(\xs), v(\xs^*)\}$. The result is obtained by recalling that each $\Delta(\xs) \le \min\{1, \sqrt{2 \log v}\}$ (Lemma~\ref{th:gap}). Instead, for the case $v(\xs) = \infty$ we have:
    \begin{align*}
   		a(\xs) + b(\xs) &  \le  c\frac{94 a\alpha}{\Delta(\xs)^2}\log n + \frac{2 }{\alpha-1} + \frac{(52 + 16a)\alpha}{ \Delta(\xs)^2}\log n  + \frac{\alpha+1}{\alpha - 1} \\
   		& \le \frac{(110 a + 52) c\alpha}{\Delta(\xs)^2}\log n +  \frac{\alpha+3}{\alpha - 1},
   \end{align*}
   where we simply exploited that $c > 1$. The result follows by trivially bounding each $\Delta(\xs) \le 1$.
\end{proof}

\subsubsection{Derivation of the $c$ term}
We now explicitly derive the value of $c$ bounding the expectation $\E \left[ \frac{1}{Q_{t}(\xs^*)} \right] $ for all $t\in [n]$. For the sake of the analysis, we will replace $a$ with $2a$. Let us define the following symbols:
\begin{align*}
     \overline{X} = s \E[\widecheck{J}_t(\xs^*)]  \in [0,s], \qquad
     {X} = s \widecheck{J}_t(\xs^*) \in [0, sM_t], \qquad
     \overline{Y} = as ,\qquad
     {Y} = \sum_{l=1}^{2as} z_{l},
   \end{align*}
   where $z_l \sim \mathrm{Ber}(1/2)$ and i.i.d.. First of all, we need to take into account the beneficial effect of our correction, that leads to the following derivation recalling that $J(\xs^*) - \E[\widecheck{J}_t(\xs^*)] \le b $:
   \begin{align*}
    \Prob \left(\eta_t(\xs^*) \ge J(\xs^*) + a - \frac{\Delta(\xs)}{2} | \Hs_{t-1}\right) & = \Prob \left( \widecheck{J}_t(\xs^*) + U_t(\xs^*)\ge J(\xs^*) + a - \frac{\Delta(\xs)}{2} | \Hs_{t-1}\right) \\
    & = \Prob \left( \widecheck{J}_t(\xs^*) + \frac{1}{s} \sum_{l=1}^{2as} z_{l}   \ge \underbrace{J(\xs^*) - b}_{\le \E[\widecheck{J}_t(\xs^*)]} + a - \frac{\Delta(\xs)}{2} | \Hs_{t-1}\right) \\
     & \ge \Prob \left( \widecheck{J}_t(\xs^*) + \frac{1}{s} \sum_{l=1}^{2as} z_{l}   \ge \E[\widecheck{J}_t(\xs^*)] + a - \frac{\Delta(\xs)}{2} | \Hs_{t-1}\right)\\
      & = \Prob \left( X +Y \ge \overline{X} + \overline{Y} - \frac{\Delta(\xs)s}{2} | \Hs_{t-1}\right)\\
      & \ge \Prob \left( X +Y \ge \overline{X} + \overline{Y} | \Hs_{t-1}\right).
    \end{align*}

    The following result bounds the probability $\Prob \left( X +Y \ge \overline{X} + \overline{Y} | \Hs_{t-1}\right)$.

\begin{lemma}\label{lemma:antiConc}
	For any $a > 4$, it holds that:
	\begin{align*}
	\E \left[\frac{1}{\Prob \left( X +Y \ge \overline{X} + \overline{Y} | \Hs_{t-1}\right)} \right] \le 2 + \frac{2 e^2 \sqrt{a}}{\sqrt{\pi}} \left[\frac{8}{a-4} \right]  \left( 1+ \sqrt{\frac{\pi a}{2(a-4)}} \right).
	\end{align*}
\end{lemma}

\begin{proof}
	The proof puts together some of the results presented in Appendix A of~\cite{kveton2019perturbed}. Let $f(X)$ be defined as follows:
	\begin{align*}
	f(X) = \left[ \sum_{y=\lceil\overline{X}-X+as\rceil}^{2as} g_{\mathrm{Bin}(2as, 1/2)}(y) \right]^{-1},
	\end{align*}
	where $g_{\mathrm{Bin}(n,p)}$ is the p.d.f. of a Binomial distribution of parameters $n$ and $p$. Given the definition of $f$, the following identity holds: $\E \left[\frac{1}{\Prob \left( X +Y \ge \overline{X} + \overline{Y} | \Hs_{t-1}\right)} \right] = \E[f(X)]$.	Let us define the partitioning of the interval $[0,sM_t]$, with $i_0$ the smallest integer s.t. $(i_0+1) \sqrt{s } \ge \overline{X}$:
    \begin{equation}
        \mathcal{P}_i = \begin{cases}
            (s, s M_t] & \text{if } i = -1,\\
            (\max\{\overline{X} - \sqrt{s }\}, s] & \text{if } i = 0,\\
            (\max\{\overline{X} - (i+1)\sqrt{s }, \overline{X} - i\sqrt{s}] & \text{if } i \in \{1, ..., i_0 \},\\
        \end{cases}
    \end{equation}
    We can now decompose the expectation $\E[f(X)]$ over the partitioning:
    \begin{align*}
        \E[f(X)] & = \sum_{i=-1}^{i_0} \E[\Ind \{X \in \mathcal{P}_i \} f(X) ] \\
            & \le \Prob (X \in \mathcal{P}_{-1} ) f(s) + \sum_{i=0}^{i_0-1} \Prob (X \in \mathcal{P}_i ) f(\overline{X} - (i+1)\sqrt{s}) + \Prob (X \in \mathcal{P}_{i_0} ) f(0),
    \end{align*}
    where we simply observed that $f$ is a decreasing function of $X$.
    We now proceed at bounding the probabilities for $i \ge 1$:
    \begin{align*}
        \Prob (X \in \mathcal{P}_i ) & \le \Prob \left(X \le \overline{X} - i \sqrt{s } \right) \\
        & \le \Prob \left(\E[\widecheck{J}_t(\xs^*)] - \widecheck{J}_t(\xs^*) \ge \frac{i}{\sqrt{s}} \right) \\
        & \le \exp \left[ - \frac{i^2}{2s} s \right] = \exp \left[ - \frac{i^2}{2}  \right] ,
    \end{align*}
    where we applied the concentration inequalities in Lemma~\ref{th:ci}. For $i \in \{ -1,0\}$ we bound trivially  $\Prob (X \in \mathcal{P}_i ) \le 1$. Notice that we are satisfying the constraint on the bias, thanks to the bias correction.

   We now analyze the terms $f(\overline{X} - (i+1)\sqrt{s })$. First of all, recall that $f(s) \le 2$ since $\lceil \overline{X} - s + as \rceil \le as$.
    For the other terms we apply Lemma 2 of~\cite{kveton2019perturbed} for $\delta \in [0, as]$:
    \begin{align*}
    	\sum_{y=\lceil as + \delta \rceil }^{2as}g_{\mathrm{Bin}(2as, 1/2)}(y) \ge \frac{\sqrt{\pi}}{e^2 \sqrt{a}} \exp \left[- \frac{2 (\delta + \sqrt{s})^2}{as} \right],
    \end{align*}
    Specifically, in our case, we have for $\delta = (i+1)\sqrt{s}$:
    \begin{align*}
    & \sum_{y=\lceil as + (i+1)\sqrt{s} \rceil }^{2as} g_{\mathrm{Bin}(2as, 1/2)}(y) \ge \frac{\sqrt{\pi}}{e^2 \sqrt{a}} \exp \left[- \frac{2 (i+2)^2}{a} \right].
    \end{align*}
    Instead, for $\delta = \overline{X}$:
    \begin{align*}
    & \sum_{y=\lceil as + \overline{X} \rceil }^{2as} g_{\mathrm{Bin}(2as,1/2)}(y) \ge \frac{\sqrt{\pi}}{e^2 \sqrt{a}} \exp \left[- \frac{2 (\overline{X} + \sqrt{s})^2}{as} \right] \ge \frac{\sqrt{\pi}}{e^2 \sqrt{a}} \exp \left[- \frac{2 (i_0+2)^2}{a} \right].
    \end{align*}
    As a consequence, the expectation of $f$ can be rewritten as:
    \begin{equation*}
    	\E[f(X)] \le 2 +  \frac{e^2 \sqrt{a}}{\sqrt{\pi}} \sum_{i=0}^{i_0} \exp \left[ -\frac{a i^2 - 4 (i+2)^2}{2a} \right].
\end{equation*}
To get to a result, we complete the square:
\begin{align*}
a i^2 - 4 (i+2)^2 = (a-4) \left( i -\frac{8}{a-4} \right)^2 - \frac{16a}{a-4}.
\end{align*}
It follows that, by bounding the summation with the integral under the assumption that $a > 4$:
\begin{align*}
\E[f(X)] & \le 2 +  \frac{e^2 \sqrt{a}}{\sqrt{\pi}} \sum_{i=0}^{i_0} \exp \left[ -\frac{a-4 }{2a} \left( i -\frac{8}{a-4} \right)^2  + \frac{8}{a-4}\right]\\
& \le 2 + \frac{2 e^2 \sqrt{a}}{\sqrt{\pi}} \exp \left[\frac{8}{a-4} \right] \sum_{i=0}^{\infty} \exp \left[ -\frac{a-4 }{2a} i^2 \right] \\
& \le 2 + \frac{2 e^2 \sqrt{a}}{\sqrt{\pi}} \left[\frac{8}{a-4} \right]  \left( 1+ \int_{x=0}^{\infty}  \exp \left[ -\frac{a-4 }{2a} x^2 \right] \de x \right) \\
& \le 2 + \frac{2 e^2 \sqrt{a}}{\sqrt{\pi}} \left[\frac{8}{a-4} \right]  \left( 1+ \sqrt{\frac{\pi a}{2(a-4)}} \right).
\end{align*}

\end{proof}


\section{Auxiliary Lemmas}\label{apx:auxiliaryLemmas}
In this appendix, we provide some auxiliary lemmas that are employed to prove the main results.

\begin{thr}\label{th:chernoff}
	Let $X_1,\dots,X_N$ are independent random variables satisfying $0 \le X_i \le M$ and $\E[X_i^2] \le v^2$. Let $\overline{X} = \frac{1}{n} \sum_{i=1}^N X_i$. Then for any $\epsilon \ge 0$, we have:
	\begin{align*}
		& \Prob \left(\overline{X} - \E[\overline{X}] \ge \epsilon \right) \le \exp \left[ \frac{-\epsilon^2 N}{2 \left( v^2 + \frac{M \epsilon}{3} \right) } \right], \\
	   & \Prob \left(\E[\overline{X}]  - \overline{X} \ge \epsilon \right) \le \exp \left[ \frac{-\epsilon^2 N}{2 v^2} \right].
	\end{align*}
\end{thr}

\begin{proof}
	These are just Theorems 2.7 and 2.8 of~\cite{chung2006old} suitably rephrased.
\end{proof}

\begin{lemma}\label{lemma:Functionf}
	Let $f(t) = a\log t - t +1$ for $t\ge 1$ and $a \ge 1$. Then, the following statements hold:
	\begin{enumerate}
		\item $f$ is concave;
		\item the maximum of $f$ is attained by $t_{\max}=a$;
		\item the derivative of $f$ in $t=1$ is $a-1$;
		\item $f$ admits two zeros: one in $t_1=1$ and the other in:
		\begin{equation}
		t_2 = -a W_{-1} \left( -\frac{1}{a} e^{-\frac{1}{a}} \right) \le 2 a \log a + 1,
		\end{equation}
		where $W_{-1}$ is the secondary component of the Lambert function~\cite{corless1996lambertw}.
	\end{enumerate}
\end{lemma}

\begin{proof}
	The first three points are trivial. We just prove 4. Consider the equation and the subsequent derivation:
	\begin{align*}
	& a\log t = t-1 \\
	& t^{a} = e^{t-1} \\
	& t = e^{t/a} e^{-1/a} \\
	 -& \frac{t}{a} = - \frac{1}{a} e^{\frac{t}{a}} e^{-\frac{1}{a}} \\
	 -& \frac{t}{a} e^{- \frac{t}{a}} = - \frac{1}{a} e^{-\frac{1}{a}}.
	\end{align*}
	Now we can apply the Lambert function yielding two solutions since the right-hand-side is in $[-1/e, 0]$, one for $W_0$ (the principal component) and one for $W_{-1}$ (the secondary component):
	\begin{equation}
	t_{1,2} = -a W_{\star} \left( - \frac{1}{a} e^{-\frac{1}{a}} \right), \quad \star \in \{0, -1\}.
	\end{equation}
	Moreover, since $-1/a > -1$, the first solution simplifies into $t_1 = -a  W_{0} \left( - \frac{1}{a} e^{-\frac{1}{a}} \right) = 1$, which was already clear from the definition of function $f(t)$. The other solution $t_2$ cannot be further simplified. We proceed to bound its value. Notice that, in order to prove that $t_2 \le 2a \log a +1 $, it is sufficient to show that $f(2a \log a +1) \le 0$, since $f$ is concave:
	\begin{align*}
	& f(2a \log a +1) = a \log (2 a \log a+1)-2 a \log a \le 0 &\text{if}
	&& 2 a \log a + 1 \le a^2.
	\end{align*}
	So it is enough to prove that function $g(a) \coloneqq a^2 - 2 a \log a - 1 \ge 0$ for all $a \ge 1$. But $g(1)= 0$ and $g$ is monotonically increasing in $a$.
\end{proof}

\section{Experimental Details}\label{apx:expDetails}
In this appendix, we present the practical aspects of RANDOMIST omitted in the main paper. In particular, we discuss the computation of \Renyi divergences (Appendix~\ref{sec:renyi}), we provide a detailed description of the adaptation of MCMC used in compact parameter spaces (Appendix~\ref{sec:mcmcApx}), we discuss the caching strategy that allows saving a factor $t$ in the computational complexity (Appendix~\ref{sec:caching}) and provide additional experimental results (Appendix~\ref{sec:hyperpar}).

\paragraph{Infrastructure}
The experiments have been run on a machine with two CPUs Intel(R) Xeon(R) CPU E7-8880 v4 @ 2.20GHz (22 cores, 44 thread, 55 MB cache) and 128 GB RAM.

\subsection{Computing \Renyi Divergences}\label{sec:renyi}
To generate the perturbation $U_t(\xs)$ of RANDOMIST (Algorithm~\ref{alg:randomist}), we need to compute the effective number of trajectories $\eta_t(\xs)$, which in turn requires the \Renyi divergence $d_2(p_{\xs}\|\Phi_t)$ between the candidate and the mixture of previously executed policies (see Section~\ref{sec:problem}). This latter quantity can be challenging to compute since the mixture $\Phi_t$ is typically difficult to characterize.
From~\cite{papini2019optimistic} (Theorem 5), this quantity is upper bounded by the harmonic mean of pairwise divergences:
\begin{align}
	d_{2}(p_{\xs}\|\Phi_t) \leq \frac{t-1}{\sum_{i=1}^{t-1}\frac{1}{d_{2}(p_{\xs}\|p_{\xs_i})}}.
\end{align}
In practice, we replace all occurrences of $d_2(p_{\xs}\|\Phi_t)$ in RANDOMIST with the harmonic mean. It is easy to show that Theorem~\ref{th:randomist} still holds for this modified version, since the proof is already based on the above upper bound. The same argument holds for the exploration bonus of OPTIMIST, as already observed in~\cite{papini2019optimistic}.

To compute $d_2(p_{\xs}\|\Phi_t)$, we just have to compute the pairwise \Renyi divergences $d_2(p_{\xs}\|p_{\xs'})$ for each $\xs'$ previously executed. This is straightforward in the PB-PO framework, where the divergence is between hyperpolicies:
\begin{align}
	\int_{\Theta}\int_{\mathcal{T}}\nu_{\vxi'}(\vtheta)p_{\vtheta}(\tau) \left(\frac{\nu_{\vxi}(\vtheta)p_{\vtheta}(\tau)}{\nu_{\vxi'}(\vtheta)p_{\vtheta}(\tau)}\right)^2 \de \tau \de \vtheta=
	\int_{\Theta}\nu_{\vxi'}(\vtheta)\left(\frac{\nu_{\vxi}(\vtheta)}{\nu_{\vxi'}(\vtheta)}\right)^2 \de \vtheta =
	d_{2}(\nu_{\vxi}\|\nu_{\vxi'}),
\end{align}

and the hyperpolicies are perfectly known, often Gaussian.\footnote{Closed-form expressions for the \Renyi divergence are available for several common distributions, including Gaussians~\cite{gil2013renyi}.}

In the AB-PO framework, the outcome distributions are trajectory distributions, which are unknown (although we can easily compute probability ratios). Possible estimators for $d_{2}(p_{\vtheta}\|p_{\vtheta'})$ are discussed in~\cite{metelli2018policy}. It is an open problem whether this approximation affects the regret. In this work, we only experiment with the easier PB-PO framework.

\subsection{RANDOMIST in compact parameter spaces}\label{sec:mcmcApx}
In Section~\ref{sec:rand}, we presented a version of \algname for infinite policy spaces.
We now first provide more details on the algorithm, then we present a more practical version with reduced computational complexity.
For both the versions, implemented in the parameter-based setting (PB-PO), we improve the hyperpolicy parameters by taking $M=10$ steps of the Metropolis-Hastings algorithm with Gaussian proposal. The pseudocode of our method is presented in Algorithm~\ref{alg:randMCMC}.
The two versions differ for how $\mathfrak{g}^\dagger$ is computed.

\begin{center}
\begin{minipage}{0.50\textwidth}
	\fbox{\parbox{\linewidth}{
			\begin{small}
				\textbf{Input}: initial policy parameters $\xs_1$, kernel covariance $\Sigma$
				\begin{algorithmic}
				\State Execute $\pi_{\xs_1}$, observe $\outcome_1 \sim p_{\xs_1}$ and $ \func(\outcome_1)$
				\For{$t = 2,\dots,n$}
					\State Initialize $\xs_m = \xs_{t-1}$
					\For{$m=1,...,M-1$}
						\State Select proposed point $\xs_p \sim \mathcal{N}(\xs_m, \Sigma)$
						\State Sample $\epsilon \sim \mathcal{U}(0,1)$
						\If{$\epsilon < \frac{\mathfrak{g}_t^\dagger(\xs_p|\xs_m)}{\mathfrak{g}_t^\dagger(\xs_m|\xs_p)}$}
							\State Set $\xs_{m+1} = \xs_p$
						\Else
							\State Set $\xs_{m+1} = \xs_m$
						\EndIf
					\EndFor
					\State Set $\xs_t = \xs_M$
					\State Execute $\pi_{\xs_t}$, observe $\outcome_t \sim p_{\xs_t}$ and $ \func(\outcome_t)$
					\EndFor
				\end{algorithmic}
			\end{small}
		}
	}
	\captionof{algorithm}{MCMC-RANDOMIST}\label{alg:randMCMC} \end{minipage}\hspace{5mm}
\end{center}

\paragraph{Full-Density MCMC-RANDOMIST}
As discussed in Section~\ref{sec:rand}, it is possible to approximately compute the probability for a policy $\xs$ of being the one with maximum perturbed estimated expected return.
Since the perturbation is drawn from a binomial distribution that sums $a s(\xs)$ Bernoulli samples, we compute the outer integral as a summation over the $a s(\xs)$ possible values, considering the generic probability density function $g_{\text{Bin}(as(\xs),\frac{1}{2})}$ of a binomial distribution.
To evaluate the probability of being the best for each one of these possible perturbed estimated expected return for a fixed $\xs$, we would like to multiply the probabilities for it of being larger than the estimated expected return  of any other policy: this amounts to multiplying the cumulative density functions of the distribution of the estimated expected return  of other policies $\xs'$, evaluated in the particular candidate expected return value for $\xs$.
Since, in a compact parameter space, the number of $\xs'$s is infinite, a \textit{product integral}, \ie the multiplication equivalent of the standard integral, must be considered.
We compute this integral by a kind of numerical quadrature, only considering policies taken in previous timesteps $t'<t$, leading to the following expression for the density:
\begin{align}
	{\mathfrak{g}}_t^\dagger(\xs)  &\propto \int_{\mathbb{R}} g_{\xs}(y) \prod_{t'=1}^{t-1} G_{\xs_{t'}}(y) \de y \\
 	& = \sum_{i=0}^{\#} g_{\text{Bin}(\#,\frac{1}{2})}(i) \prod_{t'=1}^{t-1} G_{\text{Bin}(\#,\frac{1}{2})} \left( \left\lfloor \frac{s(\xs_{t'})}{s(\xs)}i + s(\xs_{t'}) \left( \widehat{J}_t(\xs) - \widehat{J}_t(\xs_{t'})  \right)\right\rfloor \right),
\end{align}
where $\#=\lceil as(\xs) \rceil$.
Note that, for practical and efficiency purposes, we shift the argument of the cumulative density function $G_{(\#,\frac{1}{2})}$, rather than modifying its parameters for the different $\xs_{t'}$.
Thus, given that it can only take a discrete number of values, we employ memoization, by caching the first $10^5$ values for each combination of $\#$ and c.d.f. argument after their first computation.

\paragraph{1-Step-Density MCMC-RANDOMIST}
The MCMC-based \algname needs, for each call to the density function $\mathfrak{g}^\dagger$, to have an updated estimate of the expected returns of each policy taken in the previous timesteps.
This is particularly expensive, being based on the balance heuristic estimator, leading to a $\mathcal{O}(dt^3)$ per-iteration complexity, obviously computationally demanding.
Therefore, we propose an alternative heuristic for the computation of the density, based on the idea of comparing the estimated expected return of a proposed point just with the one of the last node in the chain constructed by the MCMC algorithm:
\begin{align}
	{\mathfrak{g}}_t^\dagger(\xs|\xs') = \sum_{i=0}^{\#} g_{\text{Bin}(\#,\frac{1}{2})}(i) G_{\text{Bin}(\#,\frac{1}{2})} \left( \left\lfloor \frac{s(\xs')}{s(\xs)}i + s(\xs') \left( \widehat{J}_t(\xs) - \widehat{J}_t(\xs')  \right)\right\rfloor \right),
\end{align}
where $\#=\lceil as(\xs) \rceil$.
In other words, we employ $\frac{\mathrm{P}(\widehat{J}_t(\xs') \geq \widehat{J}_t(\xs))}{\mathrm{P}(\widehat{J}_t(\xs) \geq \widehat{J}_t(\xs'))}$ as a ratio for the Metropolis-Hastings algorithm.
Thus, the estimation of the expected return for all the policies played in the previous rounds can be avoided, gaining a factor of $t$ in the asymptotic time complexity.
Despite its potentially greedy aptitude, this version enjoys very similar performance to the algorithm that uses the full density.

We leave as a future work the study of the convergence properties of these approximations of the density.

\subsection{Reducing Complexity via Caching}\label{sec:caching}
The main computational overhead for the class of algorithms presented in this paper is the update of the estimated expected return of a policy.
Indeed, for both OPTIMIST and \algname, both the importance weights (through the balance heuristic) and the \Renyi divergence between the policy and the mixture of the other policies must be computed, leading to a time complexity of the order of $\mathcal{O}(t^2)$.
In fact, for each policy, potentially, every other policy should be queried in order to obtain the correct MIS denominator for a given sample and the correct \Renyi divergence.

Nonetheless, it is possible to reduce the computational complexity by noting at time $t$ that, for all the samples collected up to time $t-1$, all the probabilities relative to those samples have already been computed.
Therefore, by storing a $(t-1) \times (t-1)$ matrix that holds the probability of each previous sample under each previous policy, one can tradeoff memory for time, and avoid the repeated computations of these values.
A similar reasoning applies to the \Renyi divergences.
As discussed in Appendix~\ref{sec:renyi}, we follow the same approach as in~\cite{papini2019optimistic}, and bound the \Renyi divergence $d_2(p_{\xs} \| \Phi_t)$ with the harmonic mean of the pairwise divergences between $\xs$ and the other policies.
At time step $t$, instead of recomputing from scratch the divergences for an old policies in order to update its TMIS estimator, we use a cached sum of the previous divergences, that we plug into the computation of the harmonic mean.
Overall, the time complexity of these computations is reduced to $\mathcal{O}(t)$.

\subsection{Tasks, additional experiments, hyperparameters}\label{sec:hyperpar}
We now give additional details on the tasks used to evaluate \algname, also presenting additional results.
For all the tasks, we parallelized the execution of different runs of our experiments by using the GNU Parallel tool~\cite{Tange2011a}.
We employ $\alpha=2$ for the importance weight truncation and computation of the exploration bonus.

\paragraph{LQG}
For the LQG task, we obtain the cumulative regret by computing, for all the $K=100$ hyperpolicies generated by the discretization, the expected return in closed form, then comparing it, at each round, with the one obtained using the selected hyperpolicy in a deterministic version of the environment.

\paragraph{Mountain Car}
To test the MCMC-\algname algorithm, we employ the standard Mountain Car~\cite{sutton2018reinforcement} task.
We compute the \emph{cumulative return} at time $t$ as the the average of the return obtained across all iterations for $t'\leq t$, following the same approach of~\cite{papini2019optimistic}.
We normalize the returns from the interval $[-5, 95]$ to $[0,1]$ whenever it is appropriate for our algorithm, for instance inside the argument of the cumulative density function of the binomial distribution.
We restrict our search for the hyperparameters inside the box $[-1, 1] \times [0, 20]$ and heuristically set the covariance matrix of the Gaussian proposal used to construct the MCMC chain (we use the full-density approximation) to be equal to the one of the hyperpolicy, \ie $\Sigma=\mathrm{diag}(0.15, 3)^2$.
We employ linear policies in the state.
In addition to the experiment discussed in Section~\ref{sec:exp}, Figure~\ref{fig:trajectories} reports an investigation of how the behavior of MCMC-\algname changes during its training.
In the first iterations of the algorithm, when very few samples have been collected, no hyperpolicy is able to obtain an high probability of leading to the maximum expected return.
Therefore, few proposed points are rejected and the Markov Chains constructed by the Metropolis-Hastings algorithm cross a considerable portion of the $\Xi$ space.
By contrast, in later iterations of the algorithm, when it is converging towards its maximum performance, acceptance of new proposed hyperparameters becomes rare and the steps that compose the trajectories are very small.
This behavior can be naturally explained as a manifestation, on the way MCMC trajectories are traced, of an initial exploration phase of the space $\Xi$, followed by an exploitation phase.

\begin{figure}[tb]
\centering
\includegraphics[width=\textwidth]{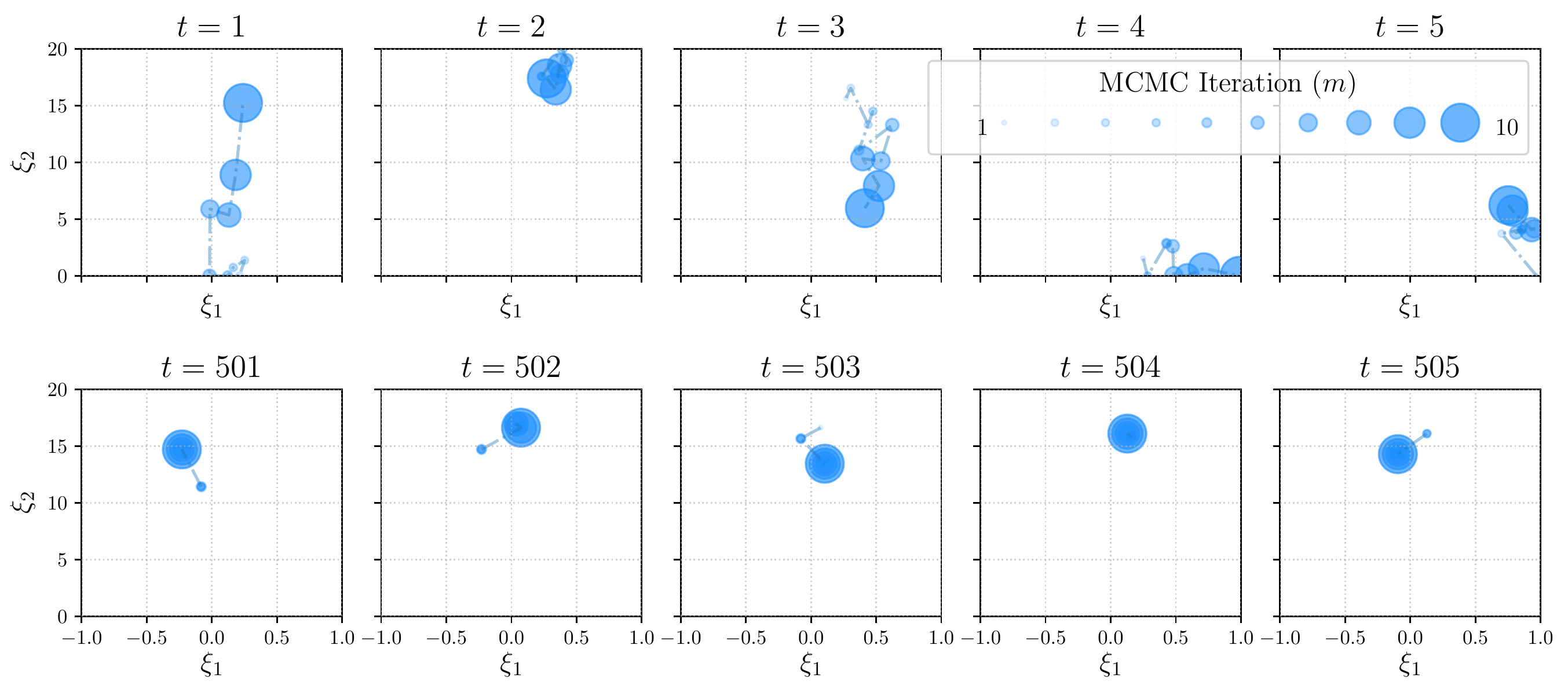}
\caption{Trajectories traced by the $M=10$ Metropolis-Hastings steps in the two-dimensional space $\Xi$ of the hyperparameters $\vxi = \left[ \xi_1, \xi_2 \right]$ of the hyperpolicy $\nu_{\vxi}$ during different iterations. Top row: first $5$ iterations. Bottom row: iterations from $501$ to $505$.}\label{fig:trajectories}
\end{figure}

\paragraph{Continuous Cartpole}
An interesting feature of the continuous version of \algname is its scalability compared to existing approaches (\ie the discretization-based version of OPTIMIST, called OPTIMIST2 in~\cite{papini2019optimistic}).
\begin{figure}[tb]
\centering
\includegraphics[width=0.9\textwidth]{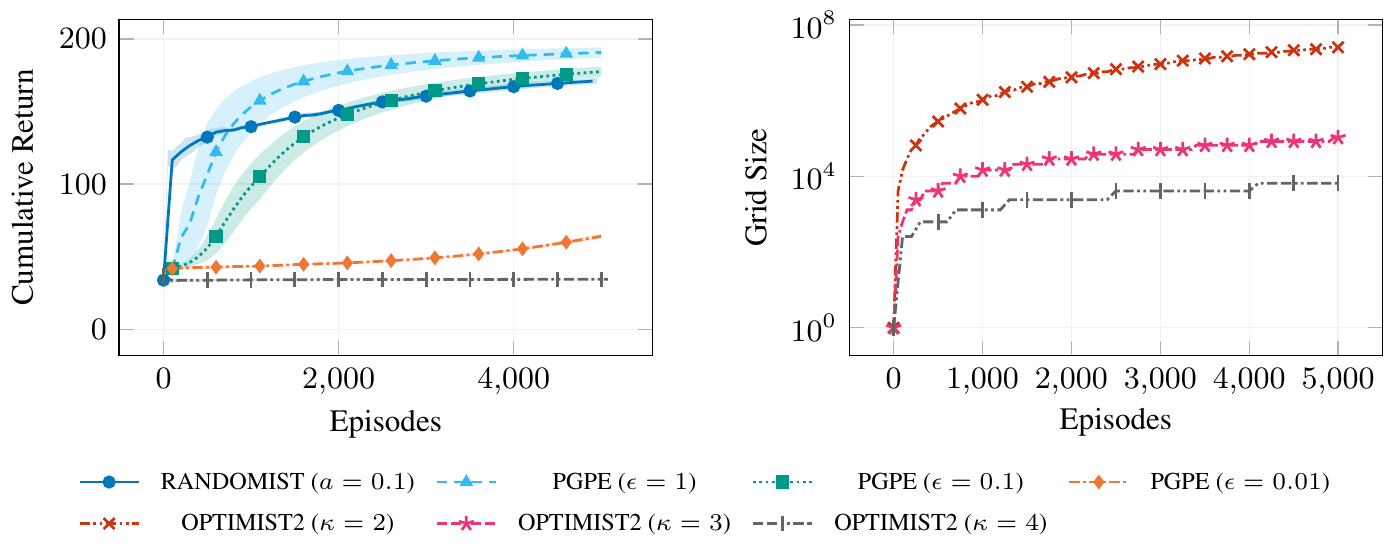}
\caption{Cumulative Return in the continuous cartpole task (5 runs, 95\% c.i.) and number of points on the discretization grid of OPTIMIST2.}\label{fig:cartpole}
\end{figure}
To assess it in practice, we employ a continuous Cartpole environment, with $\mathrm{dim}(\mathcal{S}) = 4$ and $\mathrm{dim}(\mathcal{A}) = 1$ and an horizon $H=200$.
We employ the most efficient version of \algname, by only considering proposed and current points during the computation of the density ratio in the MCMC step and the hyperparameter $a=0.1$.
We again employ the parameter-based setting, by leveraging a Gaussian hyperpolicy of learned mean and constant covariance $\Sigma=\mathrm{diag}(1,1,1,1)$, together with policies linear in the state features.
We run, at each iteration, $M=10$ steps of MCMC (we use the one-step approximation), with a Gaussian proposal centered in the current point and covariance equal to the one of the hyperpolicy, and search for the optimal hyperparameters in the box $[-2, 2] \times [0, 4] \times [0, 10] \times [0, 12]$.
Figure~\ref{fig:cartpole} (left) shows the cumulative return, computed as in the mountain car experiment, obtained by \algname, OPTIMIST2\footnote{For running OPTIMIST2, we employ the publicly-available official implementation.} and PGPE~\cite{sehnke2008policy} run with different step sizes.
\algname is able to scale to this task and to obtain satisfying performance, comparable to the one of a policy gradient approach.
It is interesting to observe that OPTIMIST2 completely fails to learn the task, or, at least, displays a very slow improving behavior. We tested OPTIMIST2 with different discretizations, all fulfilling Theorem 4 of~\cite{papini2019optimistic} with $\kappa \in \{2,3,4\}$. Increasing the value of $\kappa$ leads to coarser discretization and, consequently, reduces the computational complexity at the cost of a degradation of the regret guarantee (that remains sublinear anyway). For clarity, since all versions performed almost the same, in the left plot we report the case of $\kappa=4$ only. The plot on the right shows the number of points in the grid employed by OPTIMIST2 as a function of the number of iterations. We observe that the number of points soon becomes intractable.
This example shows how a discretization-based approach, like OPTIMIST2, despite its theoretical guarantees, does not scale in practice as the number of dimensions increases.

\end{document}